\documentclass[11pt]{article}
\usepackage[margin=1in]{geometry}
\usepackage{booktabs} 
\usepackage{times}
\usepackage[dvipsnames]{xcolor}
\usepackage[backref,colorlinks,citecolor=blue,bookmarks=true]{hyperref}
\usepackage{mathtools, amssymb, amsthm}

\usepackage[ruled,vlined]{algorithm2e}
\usepackage{svg}
 \usepackage{tikz}
\SetAlCapSkip{1em}
\usepackage{algorithmic}
\usepackage[title]{appendix}

\title{Testing Noise Assumptions of Learning Algorithms}
\author{
    \begin{tabular}{cc}
       \begin{tabular}{c}
       Surbhi Goel \thanks{\texttt{surbhig@cis.upenn.edu}. Supported by OpenAI Superalignment Fast Grant. Part of this work was conducted while the author was
visiting the Simons Institute for the Theory of Computing.} \\
     University of Pennsylvania
     \end{tabular}  &  \begin{tabular}{c}
     Adam R. Klivans\thanks{\texttt{klivans@cs.utexas.edu}. Supported by NSF award AF-1909204 and the NSF AI Institute for Foundations of Machine Learning (IFML).} \\
	 UT Austin
     \end{tabular} \\\\
       \begin{tabular}{c}
         Konstantinos Stavropoulos\thanks{\texttt{kstavrop@cs.utexas.edu}. Supported by the NSF AI Institute for Foundations of Machine Learning (IFML) and by scholarships from Bodossaki Foundation and Leventis Foundation.} \\
	 UT Austin
       \end{tabular}  &  \begin{tabular}{c}
            Arsen Vasilyan\thanks{\texttt{arsenvasilyan@gmail.com}. Supported in part by NSF awards CCF-2006664, DMS-2022448, CCF-1565235, CCF-1955217, CCF-2310818, the NSF AI Institute for Foundations of Machine Learning (IFML) Big George Fellowship and Fintech@CSAIL. Part of this work was conducted while the author was
visiting the Simons Institute for the Theory of Computing.} \\
	 UT Austin
       \end{tabular}
    \end{tabular}    
}
\date{}

\usepackage{mathtools, amssymb, amsthm}
\usepackage[capitalize]{cleveref}
\usepackage{enumitem}

\theoremstyle{plain}
\newtheorem{theorem}{Theorem}[section]
\newtheorem{lemma}[theorem]{Lemma}
\newtheorem{corollary}[theorem]{Corollary}
\newtheorem{proposition}[theorem]{Proposition}
\newtheorem{fact}[theorem]{Fact}

\newtheorem{assumption}[theorem]{Assumption}
\newtheorem{observation}[theorem]{Observation}
\newtheorem{claim}{Claim}

\theoremstyle{definition}
\newtheorem{definition}[theorem]{Definition}

\theoremstyle{remark}

\numberwithin{equation}{section}

\def\A{\mathcal{A}}

\def\D{\mathcal{D}}

\def\H{\mathcal{H}}

\def\S{\mathbb{S}}

\newcommand*{\N}{{\mathbb{N}}}
\newcommand*{\Z}{{\mathbb{Z}}}
\newcommand*{\R}{{\mathbb{R}}}

\let\eps\epsilon
\let\phi\varphi

\DeclareMathOperator*{\pr}{\mathbb{P}}

\DeclareMathOperator*{\E}{\mathbb{E}}

\DeclareMathOperator{\unif}{Unif}

\DeclarePairedDelimiter{\norm}{\|}{\|}

\DeclareMathOperator{\poly}{poly}

\DeclareMathOperator{\sign}{\mathsf{sign}}

\DeclareMathOperator{\ind}{\mathbf{1}}

\newcommand{\opt}{\mathsf{opt}}

\newcommand{\cube}[1]{\{\pm 1\}^{#1}}

\newcommand{\ignore}[1]{}

\newcommand*{\vv}{\mathbf{v}}

\newcommand*{\x}{\mathbf{x}}

\newcommand*{\Dgeneric}{\D}

\newcommand{\nats}{\mathbb{N}}

\newcommand{\Gauss}{\mathcal{N}}

\newcommand{\e}{\mathbf{e}}

\newcommand{\angleparam}{\theta}
\newcommand{\coefficientvector}{\alpha}

\newcommand{\Oracles}{\mathcal{O}}
\newcommand{\Oracle}{\mathsf{EX}}
\newcommand{\Massart}{\Oracle^{\mathsf{Massart}}}
\newcommand{\Halfspaces}{\H_{\mathsf{hs}}}
\newcommand{\RCN}{\Oracle^{\mathsf{RCN}}}

\newcommand{\accept}{\mathsf{Accept}}
\newcommand{\reject}{\mathsf{Reject}}

\newcommand{\Sunlabeled}{S}

\global\long\def\poly{\mathrm{poly}}%

\global\long\def\R{\mathbb{R}}%

\global\long\def\Z{\mathbb{Z}}%

\global\long\def\indicator{\mathbf{1}}%

\global\long\def\norm#1{\left\Vert #1\right\Vert }%

\global\long\def\abs#1{\left|#1\right|}%

\global\long\def\d{\,d}%

\global\long\def\vect#1{\mathbf{#1}}%

\global\long\def\N{\mathcal{N}}%

\global\long\def\dpairs{\mathcal{D_{\text{pairs}}}}%

\global\long\def\Tspectral{\mathcal{T}_{\text{spectral}}}%

\global\long\def\Tdisagree{\mathcal{T}_{\text{disagreement}}}%

\begin{document}

\maketitle

\begin{abstract}%
  We pose a fundamental question in computational learning theory: \textit{can we efficiently test whether a training set satisfies the assumptions of a given noise model?} This question has remained unaddressed despite decades of research on learning in the presence of noise. 
  In this work, we show that this task is tractable and present the first efficient algorithm to test various noise assumptions on the training data. 
  
 To model this question, we extend the recently proposed testable learning framework of Rubinfeld and Vasilyan \cite{rubinfeld2022testing} and require a learner to run an associated test that satisfies the following two conditions: (1) whenever the test accepts, the learner outputs a classifier along with a \textit{certificate of optimality}, and (2) the test must pass for any dataset drawn according to a specified modeling assumption on both the marginal distribution and the noise model.  We then consider the problem of learning halfspaces over Gaussian marginals with Massart noise (where each label can be flipped with probability less than $1/2$ depending on the input features), and give a fully-polynomial time testable learning algorithm. 
 
We also show a separation between the classical setting of learning in the presence of structured noise and testable learning. In fact, for the simple case of random classification noise (where each label is flipped with fixed probability $\eta = 1/2$), we show that testable learning requires super-polynomial time while classical learning is trivial.

\end{abstract}
\thispagestyle{empty}

\newpage
\setcounter{page}{1}

\section{Introduction}

Developing efficient algorithms for learning in the presence of noise is one of the most fundamental problems in machine learning with a long line of celebrated research.  Assumptions on the noise model itself vary greatly.  For example, the well-studied random classification noise model (RCN) assumes that the label corruption process is independent across examples, whereas malicious noise models allow a fraction of the (joint) data-generating distribution to be changed adversarially.  Understanding the computational landscape of learning with respect to different noise models remains a challenging open problem, serving as the central focus of numerous works in the theory of supervised learning \cite{blum1998polynomial, awasthi2015efficient,awasthi2016learning,yan2017revisiting,zhang2017hitting,mangoubi2019nonconvex,diakonikolas2020learning, diakonikolas2020learning, diakonikolas2022learning_general,diakonikolas2018learning, BSHOUTY2002255, diakonikolas2024near} and unsupervised learning \cite{unsupervised_1,unsupervised_2,unsupervised_3,unsupervised_4,unsupervised_5,unsupervised_6,unsupervised_7,unsupervised_8,unsupervised_9,unsupervised_10,unsupervised_11,unsupervised_12}.

In this paper, we address for the first time whether it is possible to efficiently test if the assumptions of a specific noise model hold for a given training set. There are two key reasons for developing such a test. First, without verifying the assumptions of the noise model, we cannot guarantee that our resulting hypothesis achieves the optimal error rate. Second, it is essential to select the learning algorithm best suited to the noise properties of the training set. Specifically, highly structured noise models often admit faster algorithms, and we should choose these algorithms whenever possible.

We use the recently introduced {\em testable learning} \cite{rubinfeld2022testing} framework to model these questions.  In this framework, a learner first runs a test on the training set.  Whenever the test accepts, the learner outputs a classifier along with a proof that the classifier has near-optimal error.  Furthermore, the test must accept with high probability whenever the training set is drawn from a distribution satisfying some specified set of modeling assumptions.  If the test rejects, the learner recognizes that one of the modeling assumptions has failed and will therefore refrain from outputting a classifier.  Here, our modeling assumptions will include both the structure of the noise model and the structure of the marginal distribution from which the data is generated.

More concretely, we will consider the problem of learning halfspaces under Gaussian marginals with respect to {\em Massart} noise, an extensively studied problem where an adversary flips binary labels independently with probability at most $1/2$ (the probability of flipping can vary across instances). The goal is to find a halfspace 
$\sign(\vv \cdot \x)$ with near-optimal misclassification error rate $\opt+\eps$, where $\opt$ is the best misclassification error rate achievable by a halfspace. For this problem, a long line of work~\cite{awasthi2015efficient,awasthi2016learning,yan2017revisiting,zhang2017hitting,mangoubi2019nonconvex,diakonikolas2020learning} resulted in the algorithm 
of Diakonikolas et al. \cite{diakonikolas2020learning}  
that runs in time $\poly(d/\eps)$ and achieves the optimal error rate.
In contrast, the worst-case-noise version of this problem (i.e. agnostic learning or, equivalently, learning with adversarial label noise) is believed to require exponential time in the accuracy parameter, even with respect to Gaussian marginals \cite{kalai2008agnostically,diakonikolas2021agnostic,goel2020statistical}.

In this work, we give a testable learning algorithm for halfspaces that runs in time $\poly(d/\eps)$ and \emph{certifies the optimality} of its output hypothesis whenever it accepts. 
 Additionally, the algorithm is guaranteed to accept (with high probability) and output a classifier if the marginal distribution is Gaussian and the noise satisfies the Massart condition.

\subsection{Our Results}

\noindent\textbf{Noise Model.} We focus on the class of i.i.d. oracles where the marginal distribution on $\R^d$ is the standard Gaussian and the labels are generated by an origin-centered halfspace with Massart noise, as defined below.

\begin{definition}[Massart Noise Oracle]\label{definition:massart}
    Let $f :\R^d \to \cube{}$ be a concept, let $\eta:\R^d\to[0,1/2]$ and let $\D$ be a distribution over $\R^d$. The oracle $\Massart_{\D,f,\eta}$ receives $m\in \nats$ and returns $m$ i.i.d. examples of the form $(\x,y)\in \R^d\times\cube{}$, where $\x\sim \D$ and $y = \xi\cdot f(\x)$, with $\xi = 1$ w.p. $1-\eta(\x)$ and $\xi = -1$ w.p. $\eta(\x)$. The quantity $\sup_{\x\in \R^d} \eta(\x)\in [0,1/2]$ is called the noise rate.  
\end{definition}

Formally, we consider the oracle class $\Massart_{\Gauss, \Halfspaces, \eta_0} = \{\Massart_{\Gauss, f, \eta}: f\in \Halfspaces, \sup_{\x\in\R^d}\eta(\x) \le \eta_0\}$, where $\Gauss$ is the standard Gaussian distribution in $d$ dimensions and $\Halfspaces$ is the class of origin-centered halfspaces over $\R^d$, which is formally defined as follows.

\begin{definition}[Origin-Centered Halfspaces]\label{definition:origin-centered-halfspaces}
    We denote with $\Halfspaces$ the class of origin-centered halfspaces over $\R^d$, i.e., the class of functions $f:\R^d\to\cube{}$ of the form $f(\x) = \sign(\vv\cdot \x)$ for some $\vv\in\S^{d-1}$, where $\sign(t) = 1$ if $t\ge 0$ and otherwise $\sign(t) = -1$.
\end{definition}

\noindent\textbf{Learning Setting.} Our results work in the following extension of testable learning \cite{rubinfeld2022testing}.

\begin{definition}[Testable Learning, extension of Definition 4 in \cite{rubinfeld2022testing}]\label{definition:testable-noise}
    Let $\H \subseteq \{\R^d \to \{\pm 1\}\}$ be a concept class, $\Oracles$ a class of (randomized) example oracles and $m:(0,1)\times(0,1)\to\nats$. The tester-learner receives $\eps,\delta\in(0,1)$ and a dataset $\bar{S}$ consisting of i.i.d. points from some distribution $\D_{\x,y}$ over $\R^d\times \cube{}$ and then either outputs $\reject$ or $(\accept,h)$ for some $h:\R^d\to \cube{}$, satisfying the following. 
    \begin{enumerate}
        \item (Soundness). The following event happens with probability at least $1-\delta$.
        \[
            \text{ If the algorithm accepts then }\pr_{(\x,y)\sim \D_{\x,y}}[y\neq h(\x)] \le \opt + \eps,\text{ where }\opt = \min_{f\in \H} \pr_{(\x,y)\sim \D_{\x,y}}[y\neq f(\x)]
        \]
        \item (Completeness). If $\bar S$ is generated by $\Oracle(m')$, for some i.i.d. oracle $\Oracle \in \Oracles$ and $m'\ge m(\eps,\delta)$, then the algorithm accepts with probability at least $1-\delta$.
    \end{enumerate}
\end{definition}

The difference between \Cref{definition:testable-noise} and the definition of \cite{rubinfeld2022testing} is that the completeness criterion does not only concern the marginal distribution on $\R^d$, but the joint distribution over $\R^d\times \cube{}$. The choice of the oracle class $\Oracles$ encapsulates all of the modeling assumptions under which our algorithm should accept (both on the marginal distribution on $\R^d$, as well as on the labels). Note that the probability of success can be amplified through repetition (see \cite{rubinfeld2022testing}), so it suffices to solve the problem for $\delta = 1/3$. Our main results and their relation to prior work are summarized in \Cref{table:sq-complexity}.

\begin{table*}[h]\begin{center}
\begin{tabular}{l c c} 
 \toprule
 \textbf{Noise Model} &  \textbf{Classical Setting} & \textbf{Testable Setting} \\ \midrule
  Massart, $\eta_0 = 1/2 - c$ & $\poly(d, 1/\eps)$ \cite{diakonikolas2020learning} & $\poly(d,1/\eps)$ {\color{orange}[\Cref{theorem:main-result}]} \\ \midrule
  Strong Massart, $\eta_0 = \frac{1}{2}$ \begin{tabular}{c}  (Upper) \\ (Lower) \end{tabular}  &  \begin{tabular}{c}$d^{O(\log(1/\eps))} 2^{\poly(1/\eps)}$ \cite{diakonikolas2022learning_general} \\$d^{\Omega(\log(1/\eps))}$ \cite{diakonikolas2022learning_general} \end{tabular} & \begin{tabular}{c}  $d^{\tilde{O}(1/\eps^2)}$ \cite{rubinfeld2022testing,gollakota2022moment}  \\ $d^{\Omega(1/\eps^2)}$ {\color{orange}[\Cref{theorem:sq-lower-bound}]} \end{tabular} 
  \\ \midrule
 Adversarial \hspace{4.8em}\begin{tabular}{c}  (Upper) \\ (Lower) \end{tabular}& \begin{tabular}{c}  $d^{\tilde{O}(1/\eps^2)}$ \cite{kalai2008agnostically}  \\ $d^{\Omega(1/\eps^2)}$ \cite{diakonikolas2021optimality} \end{tabular} & \begin{tabular}{c}  $d^{\tilde{O}(1/\eps^2)}$ \cite{rubinfeld2022testing,gollakota2022moment}  \\ $d^{\Omega(1/\eps^2)}$ (implied) \end{tabular} \\
 \bottomrule
\end{tabular}
\end{center}
\caption{Runtime upper and lower bounds (in the Statistical Query model) for learning the class of origin-centered halfspaces $\Halfspaces$ over the standard Gaussian distribution with respect to different noise assumptions. 
}
\label{table:sq-complexity}
\end{table*}

\noindent\textbf{Upper Bound.} In \Cref{theorem:main-result}, we show that there is a polynomial-time tester-learner for the class $\Halfspaces$ with respect to $\Massart_{\Gauss,\Halfspaces,\eta_0}$ for any $\eta_0\le 1/2-c$, where $c$ is any positive constant. Moreover, whenever our algorithm accepts, it is guaranteed to output the optimal halfspace with respect to the input dataset $\bar S$, even if $\bar S$ is not generated from i.i.d. examples and can, therefore, be completely arbitrary. Given the upper bounds of \Cref{table:sq-complexity} our algorithm can be used as a first step before applying the more powerful (but also more expensive) tester-learner of \cite{rubinfeld2022testing,gollakota2022moment}. If our algorithm accepts, then we do not need to run the more expensive algorithm. In other words, our results highlight that testable learning can be used for algorithm selection for problems where different assumptions motivate different algorithmic approaches.

\vspace{.5em}\noindent\textbf{Lower Bounds.} Our upper bound holds when the noise rate is bounded away below $1/2$. We show that this is necessary: in the high-noise regime ($\eta_0 = 1/2$), the best known lower bounds for learning under adversarial label noise also hold in the testable setting, with respect to random classification noise of rate $1/2$ (\Cref{definition:rcn}), which is a special case of Massart noise. We give both cryptographic lower bounds (\Cref{theorem:crypto-hardness}) assuming subexponential hardness on the problem of learning with errors (LWE), as well as statistical query lower bounds (\Cref{theorem:sq-lower-bound}).
Our lower bounds are inherited from lower bounds from the literature of agnostic learning \cite{diakonikolas2021optimality,tiegel2023hardness,diakonikolas2023near} (combined with \cref{observation:distinguishing-via-testable-learning}). Our testable learning model highlights an underappreciated aspect of these agnostic learning lower bounds, namely, that the hard instances are in fact indistinguishable from completely random instances (i.e., random classification noise of rate $1/2$).

Our results imply a separation between the classical and testable settings in the high-noise regime ($\eta_0 = 1/2$, see second row of \Cref{table:sq-complexity}), demonstrating that the complexity of testable learning displays a sharper transition with respect to varying noise models compared to classical learning. For the of RCN at noise rate $1/2$ case, the separation is even stronger, since learning is trivial in the classical setting.

\subsection{Our Techniques}
The techniques we employ in this work are significantly more sophisticated than recently developed tools from testable learning.  In fact, it is not even clear that techniques from testable learning should apply, as assumptions on the marginal distribution are quite different from assumptions on the noise model.  Concretely, we depart from prior work in testable learning \cite{gollakota2023efficient,gollakota2023tester} where the testers are designed to certify specific properties of a particular learning algorithm.  Instead, here we obtain a ``black-box" result that can take {\em any} learner that is guaranteed to output a near-optimal halfspace in the Massart setting and certify optimality properties of the learner's output hypothesis. To do this, we decompose the error of the candidate output in terms of quantities for which we can provide certifiable bounds by developing appropriate testers (see \Cref{section:poly-time} for more details on the decomposition). In particular, we provide a disagreement tester with significantly sharper guarantees compared to the one developed for standard testable learning \cite{gollakota2023tester}, and a spectral tester that combines and expands ideas from \cite{gollakota2023tester} as well as recent work on tolerant testable learning \cite{goel2024tolerant}. The main technical tool we develop to provide these improved guarantees is a notion of families of sandwiching approximators with respect to partitions of $\R^d$.

An outline of the proof of our main result (\Cref{theorem:main-result}) is provided in \Cref{section:poly-time}. As a warm-up, in \Cref{section:poly-time-rcn}, we consider the special case of random classification noise, whose analysis is simpler. We complete the proof sketch of our main result in \Cref{section:poly-time-massart}. In the following, we give an overview of the disagreement and the spectral testers.

\vspace{.5em}\noindent\textbf{Disagreement tester and sandwiching polynomials.} Let $\vv$ be a unit vector and $S$ be a dataset of size $\poly(d/\epsilon)$. If $S$ consists of Gaussian data-points, then for every unit vector $\vv'$ (w.h.p. over $S$)
\[
\pr_{\x \in S}
\left[
\sign(\vv\cdot\x)\neq \sign(\vv'\cdot\x)\right] =\measuredangle(\vv,\vv')/\pi\pm \eps
.
\]
Suppose, given $S$ and $\vv$, one would like to certify that this property approximately holds for every $\vv'$. 
The method of exhaustive search - i.e. checking this property for different candidate vectors $\vv'$ - can be shown to require at least $2^{\Omega(d)}$ time.
Using a moment-based approach, we show how to improve this run-time exponentially. In particular, in time $\poly(d,1/\eps)$, we can certify that for all $\vv'\in\S^{d-1}$
\begin{equation}
\label{eq: angle condition}
\pr_{\x \in S}
\left[
\sign(\vv\cdot\x)\neq \sign(\vv'\cdot\x)\right] =(1\pm 0.01)\measuredangle(\vv,\vv')/ \pi\pm \eps
.
\end{equation}
Moreover, whenever $S$ is Gaussian, our tests are guaranteed to pass (\Cref{theorem:disagreement-tester-up}). Note that \cite{gollakota2023efficient,gollakota2023tester,diakonikolas2023efficient} provided disagreement testers that certified one-sided bounds and suffered constant multiplicative error factors\footnote{i.e. certified that $\pr_{\x \in S}
\left[
\sign(\vv\cdot\x)\neq \sign(\vv'\cdot\x)\right] \leq O(\measuredangle(\vv,\vv'))+ \eps$.}, while here our disagreement testers certify both upper and lower bounds on the disagreement probability, with a small and controllable multiplicative error factor.

{As mentioned earlier, directly checking the disagreement for each candidate vector $\vv'$ in a Euclidean cover of the sphere $\S^{d-1}$ does not work, since their number is exponential to the dimension $d$.}
Instead, our tester discretizes {$\R^d$} into buckets corresponding to $\vv\cdot \x\in [i\eps, (i+1)\eps]$ for varying $i$ (\Cref{figure:regions}) and checks for any constant-degree monomial $m$ that
\begin{equation}
\label{eq: band tester}
\E_{\vect x \sim S}[m(\vect x)\cdot\indicator_{i \epsilon \leq\vect x\cdot\vect v\leq (i+1) \epsilon}]
\approx
\E_{\vect x\sim\N(0, I_{d})}[m(\vect x)\cdot\indicator_{i \epsilon \leq\vect x\cdot\vect v\leq (i+1)\epsilon}]
.
\end{equation}
We show that passing this test for constant-degree $m$ is sufficient for our purposes (\Cref{lemma:chow-matching-main}). The previous work \cite{gollakota2023efficient,gollakota2023tester,diakonikolas2023efficient} considered only tests involving monomials $m$ of degree at most $4$, and (as explained earlier) achieved bounds far weaker than \Cref{eq: angle condition}.
A key ingredient to our improved testers is extending the notion of sandwiching polynomials of \cite{gollakota2022moment} to much more general piecewise-polynomial functions (see \Cref{appendix:disagreement-tester}).

\vspace{.5em}\noindent\textbf{Spectral tester and monotonicity under removal.} The disagreement tester is only guaranteed to accept when the input $S$ is drawn i.i.d. from the standard Gaussian distribution. However, in our analysis it is important to have a tester that will accept even if given a set $S'$ which is a subset of a Gaussian sample. We call this property \emph{monotonicity under removal} and its importance is related to the fact that in the Massart noise model, the labels are not flipped independently of the corresponding features, but the probability of receiving a flipped label can adversarially depend on $\x$. Note that tester in \Cref{eq: band tester} is not monotone under removal.

To obtain a tester for the disagreement region that is monotone under removal (\Cref{theorem:spectral-tester-up}), we augment our Disagreement Tester using ideas from the recent work by \cite{goel2024tolerant} on tolerant testable learning (see \Cref{appendix:spectral-tester}). In particular, instead of checking \Cref{eq: band tester}, our Spectral Tester checks that for every constant-degree polynomial $p$ we have
\begin{equation}
\E_{\vect x \sim S}[p(\vect x)^2\cdot\indicator_{i \epsilon \leq\vect x\cdot\vect v\leq (i+1) \epsilon}]
\lesssim
\E_{\vect x\sim\N(0, I_{d})}[p(\vect x)^2\cdot\indicator_{i \epsilon \leq\vect x\cdot\vect v\leq (i+1)\epsilon}],
\end{equation}
which can be verified efficiently by computing the spectrum of an appropriate matrix.
The main difference between our spectral tester and the one in \cite{goel2024tolerant} is that ours partitions $\R^d$ into a number of strips and performs a test for each of them, while the one by \cite{goel2024tolerant} runs the same test on the whole $\R^d$ iteratively, each time removing a number of points from the input set. See \Cref{algorithm:spectral-tester} for the full algorithm description. As in the case of the Disagreement Tester, the analysis again leverages the method of piecewise-polynomial sandwiching functions introduced in this work.

\subsection{Related Work}
\noindent\textbf{Learning with label noise.}
Learning of halfspaces under label noise has been the topic of a large number of works. Perhaps the most well-studied noise model is the framework of \emph{agnostic learning} which corresponds to adversarial (i.e. worst-case) labels. In case of halfspaces, the literature exhibits a tradeoff between run-time and the classification error achievable:
\begin{itemize}
    \item In time $d^{\tilde{O}(1/\epsilon^2)}$ one can find a hypothesis with accuracy $\opt+\epsilon$ under Gaussian data distribution \cite{kalai2008agnostically,diakonikolas2010bounded}.  See also \cite{diakonikolas2021agnostic} for a proper learning algorithm. An algorithm with a run-time $d^{O(1/\epsilon^{2-\Omega(1)})}$ (and let alone a polynomial run-time) is believed to be impossible due to statistical query lower bounds \cite{goel2020statistical,diakonikolas2020near,diakonikolas2020near}, as well as recent cryptographic reductions from lattice problems \cite{diakonikolas2023near, tiegel2023hardness}. This works utilize reductions to the continuous LWE problem \cite{bruna2021continuous}, shown in \cite{gupte2022continuous} to be harder than the LWE problem widely used in lattice-based cryptography (a quantum reduction was given in \cite{bruna2021continuous}).  
    \item A worse error bound of $O(\opt)+\epsilon$ can be obtained in time $\poly(d/\epsilon)$ \cite{awasthi2017power}. Despite various refinements \cite{daniely2015ptas,diakonikolas2020non,diakonikolas2021agnostic}, the improvement of the error bound to $\opt+\eps$ is precluded by the aforementioned hardness results.
\end{itemize}
Overall, if one is not allowed to assume anything about data labels, one has to choose between a high run-time of $d^{\tilde{O}(1/\epsilon^2)}$  and or a higher error of $O(\opt)+\epsilon$.

In order to obtain an error bound of $\opt+\eps$ in time $\poly(d/\eps)$ a large body of works focused on moving beyond worst-case models of label noise. In the Random Classification Noise (RCN) model \cite{angluin1988learning} the labels are flipped independently with probability $\eta$.
It was shown in \cite{blum1998polynomial} that in the RCN model halfspaces can be learned up to error $\opt+\eps$ in time $\poly(d/\eps)$. See also  \cite{cohen1997learning, diakonikolas2021forster,diakonikolas2023strongly, diakonikolas2024near}.

The Massart noise model, introduced in \cite{massart2006risk}, is more general than the RCN model and allows the noise rate $\eta(\x)$ to differ across different points $\x$ in space $\R^d$, as long as it is at most some rate $\eta_0$. First studied in \cite{awasthi2015efficient}, learning halfspaces up to error $\opt+\eps$ under Massart noise model has been the focus of a long line of work \cite{awasthi2016learning,mangoubi2019nonconvex,yan2017revisiting,zhang2017hitting,diakonikolas2020learning, zhang2021improved,zhang2020efficient, diakonikolas2022learning_general}. 

We would like to note that intermediate steps in the algorithm \cite{diakonikolas2022learning_general} work by finding what is referred in  \cite{diakonikolas2022learning_general} as sum-of-squares certificates of optimality for certain halfspaces. We would like to emphasize that certificates in the sense of \cite{diakonikolas2022learning_general} have to be sound only \emph{assuming} that the labels satisfy the Massart property. In contrast with this, certificates developed in this work satisfy soundness without making any assumptions on the label distribution (which is the central goal of this work).  

There has also been work on distribution-free learning under Massart noise
\cite{diakonikolas2019distribution, chen2020classification}, which achieves an error bound of $\eta_0+\eps$, but as a result can lead to a much higher error than the information-theoretically optimal bound of $\opt+\eps$.

\vspace{.5em}\noindent\textbf{Testable learning.} The framework of testable learning was introduced in \cite{rubinfeld2022testing} with a focus on developing algorithms in the agnostic learning setting that provide certificates of (approximate) optimality of the obtained hypotheses or detect that a distributional assumption does not hold. Many of the existing agnostic learning results have since been shown to have testable learning algorithms with matching run-times. This has been the case for agnostic learning algorithms with $\opt+\eps$ error guarantee \cite{rubinfeld2022testing,klivans2023testable,goel2024tolerant, slot2024testably}, as well as $O(\opt)+\eps$ error bounds \cite{gollakota2024agnostically,gollakota2023tester,diakonikolas2023efficient,diakonikolas2024testable}.

We note that \cite{gollakota2024agnostically,gollakota2023tester} also give testable learning algorithms in the setting where data labels are \emph{assumed} to be Massart (and the algorithm needs to either output a hypothesis with error $\opt+\eps$ or detect that data distribution is e.g. not Gaussian). We emphasize that the results of \cite{gollakota2024agnostically,gollakota2023tester} do not satisfy soundness when the user is not promised that data labels satisfy the Massart noise condition (which is the central goal of this work).

\vspace{.5em}\noindent\textbf{Testable Learning with Distribution Shift (TDS learning).}
The recently introduced TDS framework \cite{klivans2023testable,klivans2024learning,chandrasekaran2024efficient, goel2024tolerant} considers a setting in which the learning algorithm is given a labeled training dataset and an unlabeled test dataset and aims to either (i) produce accurate labeling for the testing dataset (ii) detect that distribution shift has occurred and the test dataset is not produced from the same data distribution as the training dataset. Although conceptually similar, the work in TDS learning addresses a different assumption made in learning theory. Nevertheless, as we note, the spectral testing technique introduced in \cite{goel2024tolerant} is a crucial technical tool for our results in this work.

\section{Polynomial-Time Tester-Learners}\label{section:poly-time}
We first focus on the simpler RCN noise model, and the \emph{Disagreement Tester} we design to test the RCN noise model. We then proceed to the challenging Massart noise model and describe the \emph{Spectral Tester}, which we design to handle it.

\vspace{.5em}\noindent\textbf{Notation.}
We denote with $\R,\nats,\Z$ the sets of real, natural and integer numbers correspondingly. 
For simplicity, we denote the $d$-dimensional standard Gaussian distribution as $\Gauss_d$ or $\Gauss$ if $d$ is clear by context. For any set $S$, let $\unif(S)$ denote the uniform distribution over $S$. We may also use the notation $\x\sim S$ in place of $\x\sim \unif(S)$. For a set of points in $\R^d$, we denote with $\bar{S}$ the corresponding labeled set over $\R^{d}\times \cube{}$ where the corresponding labels are those in the input of the algorithm unless otherwise specified. For a vector $\x\in \R^d$, we denote with $x_i$ its $i$-th coordinate, i.e., $x_i = \x\cdot \e_i$.
\subsection{Warm-up: Random Classification Noise Oracles}\label{section:poly-time-rcn}

We first restrict our attention to tester-learners that are only guaranteed to accept a narrower class of noise models, corresponding to random classification noise with Gaussian marginal, formally defined as follows. 

\begin{definition}[Random Classification Noise (RCN) Oracle]\label{definition:rcn}
    Let $f :\R^d \to \cube{}$ be a concept, let $\eta_0 \in [0,1/2]$ and let $\D$ be a distribution over $\R^d$. The oracle $\RCN_{\D,f,\eta_0}$ receives $m\in \nats$ and returns $m$ i.i.d. examples of the form $(\x,y)\in \R^d\times\cube{}$, where $\x\sim \D$ and $y = \xi\cdot f(\x)$, where $\xi = 1$ w.p. $1-\eta_0$ and $\xi = -1$ w.p. $\eta_0$. In other words, $\RCN_{\D,f,\eta_0} = \Massart_{\D,f,\eta}$ where 
    $\eta(\x)$ is the constant function with value $\eta_0$. In the special case $\eta_0 = 1/2$, the function $f$ does not influence the output distribution and we denote the corresponding oracle with $\RCN_{\D,1/2}$.
\end{definition}

Informally, for some ground-truth halfspace $f$, the RCN oracle $\RCN_{\D,f,\eta_0}$ outputs an example $\x\sim \D$ whose label is $f(\x)$ with probability $1-\eta_0$ and is flipped with probability $\eta_0$.
We consider the case that $\D = \Gauss_d$, and $\eta_0 \le 1/2 - c$ for some constant $c>0$ and $f\in \Halfspaces$.

\begin{theorem}[Warm-up: RCN]\label{theorem:rcn}
    Let $c\in(0,1/2)$ be any constant and $\eta_0 = 1/2 - c$. Then, there is an algorithm that testably learns the class $\Halfspaces$ with respect to $\RCN_{\Gauss, \Halfspaces, \eta_0} = \{\RCN_{\Gauss, f, \eta_0}: f\in \Halfspaces\}$ with time and sample complexity $\poly(d, 1/\eps) \log(1/\delta)$.
\end{theorem}

Testing whether the noise is indeed RCN directly is impossible, since it requires estimating $\E[y | \x]$ for all $\x\in\R^d$, but we never see any example twice. Instead, we will need to design more specialized tests that only check the properties of the RCN model that are important for learning halfspaces. Specifically, we show that some key properties of the RCN noise can be certified using what we call the Disagreement Tester (Theorem \ref{theorem:disagreement-tester-up}).
Suppose, first, that when the samples are generated by an oracle $\RCN_{\Gauss, f^*, \eta_0}$, for some $f^*(\x) = \sign(\vv^*\cdot \x)$, then we can \emph{exactly} recover the ground-truth vector $\vv^*\in\S^{d-1}$ by running some algorithm $\A$ (in reality, $\vv^*$ can be recovered only approximately, and we will address this later). 

\vspace{.5em}\noindent\textbf{Relating the output error to optimum error.} Let $\bar S$ be the input set of labeled examples and let $\vv\in\S^{d-1}$ be the output of $\A$ on input $\bar S$. Note that, since $\bar S$ is not necessarily generated by $\RCN_{\Gauss, f^*, \eta_0}$, we do not have any a priori guarantees on $\vv$. We may relate the output error $\pr_{(\x,y)\sim \bar S}[ y \neq \sign(\vv\cdot \x)]$ to the optimum error $\pr_{(\x,y)\sim \bar S}[y \neq \sign(\vv^*\cdot \x)]$, by accounting for the set $\bar S_g$ of points $(\x,y)\in \bar S$ that are labeled correctly by $\vv$ but incorrectly by $\vv^*$, as well as the set $\bar S_b$ of points in $\bar S$ that are labeled incorrectly by $\vv$ but correctly by $\vv^*$. Overall, we have the following
\begin{equation}
    \pr_{(\x,y)\sim \bar S}[y\neq \sign(\vv\cdot \x)] = \pr_{(\x,y)\sim \bar S}[y\neq \sign(\vv^*\cdot \x)] + \frac{|\bar S_b|}{|\bar S|} - \frac{|\bar S_g|}{|\bar S|}\label{equation:error-bound}
\end{equation}
\begin{figure}[t]
    \centering
    \includegraphics[width=0.8\linewidth]{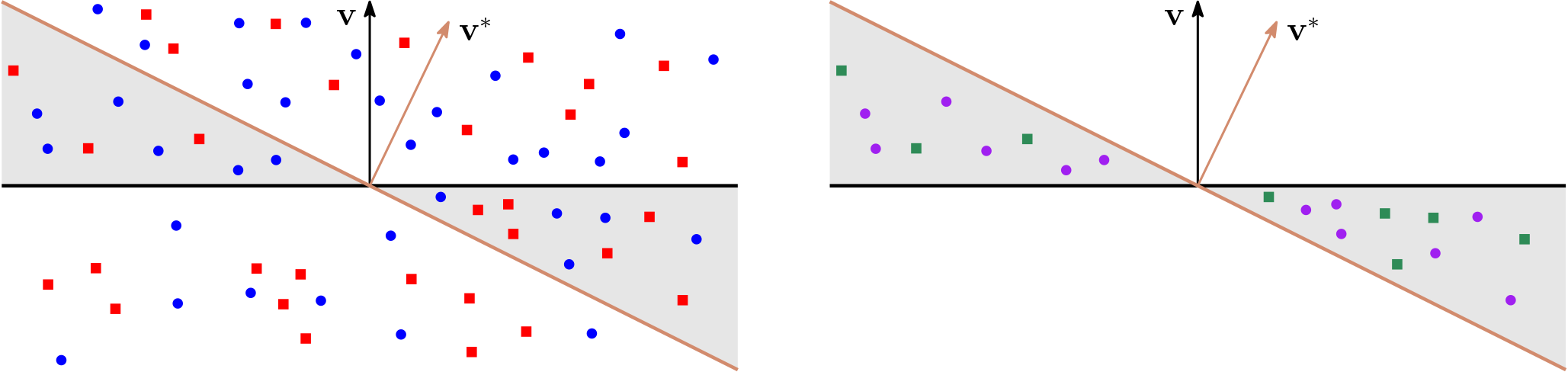}
    \caption{The shaded region is $\{\x\in\R^d: \sign(\vv\cdot \x) \neq \sign(\vv^*\cdot\x)\}$. Left: red square points have label $+1$, blue round points have label $-1$. Right: green square points are in $\bar S_g$ and purple round points are in $\bar S_b$.
    }
    \label{figure:points}
\end{figure}

\vspace{.5em}\noindent\textbf{Towards a testable bound.} We have assumed that if the noise was indeed RCN, then $\vv = \vv^*$. Therefore, in this case, $|\bar S_b| = |\bar S_g| = 0$. However, if the noise assumption is not guaranteed, given $\bar S$, we cannot directly compute the quantities $|\bar S_b|, |\bar S_g|$, as their definition involves the unknown vector $\vv^*$. Nevertheless, we show how to  obtain a certificate that $|\bar S_b|/|\bar S|- |\bar S_g|/|\bar S|$ is at most $O(\eps)$. We first express the ratios above as
\begin{align}
    {|\bar S_b|}/{|\bar S|} &= \pr_{(\x,y)\sim \bar S}[y \neq \sign(\vv\cdot \x) \text{ and } \sign(\vv^*\cdot \x) \neq \sign(\vv\cdot \x)], \label{equation:bad-set}\\
    {|\bar S_g|}/{|\bar S|} &= \pr_{(\x,y)\sim \bar S}[\sign(\vv^*\cdot \x) \neq \sign(\vv\cdot \x)] - {|\bar S_b|}/{|\bar S|}. \label{equation:good-set}
\end{align}

Combining equations \eqref{equation:error-bound}, \eqref{equation:bad-set} and \eqref{equation:good-set}, defining $\bar S_{\mathsf{False}} = \{(\x,y)\in \bar S: y\neq \sign(\vv\cdot \x)\}$ we obtain the following bound:
\begin{equation}
\label{eq: how much error changes}
\frac{|\bar S_b|}{|\bar S|} - \frac{|\bar S_g|}{|\bar S|}
\leq 2 
\frac{|\bar S_{\mathsf{False}}|}{|\bar S|}\underset{(\x,y) \sim \bar S_{\mathsf{False}}}{\pr}[\sign(\vv^*\cdot \x) \neq \sign(\vv\cdot \x) ] - \underset{(\x,y) \sim \bar S}{\pr}\bigr[\sign(\vv^*\cdot \x) \neq \sign(\vv\cdot \x)\bigr]
\end{equation}
The term ${|\bar S_{\mathsf{False}}|}/{|\bar S|}$ 
can be explicitly computed, since we have $\vv$ and $\bar S$ and we can verify whether its value is at most $1/2-c$, as would be the case if the noise was RCN. Otherwise, we may safely reject. Now, we want to obtain certificates that the first term in Equation \ref{eq: how much error changes} can't be too large and the second term can't be too small.

\vspace{.5em}\noindent\textbf{The disagreement tester and how it is applied.} 
Our goal is to certify that both probabilities in Equation \ref{eq: how much error changes} are approximately equal to ${\measuredangle(\vv,\vv^*)}/{\pi}$, which is what we would expect if the example oracle were indeed in $\RCN_{\Gauss,\Halfspaces,\eta_0}$. This is because of the following fact, as well as the fact that even after conditioning on the event $y\neq\sign(\vv^*\cdot \x)$, $\x$ remains Gaussian.
\begin{fact}\label{fact:angle-disagreement}
    Let $\x\sim \Gauss_d$ and $\vv,\vv^*\in\S^{d-1}$. Then $\pr_{\x\sim \Gauss_d}[\sign(\vv^*\cdot \x) \neq \sign(\vv\cdot \x)] = {\measuredangle(\vv,\vv^*)}/{\pi}$.
\end{fact}
Recall, however, that we do not make any assumptions on the input examples. Therefore, we would like to certify the guarantee of \Cref{fact:angle-disagreement}. We show that this is possible by developing the following tester.

\begin{theorem}[Disagreement tester, see \Cref{thm: disagreement tester}]\label{theorem:disagreement-tester-up}
    Let $\mu \in (0,1)$ be any constant. \Cref{algorithm:disagreement-tester} receives $\eps,\delta\in(0,1)$, $\vv\in \S^{d-1}$ and a set $S$ of points in $\R^d$, runs in time $\poly(d,1/\eps,|S|)$ and then either outputs $\reject$ or $\accept$, satisfying the following specifications.
    \begin{enumerate}
        \item (Soundness) If the algorithm accepts, then the following is true for any $\vv'\in\S^{d-1}$
        \[
            (1-\mu) {\measuredangle(\vv,\vv')}/{\pi} - \eps \le \pr_{\x\sim S}[\sign(\vv\cdot \x)\neq \sign(\vv'\cdot \x)] \le (1+\mu) {\measuredangle(\vv,\vv')}/{\pi} + \eps
        \]
        \item (Completeness) If $S$ consists of at least $(\frac{Cd}{\eps \delta})^C$ i.i.d. examples from $\Gauss_d$, where $C\ge 1$ is some sufficiently large constant depending on $\mu$, then the algorithm accepts with probability at least $1-\delta$.
    \end{enumerate}
\end{theorem}

We choose $\mu = c$ and run the tester above on the datapoints in $\bar S$ and $\bar S_{\mathsf{False}} = \{(\x,y)\in \bar S: y\neq \sign(\vv\cdot \x)\}$. If the tester accepts, then (using Equation \ref{eq: how much error changes}) the excess error $|\bar S_b|/|\bar S|- |\bar S_g|/|\bar S|$ 
is at most $((1-2c)(1+c)-(1-c)) {\measuredangle(\vv,\vv^*)}/{\pi} + 2\eps$ which, in turn, is upper-bounded by $ 2\eps$, since $(1-2c)(1+c)-(1-c) = -2c^2 < 0$.

\vspace{.5em}\noindent\textbf{Designing the disagreement tester.} Our disagreement tester builds on ideas from prior work on testable agnostic learning by \cite{gollakota2023efficient}. In particular, \cite{gollakota2023efficient} show that when the angle between the input vector $\vv$ and some unknown vector $\vv'$ is $\eps$, then one can give a testable bound of $O(\eps)$ on the quantity $\pr_{\x\sim S}[\sign(\vv\cdot \x)\neq \sign(\vv'\cdot \x)]$ by running some efficient tester (Proposition D.1 in \cite{gollakota2023efficient}). To achieve this, the region $\{\x\in\R^d: \sign(\vv\cdot \x)\neq \sign(\vv'\cdot \x)\}$ is covered by a disjoint union of simple regions whose masses can be testably upper bounded. In order to bound the mass of the simple regions, it is crucial to use the fact that $\vv$ is known. 

Here, our approach needs to be more careful, since we (1) require both upper and lower bounds on the quantity $\pr_{\x\sim S}[\sign(\vv\cdot \x)\neq \sign(\vv'\cdot \x)]$, (2) we do not have a specific target threshold for the angle $\measuredangle(\vv,\vv')$, but we need to provide testable bounds that involve $\measuredangle(\vv,\vv')$ as a free parameter and (3) we can only tolerate a small constant multiplicative error factor $(1\pm\mu)$. To obtain this improvement, we combine the approach of \cite{gollakota2023efficient} with the notion of sandwiching polynomial approximators. Sandwiching polynomial approximators are also used to design testable learning algorithms (see \cite{gollakota2022moment}), but we use them here in a more specialized way, by allowing the sandwiching function to be piecewise-polynomial.

\begin{figure}[t]
    \centering
    \includegraphics[width=0.6\linewidth]{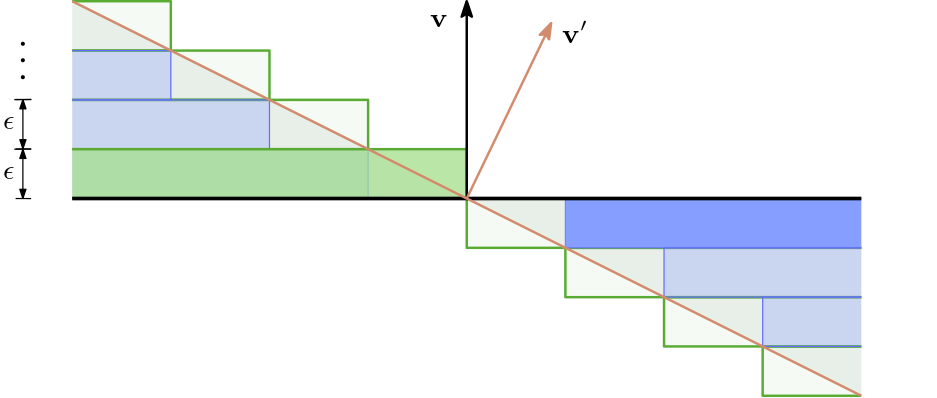}
    \caption{For vectors $\vv,\vv'\in\S^{d-1}$, the region $\{\x\in\R^d: \sign(\vv\cdot \x) \neq \sign(\vv'\cdot\x)\}$ is contained in the union of {\color{Green}green} regions 
    and it contains the union of {\color{Blue}blue} regions. In the diagram we highlight one of the green regions (top left) and one of the blue regions (bottom right).
    }
    \label{figure:regions}
\end{figure}
In particular, we first observe that the region $\{\x\in\R^d: \sign(\vv\cdot \x)\neq \sign(\vv'\cdot \x)\}$ can be approximated from above and from below by the disjoint union of a small number of simple regions (\Cref{figure:regions}). More precisely, if we let $\vv_\perp$ be the unit vector in the direction $\vv'-(\vv'\cdot \vv)\vv$, then we have
\begin{align}
    \pr_{\x\sim S}\Bigr[\vv\cdot \x \ge 0 > \vv'\cdot \x\Bigr] &\le \sum_{i=0}^\infty \pr_{\x\sim S}\Bigr[\vv\cdot \x \in [i\eps, (i+1)\eps], \vv_{\perp}\cdot \x \le -{i\eps}/{\tan(\vv,\vv')}\Bigr] \label{equation:lower-regions} \\
    \pr_{\x\sim S}\Bigr[\vv\cdot \x \ge 0 > \vv'\cdot \x\Bigr] &\ge \sum_{i=0}^\infty \pr_{\x\sim S}\Bigr[\vv\cdot \x \in [i\eps, (i+1)\eps], \vv_{\perp}\cdot \x \le -{(i+1)\eps}/{\tan(\vv,\vv')}\Bigr] \label{equation:upper-regions}
\end{align}
In fact, the number of interesting terms in the summations can be bounded by $O(\eps \log^{1/2}(1/\eps))$, because the remaining terms are testably negligible, due to Gaussian concentration and since we have access to $\vv$.

Each term in the summations of \eqref{equation:lower-regions}, \eqref{equation:upper-regions} is of the form $\E_{\x\sim S}[\mathcal{I}_i(\x)\cdot f_i(\x)]$, where $f_i$ is some unknown halfspace and $\mathcal{I}_i(\x) = \ind\{\vv\cdot\x \in [i\eps,(i+1)\eps]\}$. Since we know $\vv$, $\mathcal{I}_i(\x)$ is a known quantity for all $\x$ in $S$. The quantities $f_i(\x)$ are unknown, but we can effectively substitute them by polynomials, because, under the Gaussian distribution they admit low-degree sandwiching approximators. This allows us to provide a testable bound by matching the low degree Chow parameters of the functions $\mathcal{I}_i(\x)$ under $\unif(S)$ to the corresponding Chow parameters under $\Gauss_d$, due to the fact that polynomials are linear combinations of monomials and the number of low-degree monomials is small enough so that we can test them all. 

Moment matching and Chow matching in particular are known to have applications in testable learning (see, e.g., \cite{gollakota2022moment,rubinfeld2022testing,gollakota2023efficient,klivans2023testable,chandrasekaran2024efficient}), but here we have to use this tool in a careful way. 
More specifically, we prove the following lemma (see \Cref{prop: sandwiching wrt covering implies fooling}) based on a delicate argument that uses the sandwiching approximators of \cite{diakonikolas2010bounded,gopalan2010fooling} (see \Cref{section:sandwiching-1}).

\begin{lemma}[Informal]\label{lemma:chow-matching-main}
    Let $C\ge 1$ be some sufficiently large constant. Suppose that for all $\alpha\in\nats^d$ with $\|\alpha\|_1 \le C/\mu^4$ and for all $i$: $\left|\E_{\x\sim S}\left[\mathcal{I}_i(\x)\cdot \prod_{j\in[d]}x_j^{\alpha_j}\right] - \E_{\x\sim \Gauss}\left[\mathcal{I}_i(\x)\cdot \prod_{j\in[d]}x_j^{\alpha_j}\right]\right| \le \frac{\eps^2 \log(1/\eps)}{Cd^{C^2/\mu^4}}$. Then, we have that $\left|\pr_{\x\sim S}[\vv\cdot \x \ge 0 > \vv'\cdot \x] - \pr_{\x\sim \Gauss}[\vv\cdot \x \ge 0 > \vv'\cdot \x]\right| \le \mu\frac{\measuredangle(\vv,\vv')}{2\pi} + \eps $, for all $\vv'$.
\end{lemma}

Based on the above lemma (and a symmetric argument for the case $\vv\cdot \x < 0 \le \vv'\cdot \x$), the disagreement tester of \Cref{theorem:disagreement-tester-up} only needs to test quantities of the form $\E[\mathcal{I}_i \prod_j x_j^{\alpha_j}]$, which are known as constant-degree Chow parameters \cite{chow-parameters,o2008chow} of the functions $\mathcal{I}_i(\x)$, as described in \Cref{algorithm:disagreement-tester}. Due to standard concentration arguments, if $S$ was i.i.d. from the Gaussian distribution, then the tests would pass.

\begin{algorithm}[t]
\caption{Disagreement tester}\label{algorithm:disagreement-tester}
\KwIn{$\eps,\delta,\mu\in(0,1)$, $\vv\in\S^{d-1}$ and set $S$ of points in $\R^d$}
Let $C\ge 1$ be a sufficiently large constant \\
Set $K = \frac{2}{\eps}\sqrt{\log (2/\eps)}$, $k = C/\mu^4$ and $\Delta = \frac{\eps}{CKd^{Ck}}$\\
\For{$\alpha \in \nats^d$ with $\|\alpha\|_1 \le k$}{
    \For{$i = -K,-K+1, \dots, 0, 1, \dots, K-1$}{
        Let $\mathcal{I}_i(\x) = \ind\{i\eps \le \vv\cdot \x < (i+1)\eps\}$ for all $\x\in S$\\
        Let $\Delta_{i,\alpha} = \left|\E_{\x\sim S}[\prod_{j\in[d]}x_j^{\alpha_j} \mathcal{I}_i(\x)] - \E_{\x\sim \Gauss_d}[\prod_{j\in[d]}x_j^{\alpha_j} \mathcal{I}_i(\x)]\right|$
    }
    Let $\Delta_{\infty,\alpha} = \left|\E_{\x\sim S}[\prod_{j\in[d]}x_j^{\alpha_j} \ind\{\vv\cdot\x \ge K\eps\}] - \E_{\x\sim \Gauss_d}[\prod_{j\in[d]}x_j^{\alpha_j} \ind\{\vv\cdot\x \ge K\eps\}]\right|$ \\
    Let $\Delta_{-\infty,\alpha} = \left|\E_{\x\sim S}[\prod_{j\in[d]}x_j^{\alpha_j} \ind\{\vv\cdot\x \le -K\eps\}] - \E_{\x\sim \Gauss_d}[\prod_{j\in[d]}x_j^{\alpha_j} \ind\{\vv\cdot\x \le -K\eps\}]\right|$
}
\textbf{if} for some $(i,\alpha)$ we have $\Delta_{i,\alpha} > \Delta$ \textbf{then} output $\reject$ \textbf{else} {output $\accept$}
\end{algorithm}

\vspace{.5em}\noindent\textbf{Approximate recovery of ground truth.} The final technical hurdle that remains unaddressed in the above derivation of the testable learning result for RCN is the fact that even under the target assumption, the ground-truth vector can be recovered only approximately. In particular, the following is true.
\begin{fact}\label{fact:rcn-parameter-recovery}
    For any $\eps', \delta\in (0,1)$ and $\eta_0 = 1/2 - c$, where $c>0$ is any constant, there is an algorithm with time and sample complexity $\poly(d,1/\eps')\log(1/\delta)$ that has access to an example oracle $\RCN_{\Gauss, f^*, \eta_0}$ for some unknown $f^*(\x) = \sign(\vv^*\cdot \x)$, $\vv^*\in \S^{d-1}$ and outputs $\vv\in\S^{d-1}$ such that $\measuredangle(\vv^*,\vv) \le \eps'$, with probability at least $1-\delta$.
\end{fact}
The place where we have to be more careful is when we argue that the distribution of $\x$ conditioned on $y\neq \sign(\vv\cdot \x)$ is Gaussian, under the target assumption. This is not true anymore, because $\vv$ is not necessarily equal to $\vv^*$ and, therefore, the event $y\neq \sign(\vv\cdot \x)$ does not coincide with the event $y\neq \sign(\vv^*\cdot \x)$, which, due to the definition of RCN noise, is independent from $\x$. However, the only case that these two events do not coincide is when $\measuredangle(\x,\vv) \le O(\eps' d)$, due to the guarantee of \Cref{fact:rcn-parameter-recovery} that $\vv$ and $\vv^*$ are geometrically close. Since we have access to both $S$ and $\vv$, we may directly test whether $\pr_{\x\sim S}[\measuredangle(\x,\vv) \le O(\eps' d)]$ is bounded by $O(\eps)$, as would be the case under the target assumption, if $\eps'$ is chosen to be $\poly(\eps/d)$. Therefore, this event is certifiably negligible and the initial argument goes through.

\subsection{Massart Noise Oracles}\label{section:poly-time-massart}

We now turn to the more challenging task of testable learning with respect to Gaussian Massart oracles and state our main theorem which shows that there is a fully polynomial-time algorithm even in this case. Observe that,
informally, for some ground-truth halfspace $f$, the Massart oracle $\Massart_{\Gauss, f, \eta_0} $ can be equivalently viewed as follows: The oracle outputs a Gaussian example $\x$, and with probability $\eta_0$ the adversary is given an \emph{option} to make the accompanying label incorrect (i.e. $-f(\x)$), and otherwise the label is correct (i.e. $f(\x)$). The Massart noise model is known to be more challenging that the RCN model, because the label noise (in general) does not have symmetry properties that can be harnessed to make error terms coming from different regions cancel each other out.

\begin{theorem}[Main Result]\label{theorem:main-result}
    Let $c\in(0,1/2)$ be any constant and $\eta_0 = 1/2 - c$. Then, there is an algorithm that testably learns the class $\Halfspaces$ with respect to $\Massart_{\Gauss, \Halfspaces, \eta_0} = \{\Massart_{\Gauss, f, \eta}: f\in \Halfspaces, \sup_{\x\in\R^d}\eta(\x) \le \eta_0\}$ with time and sample complexity $\poly(d, 1/\eps) \log(1/\delta)$.

    Moreover, even if the input set $\bar S$ is arbitrary (not necessarily i.i.d.), whenever the algorithm accepts, it outputs $h\in\Halfspaces$ such that
    $
        \pr_{(\x,y)\sim \bar{S}}[y\neq h(\x)] \le \opt_{\bar S} + \eps\,,\,\text{ where }\opt_{\bar S} = \min_{f\in \Halfspaces} \pr_{(\x,y)\sim \bar{S}}[y\neq f(\x)].
    $
\end{theorem}

The final part of \Cref{theorem:main-result} states that the guarantee we achieve is actually stronger than the one in \Cref{definition:testable-noise}, since the output is near-optimal whenever the algorithm accepts, without requiring that the input $\bar S$ consists of independent examples. For small $c$, the runtime of our algorithm scales as $(d/\eps)^{\poly(1/c)}$. This gives a polynomial-time algorithm when $c$ is constant, but we leave it as an interesting open question whether the dependence on $1/c$ can be improved. We provide lower bounds for the case $c=0$ in \Cref{section:lower-bounds}.

The proof of \Cref{theorem:main-result} follows the same outline we provided for the case of random classification noise. However, there are two differences. First, we need a version of \Cref{fact:rcn-parameter-recovery} that works under Massart noise (and Gaussian marginal) and gives an algorithm that approximately recovers the parameters of the ground truth. Second, the disagreement tester from before does not give a testable bound for the quantity $|\bar S_b|/|\bar S| = \pr_{(\x,y)\sim \bar S}[y \neq \sign(\vv\cdot \x) \text{ and } \sign(\vv^*\cdot \x) \neq \sign(\vv\cdot \x)]$ anymore. Even if we assume once more that under the target assumption we have exact recovery (i.e., $\vv = \vv^*$), the event $y\neq \sign(\vv^*\cdot \x)$ is not independent from $\x$ and, if we used the disagreement tester of \Cref{theorem:disagreement-tester-up}, the completeness criterion would not necessarily be satisfied under the target assumption.

Fortunately, the first difference is not an issue, since appropriate results are known from prior work on classical learning under Massart noise and Gaussian marginal (see, e.g., \cite{awasthi2017power,diakonikolas2020learning}). 

\begin{fact}[\cite{diakonikolas2020learning}]\label{fact:massart-parameter-recovery}
    For any $\eps', \delta\in (0,1)$ and $\eta_0 = 1/2 - c$, where $c>0$ is any constant, there is an algorithm with time and sample complexity $\poly(d,1/\eps)\log(1/\delta)$ that has access to an example oracle in $\Massart_{\Gauss, f^*, \eta_0}$ for some unknown $f^*(\x) = \sign(\vv^*\cdot \x)$, $\vv^*\in \S^{d-1}$ and outputs $\vv\in\S^{d-1}$ s.t. $\measuredangle(\vv^*,\vv) \le \eps'$, with probability at least $1-\delta$.
\end{fact}

\noindent\textbf{The spectral tester and how it is applied.} In order to resolve the second complication and provide a certificate bounding the quantity $|\bar S_b|/|\bar S|$, we follow a different testing approach, based on ideas from tolerant testable learning \cite{goel2024tolerant}, where the testers must accept whenever the input distribution is close to the target (and not necessarily equal). 
We provide the following tester which is guaranteed to accept subsets of Gaussian samples, since it is monotone under datapoint removal.

\begin{theorem}[Spectral tester, see \Cref{thm: spectral tester}]\label{theorem:spectral-tester-up}
    Let $\mu \in (0,1)$ be any constant. There is an algorithm (\Cref{algorithm:spectral-tester}) that receives $\eps,\delta\in(0,1)$, $U\in \nats$, $\vv\in \S^{d-1}$ and a set $S$ of points in $\R^d$, runs in time $\poly(d,1/\eps,|S|)$ and then either outputs $\reject$ or $\accept$, satisfying the following specifications.
    \begin{enumerate}
        \item (Soundness) If the algorithm accepts and $|S| \le U$, then the following is true for any $\vv'\in\S^{d-1}$
        \[
            \frac{1}{U}\sum_{\x\in S}\ind\{\sign(\vv\cdot \x)\neq \sign(\vv'\cdot \x)\} \le (1+\mu) {\measuredangle(\vv,\vv')}/{\pi} + \eps
        \]
        \item (Completeness) If $S$ consists of at least $(\frac{Cd}{\eps \delta})^C$ i.i.d. examples from $\Gauss_d$, where $C\ge 1$ is some sufficiently large constant depending on $\mu$, then the algorithm accepts with probability at least $1-\delta$.
        \item (Monotonicity under removal) If the algorithm accepts on input $(\eps,\delta,U,\vv,S)$ and $S'$ is such that $S'\subseteq{S}$, then the algorithm also accepts on input $(\eps,\delta,U,\vv,S')$.
    \end{enumerate}
\end{theorem}

Given this tool, we are able to obtain a testable bound for $|\bar S_b| / |\bar S|$. Recall that the set $\bar S_b$ is the set of points $(\x,y)$ in $\bar S$ such that $y\neq \sign(\vv\cdot \x)$ and $\sign(\vv^*\cdot \x) \neq \sign(\vv\cdot \x)$. For the soundness, observe that $|\bar S_b| / |\bar S| = \frac{U}{|\bar S|} \cdot \frac{1}{U} \sum_{\x\in S_{\mathsf{False}}} \ind\{ \sign(\vv\cdot \x)\neq \sign(\vv^*\cdot \x) \}$, where $\bar S_{\mathsf{False}} = \{(\x,y)\in\bar S: y\neq \sign(\vv\cdot \x)\}$. The soundness condition of \Cref{theorem:spectral-tester-up} gives us that $|\bar S_b| / |\bar S| \le \frac{U}{|\bar S|} (1+\mu)\measuredangle(\vv,\vv^*)+\eps$, as long as $|S_{\mathsf{False}}| \le U$. The quantity $S_{\mathsf{False}}$ can be testably bounded by $(1/2-c)|\bar S|$, since $S_{\mathsf{False}}$ is defined with respect to $\vv$ and we can therefore pick $U = (1/2-c) |\bar S|$. Overall, we obtain the same bound for $|\bar S_b|/ |\bar S|$ as in the RCN case. 

In order to show that our test will accept under the target assumption, the main observation is that we can interpret the Massart noise oracle with noise rate $\eta_0$ as follows: To form the input set $\bar S$, the oracle first calls the RCN oracle of rate $\eta_0$ to form a set $\bar S^{\mathsf{RCN}}$. Let $\bar S_{\mathsf{False}}^{\mathsf{RCN}}$ be the subset of $\bar S^{\mathsf{RCN}}$ such that $y \neq \sign(\vv^*\cdot \x)$. The Massart noise oracle then flips the labels of some elements $\bar S_{\mathsf{False}}^{\mathsf{RCN}}$ back to match the ground-truth label. In other words, we have that $\bar S_{\mathsf{False}} \subseteq \bar S_{\mathsf{False}}^{\mathsf{RCN}}$ (assuming $\vv = \vv^*$). Observe that $S_{\mathsf{False}}^{\mathsf{RCN}}$ is drawn according to the distribution of $\x$ conditioned on $y\neq \sign(\vv^*\cdot \x)$ and is, therefore, an i.i.d. Gaussian sample.

\begin{algorithm}[htbp]
\caption{Spectral tester}\label{algorithm:spectral-tester}
\KwIn{$\eps,\delta,\mu\in(0,1)$, $\vv\in\S^{d-1}$ and set $S$ of points in $\R^d$}
Let $C\ge 1$ be a sufficiently large constant \\
Set $K = \frac{2}{\eps}\sqrt{\log (2/\eps)}$, $k = C/\mu^5$ and $\Delta = \frac{\eps^2}{CKd^{Ck}}$\\
    \For{$i = -K,-K+1, \dots, 0, 1, \dots, K-1$}{
        Let $\mathcal{I}_i(\x) = \ind\{i\eps \le \vv\cdot \x < (i+1)\eps\}$ for all $\x\in S$\\
        \textbf{if} $\E_{\x\sim S}[(\x^{\otimes k})(\x^{\otimes k})^\top \mathcal{I}_i(\x)] \preceq  \E_{\x\sim \Gauss}[(\x^{\otimes k})(\x^{\otimes k})^\top \mathcal{I}_i(\x)] + \Delta I$ \textbf{then} continue \\
        \textbf{else} output $\reject$
    }
    \textbf{if} $\E_{\x\sim S}[(\x^{\otimes k})(\x^{\otimes k})^\top \ind_{\{\vv\cdot\x \ge K\eps\}}] \preceq  \E_{\x\sim \Gauss}[(\x^{\otimes k})(\x^{\otimes k})^\top \ind_{\{\vv\cdot\x \ge K\eps\}}] + \Delta I$ \textbf{then} continue \\
    \textbf{else} output $\reject$ \\
    \textbf{if} $\E_{\x\sim S}[(\x^{\otimes k})(\x^{\otimes k})^\top \ind_{\{\vv\cdot\x \le -K\eps\}}] \preceq  \E_{\x\sim \Gauss}[(\x^{\otimes k})(\x^{\otimes k})^\top \ind_{\{\vv\cdot\x \le -K\eps\}}] + \Delta I$ \textbf{then} continue \\
    \textbf{else} output $\reject$\\
Output $\accept$
\end{algorithm}
\vspace{.5em}\noindent\textbf{Designing the spectral tester.} \Cref{theorem:spectral-tester-up} follows from a combination of ideas used to prove \Cref{theorem:disagreement-tester-up} and the spectral testing approach of \cite{goel2024tolerant}. In particular, instead of matching the Chow parameters $\E_{\x\sim S}[\mathcal{I}_i(\x) \prod_{j\in[d]}\x^{\alpha_i}]$ of the quantities $\mathcal{I}_i$ as in \Cref{algorithm:disagreement-tester}, we bound the maximum singular value of the Chow parameter matrices $\E_{\x\sim S}[(\x^{\otimes k})(\x^{\otimes k})^\top \mathcal{I}_i(\x)]$, where $\x^{\otimes k}$ denotes the vector of monomials of degree at most $k$. This can be done efficiently via the SVD algorithm and, crucially, satisfies the monotonicity under removal property of \Cref{theorem:spectral-tester-up}. Moreover, once the Chow parameter matrix is bounded, we have a bound for all quantities of the form $\E_{\x\sim S}[(p(\x))^2 \mathcal{I}_i(\x)]$, where $p$ is of degree at most $k$. Combining this observation with an analysis similar to the one for \Cref{lemma:chow-matching-main} (see \Cref{prop: sandwiching wrt covering implies fooling-1,thm: sandwiching degree of purple-1}) and a stronger version of the sandwiching polynomials of \cite{diakonikolas2010bounded,gopalan2010fooling} by \cite{klivans2023testable}, we obtain that \Cref{algorithm:spectral-tester} satisfies \Cref{theorem:spectral-tester-up}.

\vspace{.5em}\noindent\textbf{Overall algorithm.} The overall algorithm receives an input set of labeled examples $\bar S$ and obtains a candidate $\vv\in\S^{d-1}$ by running the algorithm of \cite{diakonikolas2020learning} with parameter $\eps' = \eps^{3/2}/(C\sqrt{d})$ for some large enough constant $C$ (see \Cref{fact:massart-parameter-recovery}). Then, it runs the disagreement tester of \Cref{theorem:disagreement-tester-up} with parameters $(S, \vv,\eps,\delta, \mu)$ (for some small constant $\mu$ depending on the noise rate $\eta_0 = 1/2-c$). Subsequently, the tester checks whether $\pr_{(\x,y)\sim \bar S}[y\neq \sign(\vv\cdot \x)]$ is at most $1/2-c/2$. 

Finally, it splits the set $\bar S_{\mathsf{False}} = \{(\x,y)\in \bar S: y\neq \sign(\vv\cdot \x)\}$ in two parts as follows.
\begin{align*}
    \bar S_{\mathsf{False}}^{\mathsf{far}} &= \Bigr\{(\x,y)\in \bar S_{\mathsf{False}}: |\measuredangle(\vv,\x) - {\pi}/{2}| > {\eps^{3/2}}/{(d-1)^{1/2}} \Bigr\} \;\;\;\text{ and }\;\;\;
    \bar S_{\mathsf{False}}^{\mathsf{near}} = \bar S_{\mathsf{False}} \setminus\bar S_{\mathsf{False}}^{\mathsf{far}}
\end{align*}
For $\bar S_{\mathsf{False}}^{\mathsf{near}}$, it checks that it contains at most $O(\eps) |\bar S|$ elements, while for $\bar S_{\mathsf{False}}^{\mathsf{far}}$, it runs the spectral tester of \Cref{theorem:spectral-tester-up} with inputs $(U=(1/2-c/2)|\bar S|,S= S_{\mathsf{False}}^{\mathsf{far}},\vv,\eps,\delta,\mu)$.

\section{Lower Bounds in the High-Noise Regime}\label{section:lower-bounds}

\noindent\textbf{Notation.} The set $\Z_q$ equals to $\{0,1,2,\dots,q-1\}$. We denote with $\Gauss_d(\mu,\Sigma;S)$ the Gaussian distribution in $d$ dimensions with mean $\mu\in\R^d$ and covariance matrix $\Sigma\in\R^{d\times d}$, truncated on the set $S\subseteq \R^d$. 

\vspace{.5em}

We show that there is no efficient tester-learner that accepts whenever the input dataset is generated by Gaussian examples with random classification noise (RCN) of rate $1/2$. We give both cryptographic lower bounds, as well as lower bounds in the statistical query (SQ) framework that match the best known bounds for classical (non-testable) learning under adversarial label noise. Since RCN noise is a special type of Massart noise, where all of the labels are flipped with the same rate (i.e., $\eta(\x)$ is constant), the lower bounds we give also imply lower bounds for the case of Massart noise of rate $1/2$ (which is also called \emph{strong} Massart noise). Recall that random classification noise is defined in \Cref{definition:rcn}.

The hard distributions for learning under adversarial label noise proposed by \cite{diakonikolas2021optimality,tiegel2023hardness,diakonikolas2023near} are all indistinguishable from the distribution generated by the oracle $\RCN_{\Gauss,1/2}$. Using this fact, we obtain our lower bounds by the following simple observation that any tester-learner that accepts $\RCN_{\Gauss,1/2}$ can distinguish between $\RCN_{\Gauss,1/2}$ and any distribution where the value of $\opt$ is non-trivial.

\begin{observation}\label{observation:distinguishing-via-testable-learning}
    Let $\H\subseteq\{\R^d \to \cube{}\}$ be a concept class, $\tau \in (0,1/8)$ and suppose that algorithm $\A$ testably learns $\H$ with respect to $\RCN_{\Gauss,1/2}$ up to excess error $\eps\in(0,1/4)$ and probability of failure $\delta = 1/6$. Let $\mathfrak{D}_g$ be the class of distributions over $\R^d\times\cube{}$ such that the marginal on $\R^d$ is $\Gauss_d$ and $\min_{f\in \H}\pr[y\neq f(\x)] \le \frac{1}{2} - \eps - 2\tau$. Then, there is an algorithm $\A'$ that calls $\A$ once and uses additional time $\poly(d,1/\tau)$ such that $|\pr[\A'(\Gauss_d \times \unif(\cube{})) = 1] - \pr[\A'(\D_{\x,y}) = 1]| \ge 1/3$ for any $\D_{\x,y}\in \mathfrak{D}_g$.
\end{observation}

\begin{proof}
    Let $\D_{\x,y}$ be the input distribution. The algorithm $\A'(\D_{\x,y})$ calls $\A(\D_{\x,y})$ once and then:
    \begin{itemize}
        \item If $\A$ outputs $\reject$, then $\A'$ outputs $0$.
        \item If $\A$ outputs $(\accept,h)$, then $\A'$ estimates the quantity $q = \pr_{(\x,y)\sim \D_{\x,y}}[y\neq h(\x)]$ up to tolerance $\tau$ and probability of failure $1/6$ and outputs $1$ if the estimate $q$ is at least $1/2-\tau$ and $0$ otherwise.
    \end{itemize}

    We now consider the case that $\D_{\x,y} = \Gauss_d\times \unif(\cube{})$. According to \Cref{definition:testable-noise}, the probability that $\A$ accepts is at least $5/6$. Moreover, we have that regardless of the choice of $h$, $\pr_{(\x,y)\sim \D_{\x,y}}[y\neq h(\x)] = 1/2$ and therefore $\A'$ will overall output $1$ with probability at least $2/3$.

    In the case that $\D_{\x,y} \in \mathfrak{D}_g$, $\A'$ will output $0$ unless the guarantee of the soundness does not hold (which happens with probability at most $1/6$) or the error of estimation of $q$ is more than $\tau$ (which happens with probability at most $1/6$. Hence, overall, $\A'$ will output $1$ with probability at most $1/3$.
\end{proof}

\subsection{Cryptographic Hardness}\label{section:crypto-hardness}

We provide cryptographic lower bounds based on the widely-believed hardness of the problem of learning with errors (LWE), which was introduced by \cite{regev2009lattices} and is defined as follows.

\begin{definition}[Learning with Errors]\label{definition:lwe}
    Let $d, q, m \in \nats$ and $\sigma>0$. The LWE problem with parameters $d,q,m,\sigma$ and advantage $\alpha \in (0,1)$ is defined as follows. Let $\mathbf{s}\sim \unif(\Z_q^d)$ and consider the following distributions over $\Z_q^d\times \R$.
    \begin{itemize}
        \item $\D_{\mathsf{null}}$: $\x\sim \unif(\Z_q^d)$ and $y\sim\unif(\Z_q)$.
        \item $\D_{\mathsf{alt}}$: $\x\sim \unif(\Z_q^d)$, $z\sim \Gauss_1(0,\sigma^2;\Z)$, $y = (\x\cdot \mathbf{s} + z) \mod q$
    \end{itemize}
    We receive $m$ i.i.d. examples from some distribution $\D_{\x,y}$ over $\Z_q^d\times \R$ which is either equal to $\D_{\mathsf{null}}$ or $\D_{\mathsf{alt}}$ and we are asked to output $v\in\cube{}$ such that $|\pr[v = 1 | \D_{\x,y} = \D_{\mathsf{null}}] - \pr[v = 1 | \D_{\x,y} = \D_{\mathsf{alt}}]| \ge \alpha$.
\end{definition}

There is strong evidence that the LWE problem cannot be solved in subexponential time, since there are quantum reductions from worst-case lattice problems \cite{regev2009lattices,peikert2009public}.

\begin{assumption}[Hardness of LWE]\label{assumption:lwe-hardness}
    Let $d,q,m\in \nats$ and $\sigma >0$ such that $q\le d^k$, $\sigma = c\sqrt{d}$ and $m = 2^{O(d^{\gamma})}$, where $\gamma\in(0,1)$, $k\in \nats$ are arbitrary constants and $c>0$ is a sufficiently large constant. Then, any algorithm that solves LWE with parameters $d,q,m,\sigma$ and advantage $2^{-O(d^\gamma)}$ requires time $2^{\Omega(d^\gamma)}$.
\end{assumption}

As an immediate corollary of results in \cite{diakonikolas2023near} (combined with \Cref{observation:distinguishing-via-testable-learning}), we obtain the following lower bound under \Cref{assumption:lwe-hardness}.

\begin{theorem}[Cryptographic Hardness in High-Noise Regime, Theorem 3.1 in \cite{diakonikolas2023near}]\label{theorem:crypto-hardness}
    Under \Cref{assumption:lwe-hardness}, every algorithm with the guarantees of $\A'$ in \Cref{observation:distinguishing-via-testable-learning} for $\tau = \eps \le 1/\log^{1/2+\beta}(d)$ and $\H = \Halfspaces$, requires time $\min\{d^{\Omega(1/(\eps\sqrt{\log(d)})^{\alpha})}, 2^{d^{0.99}}\}$, where $\alpha, \beta \in (0,2)$ are arbitrary constants. 
    
    Therefore, the same is true for any testable learning algorithm for $\Halfspaces$ with respect to the RCN oracle with noise rate $\eta_0 = 1/2$ that has excess error $\eps\le 1/\log^{1/2+\beta}(d)$ and failure probability $\delta\le 1/6$.
\end{theorem}

\subsection{SQ Lower Bounds}\label{section:sq-lower-bounds}

We also give lower bounds in the statistical query (SQ) model, which was originally defined by \cite{kearns1998efficient}. The SQ framework captures most of the usual algorithmic techniques like moment methods and gradient descent (\cite{feldman2017statistical,feldman2017statisticalquery}), and there is a long line of works in computational learning theory giving SQ lower bounds for various learning tasks.

\begin{definition}[Statistical Query Model]\label{definition:sq}
    Let $\D_{\x,y}$ be a distribution over $\R^d\times \cube{}$ and $\tau>0$. A statistical query (SQ) algorithm $\A$ with tolerance $\tau$ has access to $\D_{\x,y}$ as follows: The algorithm (adaptively) makes bounded queries of the form $q:\R^d\times[-1,1]\to[-1,1]$. For each query $q$, the algorithm receives a value $v\in\R$ with $|v-\E_{\x\sim\Dgeneric}[q(\x,y)]| \le \tau$.
\end{definition}

We obtain our lower bound as an immediate corollary of results in \cite{diakonikolas2023near}, combined with \Cref{observation:distinguishing-via-testable-learning}, where note that the reduction of the hard distinguishing problem to testable learning also works in the SQ framework, using one statistical query with tolerance $\tau$.

\begin{theorem}[SQ Lower Bound in High-Noise Regime, Propositions 2.1, 2.8, Corollary B.1 in \cite{diakonikolas2021optimality}]\label{theorem:sq-lower-bound}
    Every SQ algorithm with the guarantees of $\A'$ in \Cref{observation:distinguishing-via-testable-learning} for $\tau = \eps \ge d^{-c}$ and $\H = \Halfspaces$, where $c>0$ is a sufficiently small constant, either requires queries of tolerance $d^{-\Omega(1/\eps^2)}$ or makes $2^{d^{\Omega(1)}}$ queries.

    Therefore, the same is true for any SQ testable learning algorithm for $\Halfspaces$ with respect to the RCN oracle with noise rate $\eta_0 = 1/2$ that has excess error $\eps\ge d^{-c}$ and failure probability $\delta\le 1/6$.
\end{theorem}

\section*{Acknowledgments.}

We thank Aravind Gollakota for the initial observation that known hardness results rule out poly-time assumption testers for the RCN model 
in the high-noise regime. We also thank Vasilis Kontonis for insightful conversations.

\bibliographystyle{alpha}
\bibliography{refs}

@INPROCEEDINGS{chow-parameters,
  author={Chow, C. K.},
  booktitle={2nd Annual Symposium on Switching Circuit Theory and Logical Design (SWCT 1961)}, 
  title={On the characterization of threshold functions}, 
  year={1961},
  volume={},
  number={},
  pages={34-38},
  doi={10.1109/FOCS.1961.24}}

@article{gollakota2023tester,
  title={Tester-learners for halfspaces: Universal algorithms},
  author={Gollakota, Aravind and Klivans, Adam and Stavropoulos, Konstantinos and Vasilyan, Arsen},
  journal={Advances in Neural Information Processing Systems},
  volume={36},
  year={2023}
}

@article{diakonikolas2010bounded,
	author = {Diakonikolas, Ilias and Gopalan, Parikshit and Jaiswal, Ragesh and Servedio, Rocco A and Viola, Emanuele},
	journal = {SIAM Journal on Computing},
	number = {8},
	pages = {3441--3462},
	publisher = {SIAM},
	title = {Bounded independence fools halfspaces},
	volume = {39},
	year = {2010}}

@article{kalai2008agnostically,
	author = {Kalai, Adam Tauman and Klivans, Adam R and Mansour, Yishay and Servedio, Rocco A},
	journal = {SIAM Journal on Computing},
	number = {6},
	pages = {1777--1805},
	publisher = {SIAM},
	title = {Agnostically learning halfspaces},
	volume = {37},
	year = {2008}}

@article{rubinfeld2022testing,
	author = {Rubinfeld, Ronitt and Vasilyan, Arsen},
	journal = {Proceedings of the fifty-fifth annual ACM Symposium on Theory of Computing},
	title = {Testing distributional assumptions of learning algorithms},
	year = {2023}}

@article{angluin1988learning,
  title={Learning from noisy examples},
  author={Angluin, Dana and Laird, Philip},
  journal={Machine learning},
  volume={2},
  pages={343--370},
  year={1988},
  publisher={Springer}
}

@article{diakonikolas2020near,
  title={Near-optimal sq lower bounds for agnostically learning halfspaces and relus under gaussian marginals},
  author={Diakonikolas, Ilias and Kane, Daniel and Zarifis, Nikos},
  journal={Advances in Neural Information Processing Systems},
  volume={33},
  pages={13586--13596},
  year={2020}
}

@inproceedings{diakonikolas2021optimality,
  title={The optimality of polynomial regression for agnostic learning under gaussian marginals in the SQ model},
  author={Diakonikolas, Ilias and Kane, Daniel M and Pittas, Thanasis and Zarifis, Nikos},
  booktitle={Conference on Learning Theory},
  pages={1552--1584},
  year={2021},
  organization={PMLR}
}

@article{goel2020statistical,
  title={Statistical-query lower bounds via functional gradients},
  author={Goel, Surbhi and Gollakota, Aravind and Klivans, Adam},
  journal={Advances in Neural Information Processing Systems},
  volume={33},
  pages={2147--2158},
  year={2020}
}

@article{gollakota2022moment,
  title={A Moment-Matching Approach to Testable Learning and a New Characterization of Rademacher Complexity},
  author={Gollakota, Aravind and Klivans, Adam R and Kothari, Pravesh K},
  journal={Proceedings of the fifty-fifth annual ACM Symposium on Theory of Computing},
  year={2023}
}

@article{diakonikolas2020non,
  title={Non-convex SGD learns halfspaces with adversarial label noise},
  author={Diakonikolas, Ilias and Kontonis, Vasilis and Tzamos, Christos and Zarifis, Nikos},
  journal={Advances in Neural Information Processing Systems},
  volume={33},
  pages={18540--18549},
  year={2020}
}

@inproceedings{diakonikolas2020learning,
  title={Learning halfspaces with massart noise under structured distributions},
  author={Diakonikolas, Ilias and Kontonis, Vasilis and Tzamos, Christos and Zarifis, Nikos},
  booktitle={Conference on Learning Theory},
  pages={1486--1513},
  year={2020},
  organization={PMLR}
}

@article{awasthi2017power,
  title={The power of localization for efficiently learning linear separators with noise},
  author={Awasthi, Pranjal and Balcan, Maria Florina and Long, Philip M},
  journal={Journal of the ACM (JACM)},
  volume={63},
  number={6},
  pages={1--27},
  year={2017},
  publisher={ACM New York, NY, USA}
}

@inproceedings{daniely2015ptas,
  title={A PTAS for agnostically learning halfspaces},
  author={Daniely, Amit},
  booktitle={Conference on Learning Theory},
  pages={484--502},
  year={2015},
  organization={PMLR}
}

@inproceedings{diakonikolas2018learning,
  title={Learning geometric concepts with nasty noise},
  author={Diakonikolas, Ilias and Kane, Daniel M and Stewart, Alistair},
  booktitle={Proceedings of the 50th Annual ACM SIGACT Symposium on Theory of Computing},
  pages={1061--1073},
  year={2018}
}

@inproceedings{awasthi2015efficient,
  title={Efficient learning of linear separators under bounded noise},
  author={Awasthi, Pranjal and Balcan, Maria-Florina and Haghtalab, Nika and Urner, Ruth},
  booktitle={Conference on Learning Theory},
  pages={167--190},
  year={2015},
  organization={PMLR}
}

@inproceedings{awasthi2016learning,
  title={Learning and 1-bit compressed sensing under asymmetric noise},
  author={Awasthi, Pranjal and Balcan, Maria-Florina and Haghtalab, Nika and Zhang, Hongyang},
  booktitle={Conference on Learning Theory},
  pages={152--192},
  year={2016},
  organization={PMLR}
}

@inproceedings{mangoubi2019nonconvex,
  title={Nonconvex sampling with the Metropolis-adjusted Langevin algorithm},
  author={Mangoubi, Oren and Vishnoi, Nisheeth K},
  booktitle={Conference on learning theory},
  pages={2259--2293},
  year={2019},
  organization={PMLR}
}

@article{diakonikolas2019distribution,
  title={Distribution-independent pac learning of halfspaces with massart noise},
  author={Diakonikolas, Ilias and Gouleakis, Themis and Tzamos, Christos},
  journal={Advances in Neural Information Processing Systems},
  volume={32},
  year={2019}
}

@inproceedings{diakonikolas2022learning_general,
  title={Learning general halfspaces with general massart noise under the gaussian distribution},
  author={Diakonikolas, Ilias and Kane, Daniel M and Kontonis, Vasilis and Tzamos, Christos and Zarifis, Nikos},
  booktitle={Proceedings of the 54th Annual ACM SIGACT Symposium on Theory of Computing},
  pages={874--885},
  year={2022}
}

@article{chen2020classification,
  title={Classification under misspecification: Halfspaces, generalized linear models, and evolvability},
  author={Chen, Sitan and Koehler, Frederic and Moitra, Ankur and Yau, Morris},
  journal={Advances in Neural Information Processing Systems},
  volume={33},
  pages={8391--8403},
  year={2020}
}

@article{diakonikolas2021forster,
  title={Forster decomposition and learning halfspaces with noise},
  author={Diakonikolas, Ilias and Kane, Daniel and Tzamos, Christos},
  journal={Advances in Neural Information Processing Systems},
  volume={34},
  pages={7732--7744},
  year={2021}
}

@article{zhang2020efficient,
  title={Efficient active learning of sparse halfspaces with arbitrary bounded noise},
  author={Zhang, Chicheng and Shen, Jie and Awasthi, Pranjal},
  journal={Advances in Neural Information Processing Systems},
  volume={33},
  pages={7184--7197},
  year={2020}
}

@inproceedings{zhang2021improved,
  title={Improved algorithms for efficient active learning halfspaces with massart and tsybakov noise},
  author={Zhang, Chicheng and Li, Yinan},
  booktitle={Conference on Learning Theory},
  pages={4526--4527},
  year={2021},
  organization={PMLR}
}

@inproceedings{zhang2017hitting,
  title={A hitting time analysis of stochastic gradient langevin dynamics},
  author={Zhang, Yuchen and Liang, Percy and Charikar, Moses},
  booktitle={Conference on Learning Theory},
  pages={1980--2022},
  year={2017},
  organization={PMLR}
}

@article{yan2017revisiting,
  title={Revisiting perceptron: Efficient and label-optimal learning of halfspaces},
  author={Yan, Songbai and Zhang, Chicheng},
  journal={Advances in Neural Information Processing Systems},
  volume={30},
  year={2017}
}

@inproceedings{gopalan2010fooling,
  title={Fooling functions of halfspaces under product distributions},
  author={Gopalan, Parikshit and O'Donnell, Ryan and Wu, Yi and Zuckerman, David},
  booktitle={2010 IEEE 25th Annual Conference on Computational Complexity},
  pages={223--234},
  year={2010},
  organization={IEEE}
}

@article{gollakota2023efficient,
  title={An Efficient Tester-Learner for Halfspaces},
  author={Gollakota, Aravind and Klivans, Adam R and Stavropoulos, Konstantinos and Vasilyan, Arsen},
  journal={The Twelfth International Conference on Learning Representations},
  year={2024}
}

@inproceedings{diakonikolas2023near,
  title={Near-optimal cryptographic hardness of agnostically learning halfspaces and relu regression under gaussian marginals},
  author={Diakonikolas, Ilias and Kane, Daniel and Ren, Lisheng},
  booktitle={International Conference on Machine Learning},
  pages={7922--7938},
  year={2023},
  organization={PMLR}
}

@InProceedings{klivans2023testable,
  title = 	 {Testable Learning with Distribution Shift},
  author =       {Klivans, Adam and Stavropoulos, Konstantinos and Vasilyan, Arsen},
  booktitle = 	 {Proceedings of Thirty Seventh Conference on Learning Theory},
  pages = 	 {2887--2943},
  year = 	 {2024},
  editor = 	 {Agrawal, Shipra and Roth, Aaron},
  volume = 	 {247},
  series = 	 {Proceedings of Machine Learning Research},
  month = 	 {30 Jun--03 Jul},
  publisher =    {PMLR},
  url = 	 {https://proceedings.mlr.press/v247/klivans24a.html},
  }

@InProceedings{klivans2024learning,
  title = 	 {Learning Intersections of Halfspaces with Distribution Shift: Improved Algorithms and SQ Lower Bounds},
  author =       {Klivans, Adam and Stavropoulos, Konstantinos and Vasilyan, Arsen},
  booktitle = 	 {Proceedings of Thirty Seventh Conference on Learning Theory},
  pages = 	 {2944--2978},
  year = 	 {2024},
  editor = 	 {Agrawal, Shipra and Roth, Aaron},
  volume = 	 {247},
  series = 	 {Proceedings of Machine Learning Research},
  month = 	 {30 Jun--03 Jul},
  publisher =    {PMLR},
  url = 	 {https://proceedings.mlr.press/v247/klivans24b.html}
}

@inproceedings{o2008chow,
  title={The chow parameters problem},
  author={O'Donnell, Ryan and Servedio, Rocco A},
  booktitle={Proceedings of the fortieth annual ACM symposium on Theory of computing},
  pages={517--526},
  year={2008}
}

@article{diakonikolas2023efficient,
  title={Efficient testable learning of halfspaces with adversarial label noise},
  author={Diakonikolas, Ilias and Kane, Daniel and Kontonis, Vasilis and Liu, Sihan and Zarifis, Nikos},
  journal={Advances in Neural Information Processing Systems},
  volume={36},
  year={2023}
}

@article{gollakota2024agnostically,
  title={Agnostically learning single-index models using omnipredictors},
  author={Gollakota, Aravind and Gopalan, Parikshit and Klivans, Adam and Stavropoulos, Konstantinos},
  journal={Advances in Neural Information Processing Systems},
  volume={36},
  year={2024}
}

@article{chandrasekaran2024efficient,
  title={Efficient Discrepancy Testing for Learning with Distribution Shift},
  author={Chandrasekaran, Gautam and Klivans, Adam R and Kontonis, Vasilis and Stavropoulos, Konstantinos and Vasilyan, Arsen},
  journal={Advances in Neural Information Processing Systems (to Appear)},
  volume={37},
  year={2024}
}

@article{goel2024tolerant,
  title={Tolerant Algorithms for Learning with Arbitrary Covariate Shift},
  author={Goel, Surbhi and Shetty, Abhishek and Stavropoulos, Konstantinos and Vasilyan, Arsen},
  journal={arXiv preprint arXiv:2406.02742},
  year={2024}
}

@article{kearns1998efficient,
  title={Efficient noise-tolerant learning from statistical queries},
  author={Kearns, Michael},
  journal={Journal of the ACM (JACM)},
  volume={45},
  number={6},
  pages={983--1006},
  year={1998},
  publisher={ACM New York, NY, USA}
}

@inproceedings{diakonikolas2024testable,
  title={Testable Learning of General Halfspaces with Adversarial Label Noise},
  author={Diakonikolas, Ilias and Kane, Daniel and Liu, Sihan and Zarifis, Nikos},
  booktitle={The Thirty Seventh Annual Conference on Learning Theory},
  pages={1308--1335},
  year={2024},
  organization={PMLR}
}

@article{BSHOUTY2002255,
title = {PAC learning with nasty noise},
journal = {Theoretical Computer Science},
volume = {288},
number = {2},
pages = {255-275},
year = {2002},
note = {Algorithmic Learning Theory},
issn = {0304-3975},
doi = {https://doi.org/10.1016/S0304-3975(01)00403-0},
url = {https://www.sciencedirect.com/science/article/pii/S0304397501004030},
author = {Nader H. Bshouty and Nadav Eiron and Eyal Kushilevitz}
}

@inproceedings{diakonikolas2021agnostic,
  title={Agnostic proper learning of halfspaces under gaussian marginals},
  author={Diakonikolas, Ilias and Kane, Daniel M and Kontonis, Vasilis and Tzamos, Christos and Zarifis, Nikos},
  booktitle={Conference on Learning Theory},
  pages={1522--1551},
  year={2021},
  organization={PMLR}
}

@inproceedings{tiegel2023hardness,
  title={Hardness of agnostically learning halfspaces from worst-case lattice problems},
  author={Tiegel, Stefan},
  booktitle={The Thirty Sixth Annual Conference on Learning Theory},
  pages={3029--3064},
  year={2023},
  organization={PMLR}
}

@article{feldman2017statistical,
  title={Statistical algorithms and a lower bound for detecting planted cliques},
  author={Feldman, Vitaly and Grigorescu, Elena and Reyzin, Lev and Vempala, Santosh S and Xiao, Ying},
  journal={Journal of the ACM (JACM)},
  volume={64},
  number={2},
  pages={1--37},
  year={2017},
  publisher={ACM New York, NY, USA}
}

@inproceedings{feldman2017statisticalquery,
  title={Statistical query algorithms for mean vector estimation and stochastic convex optimization},
  author={Feldman, Vitaly and Guzman, Cristobal and Vempala, Santosh},
  booktitle={Proceedings of the Twenty-Eighth Annual ACM-SIAM Symposium on Discrete Algorithms},
  pages={1265--1277},
  year={2017},
  organization={SIAM}
}

@article{blum1998polynomial,
  title={A polynomial-time algorithm for learning noisy linear threshold functions},
  author={Blum, Avrim and Frieze, Alan and Kannan, Ravi and Vempala, Santosh},
  journal={Algorithmica},
  volume={22},
  pages={35--52},
  year={1998},
  publisher={Springer}
}

@inproceedings{cohen1997learning,
  title={Learning noisy perceptrons by a perceptron in polynomial time},
  author={Cohen, Edith},
  booktitle={Proceedings 38th Annual Symposium on Foundations of Computer Science},
  pages={514--523},
  year={1997},
  organization={IEEE}
}

@inproceedings{diakonikolas2023strongly,
  title={A strongly polynomial algorithm for approximate forster transforms and its application to halfspace learning},
  author={Diakonikolas, Ilias and Tzamos, Christos and Kane, Daniel M},
  booktitle={Proceedings of the 55th Annual ACM Symposium on Theory of Computing},
  pages={1741--1754},
  year={2023}
}

@article{diakonikolas2024near,
  title={Near-optimal bounds for learning gaussian halfspaces with random classification noise},
  author={Diakonikolas, Ilias and Diakonikolas, Jelena and Kane, Daniel and Wang, Puqian and Zarifis, Nikos},
  journal={Advances in Neural Information Processing Systems},
  volume={36},
  year={2024}
}

@article{massart2006risk,
  title={Risk bounds for statistical learning},
  author={Massart, Pascal and N{\'e}d{\'e}lec, {\'E}lodie},
  year={2006}
}

@inproceedings{bruna2021continuous,
  title={Continuous lwe},
  author={Bruna, Joan and Regev, Oded and Song, Min Jae and Tang, Yi},
  booktitle={Proceedings of the 53rd Annual ACM SIGACT Symposium on Theory of Computing},
  pages={694--707},
  year={2021}
}

@inproceedings{gupte2022continuous,
  title={Continuous lwe is as hard as lwe \& applications to learning gaussian mixtures},
  author={Gupte, Aparna and Vafa, Neekon and Vaikuntanathan, Vinod},
  booktitle={2022 IEEE 63rd Annual Symposium on Foundations of Computer Science (FOCS)},
  pages={1162--1173},
  year={2022},
  organization={IEEE}
}

@article{slot2024testably,
  title={Testably Learning Polynomial Threshold Functions},
  author={Slot, Lucas and Tiegel, Stefan and Wiedmer, Manuel},
  journal={Advances in Neural Information Processing Systems (to Appear)},
  volume={37},
  year={2024}
}

@article{regev2009lattices,
  title={On lattices, learning with errors, random linear codes, and cryptography},
  author={Regev, Oded},
  journal={Journal of the ACM (JACM)},
  volume={56},
  number={6},
  pages={1--40},
  year={2009},
  publisher={ACM New York, NY, USA}
}

@inproceedings{peikert2009public,
  title={Public-key cryptosystems from the worst-case shortest vector problem},
  author={Peikert, Chris},
  booktitle={Proceedings of the forty-first annual ACM symposium on Theory of computing},
  pages={333--342},
  year={2009}
}

@inproceedings{unsupervised_1,
	title = {Robust {Sparse} {Estimation} for {Gaussians} with {Optimal} {Error} under {Huber} {Contamination}},
	url = {https://proceedings.mlr.press/v235/diakonikolas24a.html},
	language = {en},
	urldate = {2025-04-02},
	booktitle = {Proceedings of the 41st {International} {Conference} on {Machine} {Learning}},
	publisher = {PMLR},
	author = {Diakonikolas, Ilias and Kane, Daniel and Karmalkar, Sushrut and Pensia, Ankit and Pittas, Thanasis},
	month = jul,
	year = {2024},
	note = {ISSN: 2640-3498},
	pages = {10811--10840},
}

@inproceedings{unsupervised_2,
	title = {Online and {Distribution}-{Free} {Robustness}: {Regression} and {Contextual} {Bandits} with {Huber} {Contamination}},
	shorttitle = {Online and {Distribution}-{Free} {Robustness}},
	url = {https://ieeexplore.ieee.org/abstract/document/9719757},
	doi = {10.1109/FOCS52979.2021.00072},
	urldate = {2025-04-02},
	booktitle = {2021 {IEEE} 62nd {Annual} {Symposium} on {Foundations} of {Computer} {Science} ({FOCS})},
	author = {Chen, Sitan and Koehler, Frederic and Moitra, Ankur and Yau, Morris},
	month = feb,
	year = {2022},
	note = {ISSN: 2575-8454},
	pages = {684--695},
}

@article{unsupervised_3,
	title = {Robust {Covariance} and {Scatter} {Matrix} {Estimation} {Under} {Huber}’s {Contamination} {Model}},
	volume = {46},
	issn = {0090-5364},
	url = {https://www.jstor.org/stable/26542852},
	number = {5},
	urldate = {2025-04-02},
	journal = {The Annals of Statistics},
	author = {Chen, Mengjie and Gao, Chao and Ren, Zhao},
	year = {2018},
	note = {Publisher: Institute of Mathematical Statistics},
	pages = {1932--1960},
}

@article{unsupervised_4,
	title = {Causal discovery with continuous additive noise models},
	volume = {15},
	issn = {1532-4435},
	number = {1},
	journal = {J. Mach. Learn. Res.},
	author = {Peters, Jonas and Mooij, Joris M. and Janzing, Dominik and Schölkopf, Bernhard},
	month = jan,
	year = {2014},
	pages = {2009--2053},
}

@inproceedings{unsupervised_5,
	address = {New York, NY, USA},
	series = {{STOC} '16},
	title = {How robust are reconstruction thresholds for community detection?},
	isbn = {978-1-4503-4132-5},
	url = {https://dl.acm.org/doi/10.1145/2897518.2897573},
	doi = {10.1145/2897518.2897573},
	urldate = {2025-04-02},
	booktitle = {Proceedings of the forty-eighth annual {ACM} symposium on {Theory} of {Computing}},
	publisher = {Association for Computing Machinery},
	author = {Moitra, Ankur and Perry, William and Wein, Alexander S.},
	month = jun,
	year = {2016},
	pages = {828--841},
}

@inproceedings{unsupervised_6,
	title = {Reducibility and {Statistical}-{Computational} {Gaps} from {Secret} {Leakage}},
	url = {https://proceedings.mlr.press/v125/brennan20a.html},
	language = {en},
	urldate = {2025-04-02},
	booktitle = {Proceedings of {Thirty} {Third} {Conference} on {Learning} {Theory}},
	publisher = {PMLR},
	author = {Brennan, Matthew and Bresler, Guy},
	month = jul,
	year = {2020},
	note = {ISSN: 2640-3498},
	pages = {648--847},
}

@inproceedings{unsupervised_7,
	address = {New York, NY, USA},
	series = {{STOC} '16},
	title = {Semidefinite programs on sparse random graphs and their application to community detection},
	isbn = {978-1-4503-4132-5},
	url = {https://dl.acm.org/doi/10.1145/2897518.2897548},
	doi = {10.1145/2897518.2897548},
	urldate = {2025-04-02},
	booktitle = {Proceedings of the forty-eighth annual {ACM} symposium on {Theory} of {Computing}},
	publisher = {Association for Computing Machinery},
	author = {Montanari, Andrea and Sen, Subhabrata},
	month = jun,
	year = {2016},
	pages = {814--827},
}

@inproceedings{unsupervised_8,
	title = {Robust recovery for stochastic block models},
	url = {https://ieeexplore.ieee.org/abstract/document/9719791},
	doi = {10.1109/FOCS52979.2021.00046},
	urldate = {2025-04-02},
	booktitle = {2021 {IEEE} 62nd {Annual} {Symposium} on {Foundations} of {Computer} {Science} ({FOCS})},
	author = {Ding, Jingqiu and D'Orsi, Tommaso and Nasser, Rajai and Steurer, David},
	month = feb,
	year = {2022},
	note = {ISSN: 2575-8454},
	pages = {387--394},
}

@inproceedings{unsupervised_9,
	address = {New York, NY, USA},
	series = {{STOC} 2023},
	title = {Algorithms {Approaching} the {Threshold} for {Semi}-random {Planted} {Clique}},
	isbn = {978-1-4503-9913-5},
	url = {https://dl.acm.org/doi/10.1145/3564246.3585184},
	doi = {10.1145/3564246.3585184},
	urldate = {2025-04-02},
	booktitle = {Proceedings of the 55th {Annual} {ACM} {Symposium} on {Theory} of {Computing}},
	publisher = {Association for Computing Machinery},
	author = {Buhai, Rares-Darius and Kothari, Pravesh K. and Steurer, David},
	month = jun,
	year = {2023},
	pages = {1918--1926},
}

@inproceedings{unsupervised_10,
	title = {Semi-{Random} {Sparse} {Recovery} in {Nearly}-{Linear} {Time}},
	url = {https://proceedings.mlr.press/v195/kelner23a.html},
	language = {en},
	urldate = {2025-04-02},
	booktitle = {Proceedings of {Thirty} {Sixth} {Conference} on {Learning} {Theory}},
	publisher = {PMLR},
	author = {Kelner, Jonathan and Li, Jerry and Liu, Allen X. and Sidford, Aaron and Tian, Kevin},
	month = jul,
	year = {2023},
	note = {ISSN: 2640-3498},
	pages = {2352--2398},
}

@inproceedings{unsupervised_11,
	title = {Non-{Convex} {Matrix} {Completion} {Against} a {Semi}-{Random} {Adversary}},
	url = {https://proceedings.mlr.press/v75/cheng18b.html},
	language = {en},
	urldate = {2025-04-02},
	booktitle = {Proceedings of the 31st  {Conference} {On} {Learning} {Theory}},
	publisher = {PMLR},
	author = {Cheng, Yu and Ge, Rong},
	month = jul,
	year = {2018},
	note = {ISSN: 2640-3498},
	pages = {1362--1394},
}

@inproceedings{unsupervised_12,
	title = {Semirandom {Planted} {Clique} and the {Restricted} {Isometry} {Property}},
	url = {https://ieeexplore.ieee.org/abstract/document/10756085},
	doi = {10.1109/FOCS61266.2024.00064},
	urldate = {2025-04-02},
	booktitle = {2024 {IEEE} 65th {Annual} {Symposium} on {Foundations} of {Computer} {Science} ({FOCS})},
	author = {Błasiok, Jarosław and Buhai, Rares-Darius and Kothari, Pravesh K. and Steurer, David},
	month = oct,
	year = {2024},
	note = {ISSN: 2575-8454},
	pages = {959--969},
}
\appendix

\section{Preliminaries}

\subsection{Some standard notation.}

\label{subsec: Some standard notation.}

When we say $a=b\pm c$ we mean that $a$ in in the interval $[b-c, b+c]$. When we say a polynomial is ``degree-$k$'' we mean that the degree
polynomial is \emph{at most} $k$.
For a vector $\vect x$ in $\R^{d}$,  let $\vect x^{\otimes k}$ denote
the $\binom{d+1}{k}$-dimensional vector whose elements are of the form $ \x^\coefficientvector = \prod_{j\in[d] }x_{j}^{\coefficientvector_{j}}$ for $ \coefficientvector\in\nats^d $ with $
\|\coefficientvector\|_1 = \sum_{j}\coefficientvector_{j}\leq k$. In other words,  $\vect x^{\otimes k}$ is the vector one gets by
evaluating all multidimensional monomials of degree at most $k$ on
input $\vect x$. We also view degree-$k$ polynomials as corresponding
to elements $\R^{\binom{d+1}{k}}$,  i.e. their coefficient vectors.
Using this notation we have 
\[
p(\vect x)=p\cdot\vect x^{\otimes k}.
\]

Additionally,  for a degree-$k$ polynomial $p$ over $\R^{d}$,  say $p(\x) = \sum_{\coefficientvector:\|\coefficientvector\|_1\le k} p_{\coefficientvector} \x^\coefficientvector$, 
we will use the notation $\norm p_{\mathrm{coeff}}$ to denote the
$2$-norm of the coefficients of $p$,  specifically
\[
    \|p\|_{\mathrm{coef}} := \Bigr(\sum_{\coefficientvector\in \nats^d} p_\coefficientvector^2\Bigr)^{1/2}
\]
Note that if the largest in absolute value coefficient of polynomial
$p$ has absolute value $B$,  then we have 
\begin{equation}
B\leq\norm p_{\mathrm{coeff}}\leq B\left(d+1\right)^{k/2}.\label{eq: coefficient L2 norm of polynomial vs largest coefficient}
\end{equation}
In this work we will use the convention that $\sign(0)=1$.

\subsection{Standard lemmas.}

We will also need the following lemma:
\begin{lemma}
\label{lem: concentration for moments}

Let $\mathcal{H}$ be a collection of subsets of $\R^{d}$  of non-zero VC dimension $\Delta_{\text{VC}}$ 
and let $S$ be a collection of $N$ i.i.d. samples from $\N(0, I_{d})$.
Then,  with probability at least $1-\delta$ for every polynomial $p$
of degree at most $k$ with coefficients bounded by $B$ in absolute
value and all sets $A$ in $\mathcal{H}$ we have
\[
\left|\E_{\vect x\sim S}\left[p(\x)\indicator_{\vect x\in A}\right]-\E_{\vect x\sim\N(0, I_{d})}\left[p(\x)\indicator_{\vect x\in A}\right]\right|
\leq
\frac{60(2k)^{k+2} (d+1)^{k}\Delta_{\text{VC}} }{\delta}
\left(
\frac{ \log N}{N }
\right)^{1/4}
, 
\]

\[
\left|\E_{\vect x\sim S}\left[\left(p(\x)\right)^{2}\indicator_{\vect x\in A}\right]-\E_{\vect x\sim\N(0, I_{d})}\left[\left(p(\x)\right)^{2}\indicator_{\vect x\in A}\right]\right|
\leq 
\frac{60 B^2 (4k)^{2k+2} (d+1)^{6k}\Delta_{\text{VC}}  }{\delta}
\left(
\frac{ \log N}{N }
\right)^{1/4}
.
\]
\end{lemma}

\begin{proof}
Let $R$ be a positive real number,  to be set later.
For any monomial $m$ over $\R^d$ of degree at most $k$,  we can decompose
\[
m(\vect x)\indicator_{\vect x\in A}
=
m(\vect x)\indicator_{\vect x\in A \land |m(\x)|\leq R}
\pm|m(\vect x)|\indicator_{|m(\x)|> R}
\]
This allows us to bound the quantity in our lemma in the following way:
\begin{multline}
    \label{eq: breaking error into three terms}
\left \lvert
\E_{\vect x\sim S}\left[m(\vect x)\indicator_{\vect x\in A}
\right]
-
\E_{\vect x\sim \N(0, I_d)}\left[m(\vect x)\indicator_{\vect x\in A}
\right]
\right \rvert
=\\
\left \lvert
\E_{\vect x\sim S}\left[|m(\vect x)|\indicator_{\vect x\in A \land |m(\x)|\leq R \land m(\x)>0}\right]
-\E_{\vect x\sim S}\left[|m(\vect x)|\indicator_{\vect x\in A \land |m(\x)|\leq R \land m(\x)<0}\right]
\right \rvert\\
\pm
\left(\E_{\vect x\sim S}\left[|m(\vect x)|\indicator_{ |m(\vect x)|>R}\right]
+
\E_{\vect x\sim \N(0,  I_d)}\left[|m(\vect x)|\indicator_{ |m(\vect x)|>R}\right]
\right)
\end{multline}
We start by considering the first term above:
\begin{multline}
\label{eq: first term in error}
\left \vert
\E_{\vect x\sim S}\left[m(\vect x)\indicator_{\vect x\in A \land |m(\x)|\leq R \land m(\x)>0}\right]
-\E_{\vect x\sim \N(0, I_d)}\left[m(\vect x)\indicator_{\vect x\in A \land |m(\x)|\leq R \land m(\x)>0}\right]
\right \rvert
=\\
\left \vert
\int_0^R
\pr_{\vect x\sim S}\left[m(\vect x)\indicator_{\vect x\in A \land |m(\x)|\leq R } \geq z\right] \d z
-
\int_0^R
\pr_{\vect x\sim \N(0, I_d)}\left[m(\vect x)\indicator_{\vect x\in A \land |m(\x)|\leq R } \geq z\right] \d z\right \rvert \leq \\
\int_0^R
\left \lvert
\pr_{\vect x\sim S}\left[m(\vect x)\indicator_{\vect x\in A \land |m(\x)|\leq R } \geq z\right]
-
\pr_{\vect x\sim \N(0, I_d)}\left[m(\vect x)\indicator_{\vect x\in A \land |m(\x)|\leq R} \geq z\right]\right \rvert 
\d z \leq \\
R \max_{z \in [0, R]} \left \lvert
\pr_{\vect x\sim S}\left[z \leq m(\vect x)\ \leq R \land \vect x\in A \right]
-
\pr_{\vect x\sim \N(0, I_d)}\left[ z \leq m(\vect x)\ \leq R \land \vect x\in A \right]\right \rvert 
\end{multline}
To bound the right side of Equation \ref{eq: first term in error},  consider the class $\mathcal{G}$ of $\{0, 1\}$-valued functions of the form $\indicator_{z \leq p(\vect x)\ \leq R \land \vect x\in A}$,  where $A$ is a set in $\mathcal{H}$ and $p$ is a  polynomial in $d$ dimensions of degree at most $d$. Recall that the VC dimension of degree-$k$ polynomial threshold functions is at most $(d+1)^{k}$.
From the Sauer-Shelah lemma,  it follows that the VC dimension of $\mathcal{G}$ is at most $10((d+1)^{k}+\Delta_{\text{VC}})$.  Combining it with the standard VC bound,  we see that with probability at least $1-\delta/4$ for all monomials $m$ of degree at most $k$ we have
\[
\left \lvert
\pr_{\vect x\sim S}\left[z \leq m(\vect x)\ \leq R \land \vect x\in A \right]
-
\pr_{\vect x\sim \N}\left[ z \leq m(\vect x)\ \leq R \land \vect x\in A \right]\right \rvert 
\leq 
\Bigr({
\frac{100((d+1)^{k}+\Delta_{\text{VC}}) \log N}{N \delta}
}\Bigr)^{\frac{1}{2}}
\]
Combining this with Equation \ref{eq: first term in error} we get:
\begin{multline}
\label{eq: first term final bound}
\left \vert
\E_{\vect x\sim S}\left[m(\vect x)\indicator_{\vect x\in A \land |m(\x)|\leq R \land m(\x)>0}\right]
-\E_{\vect x\sim \N(0, I_d)}\left[m(\vect x)\indicator_{\vect x\in A \land |m(\x)|\leq R \land m(\x)>0}\right]
\right \rvert
\leq\\
10R
\sqrt{
\frac{((d+1)^{k}+\Delta_{\text{VC}}) \log N}{N \delta}
}
\end{multline}
We now proceed to bounding the second term and the third terms in Equation \ref{eq: breaking error into three terms}. If we express $m(\vect x)$ as $\prod_{j}x_{j}^{i_{j}}$,  we see that  each $i_j$ is at most $k$ and there are at most $k$ values of $j$ for which the power $i_j$ is non-zero. This implies 
\begin{equation}
\label{eq: bound on the monomial second moment}
\E_{\vect x\sim\N(0, I_{d})}\left[\left(m(\vect x)\right)^{2}\right]\leq2k\cdot(2k)!!\leq(2k)^{k+2}.
\end{equation}
Let $\delta'$ be a real number between $0$ and $1$,  value of which will be chosen later.
The Markov's inequality implies that with probability at least $1-\delta'$ we have
\begin{equation}
\label{eq: bounding second term}
\E_{\vect x\sim S}\left[|m(\vect x)|\indicator_{ |m(\vect x)|>R}\right]
\leq
\frac{\E_{\vect x\sim \N(0,  I_d)}\left[
|m(\vect x)|\indicator_{ |m(\vect x)|>R}\right]}{\delta'}
\leq
\frac{\E_{\vect x\sim \N(0,  I_d)}\left[
|m(\vect x)|^2\right]}{R \delta'}
\leq
\frac{(2k)^{k+2} }{R \delta'}
\end{equation}
We also note that
\begin{equation}
\label{eq: bounding third term}
\E_{\vect x\sim \N(0, I_d)}\left[|m(\vect x)|\indicator_{ |m(\vect x)|>R}\right]
\leq
\frac{\E_{\vect x\sim \N(0,  I_d)}\left[
|m(\vect x)|^2\right]}{R }
\leq
\frac{(2k)^{k+2} }{R }
\end{equation}
Overall,  substituting Equations \ref{eq: first term final bound},  \ref{eq: bounding second term} and \ref{eq: bounding third term} into Equation \ref{eq: breaking error into three terms} we get
\[
\left \vert
\E_{\vect x\sim S}\left[m(\vect x)\indicator_{\vect x\in A}
\right]
-
\E_{\vect x\sim \N(0, I_d)}\left[m(\vect x)\indicator_{\vect x\in A}
\right]
\right \rvert
\leq
10R
\sqrt{
\frac{((d+1)^{k}+\Delta_{\text{VC}}) \log N}{N \delta}
}
+
2
\frac{(2k)^{k+2} }{R \delta'}.
\]
Choosing $R$ to balance the two terms above,  we get
\[
\left \vert
\E_{\vect x\sim S}\left[m(\vect x)\indicator_{\vect x\in A}
\right]
-
\E_{\vect x\sim \N(0, I_d)}\left[m(\vect x)\indicator_{\vect x\in A}
\right]
\right \rvert
\leq
\sqrt{80
\frac{(2k)^{k+2} }{\delta'}
\sqrt{
\frac{((d+1)^{k}+\Delta_{\text{VC}}) \log N}{N \delta}
}
}.
\]
Taking $\delta'=\frac{\delta}{4 (d+1)^{k}}$ and taking a union bound to insure that Equation \ref{eq: bounding second term} holds for all monomials $m$ of degree at most $k$,  we see that with probability at least $1-\delta/2$ it is the case that all monomials $m$ of degree at most $k$ and all $A$ in $\mathcal{H}$ it is the case that
\begin{multline}
\label{eq: final bound on moment times indicator}
\left \vert
\E_{\vect x\sim S}\left[m(\vect x)\indicator_{\vect x\in A}
\right]
-
\E_{\vect x\sim \N(0, I_d)}\left[m(\vect x)\indicator_{\vect x\in A}
\right]
\right \rvert
\leq
\sqrt{320
\frac{(2k)^{k+2} (d+1)^{k} }{\delta^{3/2}}
\sqrt{
\frac{((d+1)^{k}+\Delta_{\text{VC}}) \log N}{N }
}
}
\leq\\
\frac{60(2k)^{k+2} (d+1)^{k}\Delta_{\text{VC}} }{\delta}
\left(
\frac{ \log N}{N }
\right)^{1/4}
.
\end{multline}
Recall again that there are at most $\left(d+1\right)^{k}$ degree-$k$
monomials $m$. This allows us to combine Equation \ref{eq: final bound on moment times indicator} with the triangle inequality to conclude that
with probability at least $1-\delta/2$ for every polynomial $p$
of degree at most $k$ with coefficients bounded by $B$ in absolute
value and for all $A$ in $\mathcal{H}$ we have
\[
\left|\E_{\vect x\sim S}\left[p(\x)\indicator_{\vect x\in A}\right]-\E_{\vect x\sim\N(0, I_{d})}\left[p(\x)\indicator_{\vect x\in A}\right]\right|
\leq
\frac{60 B (2k)^{k+2} (d+1)^{2k}\Delta_{\text{VC}}  }{\delta}
\left(
\frac{ \log N}{N }
\right)^{1/4}.
\]
The polynomial $p^{2}$ has a degree of at most $2k$ and each coefficient
of $p^{2}$ is bounded by $B^{2}\left(d+1\right)^{2k}$. Therefore, 
with probability at least $1-\delta/2$ for every polynomial $p$
of degree at most $k$ with coefficients bounded by $B$ in absolute
value and for all $A$ in $\mathcal{H}$ we have
\[
\left|\E_{\vect x\sim S}\left[\left(p(\x)\right)^{2}\indicator_{\vect x\in A}\right]-\E_{\vect x\sim\N(0, I_{d})}\left[\left(p(\x)\right)^{2}\indicator_{\vect x\in A}\right]\right|\leq
\frac{60 B^2 (4k)^{2k+2} (d+1)^{6k}\Delta_{\text{VC}}  }{\delta}
\left(
\frac{ \log N}{N }
\right)^{1/4}, \]
which completes the proof.
\end{proof}

\section{Disagreement Tester}\label{appendix:disagreement-tester}

In this section we prove the following theorem.
\begin{theorem}
\label{thm: disagreement tester}For every positive absolute constant
$\mu$,  there exists a deterministic algorithm $\Tdisagree$
and some absolute constant $C$ that,  given
\begin{itemize} 
\item a dataset $S$ of points in $\R^{d}$ of size $N\geq\left(\frac{Cd}{\epsilon\delta}\right)^{C}$. 
\item a unit vector $\vect v$ in $\R^{d}$,  
\item parameters $\epsilon, \delta$ and $\mu$ in $(0, 1)$.
\end{itemize}
For any absolute constant $\mu$,  the algorithm
runs in time $poly\left(\frac{dN}{\epsilon\delta}\right)$ and outputs
$\accept$ or outputs $\reject$,  subject to the following for all $\epsilon$ and $\delta$ in $(0, 1)$:
\begin{itemize}
\item \textbf{Completeness:} if $S$ consists of $N\geq\left(\frac{Cd}{\epsilon\delta}\right)^{C}$ i.i.d. samples from the
standard Gaussian distribution,  then with probability at least $1-O\left(\delta\right)$ the set $S$ is such that for all unit vectors $\vect{v}$ the algorithm $\Tdisagree$ accepts when given $(S,  \vect v,  \epsilon,  \delta,  \mu)$ as the input.
\item \textbf{Soundness:} For any dataset $S$ and unit vector $\vect v$, 
if the tester $\mathcal{\Tdisagree}$ accepts,  then for every unit vector $\vect v'$ in $\R^{d}$
the following holds
\[
\pr_{\vect x\sim S}\left[\sign(\vect x\cdot\vect v)\neq\sign(\vect x\cdot\vect v')\right]=(1\pm\mu)\frac{\measuredangle(\vect v, \vect v')}{\pi}\pm O(\epsilon).
\]
\end{itemize}
\end{theorem}

We argue that the following algorithm (see also \Cref{algorithm:disagreement-tester}) satisfies the specifications
above:
\begin{itemize}
\item \textbf{Given: }parameter $\epsilon, \delta$ in $(0, 1)$,  dataset
$S$ of points in $\R^{d}$ of size $N\geq\left(\frac{Cd}{\epsilon\delta}\right)^{C}$, 
a unit vector $\vect v$ in $\R^{d}$,  
\end{itemize}
\begin{enumerate}
\item $k_{1}\leftarrow\frac{2\sqrt{\log2/\epsilon}}{\epsilon}$
\item $k_{2}\leftarrow\frac{C^{0.1}}{\mu^{4}}$
\item For all $a$ and \textbf{$b$} in $\left\{ -\infty, -k_{1}\epsilon, -(k_{1}-1)\epsilon, \cdots, -\epsilon, 0, +\epsilon, \cdots, (k_{1}-1)\epsilon, k_{1}\epsilon, +\infty\right\} $
\begin{enumerate}
\item For all monomials $m$ of degree at most $k_{2}$ over $\R^{d}$:
\begin{enumerate}
\item $A_{m}^{a, b}\leftarrow\E_{\vect x\sim\N(0, I_{d})}[m(\vect x)\cdot\indicator_{a\leq\vect x\cdot\vect v<b}]\pm \frac{60(2k_2(d+1))^{k_2+2}}{\delta}
\left(
\frac{ \log N}{N }
\right)^{1/4}$.
(For how to compute this approximation,  see Claim \ref{claim: deterministic moments in a strip}).
\item If $\left|\E_{\vect x\sim S}[m(\vect x)\cdot\indicator_{a\leq\vect x\cdot\vect v<b}]-A_{m}^{a, b}\right|>\frac{200 (2k_2(d+1))^{k_2+2}}{\delta}
\left(
\frac{ \log N}{N }
\right)^{1/4}$, 
then output $\reject$.
\end{enumerate}
\end{enumerate}
\item If did not reject in any previous step,  output $\accept$.
\end{enumerate}
It is immediate that the algorithm indeed runs in time $poly\left(\frac{dN}{\epsilon\delta}\right)$. 

\subsection{Completeness}
\label{subsec: disagreement tester completeness}
Suppose the set dataset $S$ consists of i.i.d. samples from $\N(0, I_{d})$.
We observe that the collection $\mathcal{H}$ of sets of the form $\indicator_{a\leq \vect v \cdot \vect x < b}$ has VC dimension at most $(d+1)^2$. This allows us to use 
Lemma \ref{lem: concentration for moments},  to conclude that with probability at least $1-\delta$ for all pairs
of $a$ and $b$,  for all unit vectors $\vect{v}$ and for all monomials $m$ of degree at most $k_{2}$
we have 
\[
\left|\E_{\vect x\sim S}[m(\vect x)\cdot\indicator_{a\leq\vect x\cdot\vect v<b}]-\E_{\vect x\sim\N(0.I_{d})}[m(\vect x)\cdot\indicator_{a\leq\vect x\cdot\vect v<b}]\right|\leq
\frac{60(2k_2)^{k_2+2} (d+1)^{k_2+2}}{\delta}
\left(
\frac{ \log N}{N }
\right)^{1/4}
\]
and Claim \ref{claim: deterministic moments in a strip} implies that
\[
\left|A_{m}^{a, b}-\E_{\vect x\sim\N(0.I_{d})}[m(\vect x)\cdot\indicator_{a\leq\vect x\cdot\vect v<b}]\right|\leq \frac{60(2k_2)^{k_2+2} (d+1)^{k_2+2}}{\delta}
\left(
\frac{ \log N}{N }
\right)^{1/4}
\]
The two inequalities above together imply the completeness condition. 

\subsection{Soundness}

\label{subsec: Soundness for disagreement tester}

In order to deduce the soundness condition,  we will need the following
notions:
\begin{definition}
Let $\mathcal{C}$ be a collection of disjoint subsets of $\R^{d}$.
We say that $\mathcal{C}$ is a \emph{partition} of $\R^{d}$ if
$\R^{d}$ equals to the union $\bigcup_{A\in\mathcal{C}}A$.
\end{definition}

\begin{definition}
We say that a function $f:\R^{d}\rightarrow\left\{ 0, 1\right\} $
is $\epsilon$-sandwiched in $L_{1}$ norm between a pair of functions
$f_{\text{up}}:\R^{d}\rightarrow\R$ and $f_{\text{down}}:\R^{d}\rightarrow\R$
under $\N(0, I_{d})$ if:
\begin{itemize}
\item For all $\vect x$ in $\R^{d}$ we have $f_{\text{down}}(\vect x)\leq f(\vect x)\leq f_{\text{up}}(\vect x)$
\item $\E_{\vect x\sim\N(0, I_{d})}\left[f_{\text{up}}(\vect x)-f_{\text{down}}(\vect x)\right]\leq\epsilon$.
\end{itemize}
\end{definition}

\begin{definition}
\label{def: sandwiching wrt partition} We say that a function $f:\R^{d}\rightarrow\left\{ 0, 1\right\} $
has $(\epsilon, B)$-sandwiching degree of at most $k$ in $L_{1}$
norm under $\N(0, I_{d})$ with respect to a partition $\mathcal{C}$
of $\R^{d}$ if the function $f$ is $\epsilon$-sandwiched in $L_{1}$
norm under $\N(0, I_{d})$ between $\sum_{A\in\mathcal{C}}\left(p_{\text{down}}^{A}\indicator_{A}\right)$
and $\sum_{A\in\mathcal{C}}\left(p_{\text{up}}^{A}\indicator_{A}\right)$, 
where $p_{\text{up}}^{A}$ and $p_{\text{down}}^{A}$ are degree$-k$
polynomials over $\R^{d}$whose coefficients are bounded by $B$ in
absolute value.
\end{definition}

Subsection \ref{subsec: bounding sandwiching degree} is dedicated
to proving the following bound on the sandwiching degree of a specific
family of functions with respect to a specific partition of $\R^{d}$.
\begin{proposition}
\label{thm: sandwiching degree of purple}For all $\epsilon$ and
$k_{2}$,  let $k_{1}=\frac{2\sqrt{\log2/\epsilon}}{\epsilon}$,  and
let $\vect v$ be a unit vector in $\R^{d}$. Then,  there exists a
partition $\mathcal{C}$ of $\R^{d}$ consisting of sets of the form
$\left\{ \vect x\in\R^{d}:a\leq\vect v\cdot\vect x\leq b\right\} $
for a certain collection of pairs $a, b$ in $\left\{ -\infty, -k_{1}\epsilon, -(k_{1}-1)\epsilon, \cdots, -\epsilon, 0, +\epsilon, \cdots, (k_{1}-1)\epsilon, k_{1}\epsilon, +\infty\right\} $.
Then,  for every unit vector $\vect v'$,  the function $f(\vect x)=\indicator_{\sign(\vect v\cdot\vect x)\neq\sign(\vect v'\cdot\vect x)}$
has $\left(O\left(\frac{\measuredangle(\vect v, \vect v')}{k_{2}^{1/4}}\right)+10\epsilon, O\left(d^{10k_{2}}\right)\right)$-sandwiching
degree of at most $k_{2}$ in $L_{1}$ norm under $\N(0, I_{d})$ with
respect to the partition $\mathcal{C}$ of $\R^{d}$.
\end{proposition}

A bound on the sandwiching degree of a class of functions leads to
a guarantee for the tester $\Tdisagree$:

\begin{proposition}
\label{prop: sandwiching wrt covering implies fooling} Let $\mathcal{C}$
be a partition of $\R^{d}$ and suppose that a set $S$ of points
in $\R^{d}$ satisfies the following condition for all $A$ in $\mathcal{C}$
and degree-$k_{2}$ monomials $m$ over $\R^{d}$: 
\begin{equation}
\left|\E_{\vect x\sim S}[m(\vect x)\cdot\indicator_{\vect x\in A}]-\E_{\vect x\sim\N(0.I_{d})}[m(\vect x)\cdot\indicator_{\vect x\in A}]\right|\leq\frac{\epsilon}{\left(d+1\right)^{k}|\mathcal{C}|B}\label{eq:moment-matching premise}
\end{equation}
Then,  every $\left\{ 0, 1\right\} $-valued function $f$ that has
has $(\nu, B)$-sandwiching degree of at most $k$ in $L_{1}$ norm
under $\N(0, I_{d})$ with respect to the partition $\mathcal{C}$
we have
\[
\left|\pr_{\vect x\sim S}[f(\vect x)=1]-\pr_{\vect x\sim\N(0, I_{d})}[f(\vect x)=1]\right\rceil \leq\nu+O(\epsilon)
\]
\end{proposition}

\begin{proof}
Since $f$ has $(\nu, B)$-sandwiching degree of at most $k$ in $L_{1}$
norm under $\N(0, I_{d})$ with respect to the partition $\mathcal{C}$, 
we have a collection of polynomials $\left\{ p_{\text{down}}^{A}, p_{\text{up}}^{A}\right\}$ for
all $A$ in $\mathcal{C}$ that have coefficients bounded by $B$, 
satisfy for all $\vect x$ the condition 
\begin{equation}
f(\vect x)\in\left[\sum_{A\in\mathcal{C}}\left(p_{\text{down}}^{A}(\vect x)\indicator_{\vect x\in A}\right), \sum_{A\in\mathcal{C}}\left(p_{\text{up}}^{A}(\vect x)\indicator_{\vect x\in A}\right)\right], \label{eq: f is sandwiched}
\end{equation}
as well as
\begin{equation}
\E_{\vect x\sim\N(0, I_{d})}\left[\sum_{A\in\mathcal{C}}\left(p_{\text{up}}^{A}(\vect x)\indicator_{\vect x\in A}\right)-\sum_{A\in\mathcal{C}}\left(p_{\text{down}}^{A}(\vect x)\indicator_{\vect x\in A}\right)\right]\leq\nu.\label{eq: sandwiching}
\end{equation}
From the bound $B$ on all coefficients of $p_{\text{up}}^{A}$ and
{$p_{\text{down}}^{A}$}and Equation \ref{eq:moment-matching premise}
we see that:
\begin{align}
\left|\sum_{A\in\mathcal{C}}\E_{\vect x\sim\N(0, I_{d})}\left[\left(p_{\text{down}}^{A}(\vect x)\indicator_{\vect x\in A}\right)\right]-\sum_{A\in\mathcal{C}}\E_{\vect x\sim D}\left[\left(p_{\text{down}}^{A}(\vect x)\indicator_{\vect x\in A}\right)\right]\right| & \leq\frac{\epsilon \left(d+1\right)^{k}|\mathcal{C}|B}{\left(d+1\right)^{k}|\mathcal{C}|B}=\epsilon, \label{eq: approximation for sandwichers 1}\\
\left|\sum_{A\in\mathcal{C}}\E_{\vect x\sim\N(0, I_{d})}\left[\left(p_{\text{up}}^{A}(\vect x)\indicator_{\vect x\in A}\right)\right]-\sum_{A\in\mathcal{C}}\E_{\vect x\sim D}\left[\left(p_{\text{up}}^{A}(\vect x)\indicator_{\vect x\in A}\right)\right]\right| & \leq\frac{\epsilon \left(d+1\right)^{k}|\mathcal{C}|B}{\left(d+1\right)^{k}|\mathcal{C}|B}=\epsilon.\label{eq: approximation for sandwichers 2}
\end{align}
Equation \ref{eq: f is sandwiched} implies that

\begin{equation}
\sum_{A\in\mathcal{C}}\E_{\vect x\sim\N(0, I_{d})}\left[\left(p_{\text{down}}^{A}(\vect x)\indicator_{\vect x\in A}\right)\right]\leq\E_{\vect x\sim\N(0, I_{d})}[f(\vect x)]\leq\sum_{A\in\mathcal{C}}\E_{\vect x\sim\N(0, I_{d})}\left[\left(p_{\text{up}}^{A}(\vect x)\indicator_{\vect x\in A}\right)\right], \label{eq: bound for expectation under N}
\end{equation}
and Equation \ref{eq: f is sandwiched} together with Equations \ref{eq: approximation for sandwichers 1}
and \ref{eq: approximation for sandwichers 2} implies that:

\begin{multline}
\sum_{A\in\mathcal{C}}\E_{\vect x\sim\N(0, I_{d})}\left[\left(p_{\text{down}}^{A}(\vect x)\indicator_{\vect x\in A}\right)\right]-\epsilon\leq\sum_{A\in\mathcal{C}}\E_{\vect x\sim D}\left[\left(p_{\text{down}}^{A}(\vect x)\indicator_{\vect x\in A}\right)\right]\leq\E_{\vect x\sim D}[f(\vect x)]\leq\\
\leq\sum_{A\in\mathcal{C}}\E_{\vect x\sim D}\left[\left(p_{\text{up}}^{A}(\vect x)\indicator_{\vect x\in A}\right)\right]\leq\sum_{A\in\mathcal{C}}\E_{\vect x\sim\N(0, I_{d})}\left[\left(p_{\text{up}}^{A}(\vect x)\indicator_{\vect x\in A}\right)\right]+\epsilon.\label{eq: bound for expectation over D}
\end{multline}
Together Equations \ref{eq: bound for expectation over D} and \ref{eq: bound for expectation under N}
constraint the values of both $\E_{\vect x\sim D}[f(\vect x)]$ and
$\E_{\vect x\sim\N(0, I_{d})}[f(\vect x)]$ to the same interval that
via Equation \ref{eq: sandwiching} has a width of at most $\nu+2\epsilon$.
This allows us to conclude
\[
\left|\pr_{\vect x\sim D}[f(\vect x)=1]-\pr_{\vect x\sim\N(0, I_{d})}[f(\vect x)=1]\right|=\left|\E_{\vect x\sim D}[f(\vect x)]-\E_{\vect x\sim\N(0, I_{d})}[f(\vect x)]\right|\leq\nu+2\epsilon, 
\]
completing the proof. 
\end{proof}
Claim \ref{claim: deterministic moments in a strip} implies that all pairs
of $a$ and $b$ in the set

\noindent $\left\{ -\infty, -k_{1}\epsilon,  -(k_{1}-1)\epsilon,  \cdots,  -\epsilon, 0, +\epsilon, \cdots, (k_{1}-1)\epsilon, k_{1}\epsilon, +\infty\right\} $ and
for all monomials $m$ of degree at most $k_{2}$ we have  
\[
\left|A_{m}^{a, b}-\E_{\vect x\sim\N(0.I_{d})}[m(\vect x)\cdot\indicator_{a\leq\vect x\cdot\vect v<b}]\right|
\leq
\frac{60(2k_2)^{k_2+2} (d+1)^{k_2+2}}{\delta}
\left(
\frac{ \log N}{N }
\right)^{1/4}
.
\]
If the algorithm accepts,  then we have for all pairs of $a$ and $b$
in

\noindent $\left\{ -\infty, -k_{1}\epsilon, -(k_{1}-1)\epsilon, \cdots, -\epsilon, 0, +\epsilon, \cdots, (k_{1}-1)\epsilon, k_{1}\epsilon, +\infty\right\} $
and for all monomials $m$ of degree at most $k_{2}$ that
\[
\left|\E_{\vect x\sim S}[m(\vect x)\cdot\indicator_{a\leq\vect x\cdot\vect v<b}]-A_{m}^{a, b}\right|\leq
\frac{200(2k_2)^{k_2+2} (d+1)^{k_2+2}}{\delta}
\left(
\frac{ \log N}{N }
\right)^{1/4}
.
\]
The two inequalities above imply that 
\[
\left|\E_{\vect x\sim S}[m(\vect x)\cdot\indicator_{a\leq\vect x\cdot\vect v<b}]-\E_{\vect x\sim\N(0.I_{d})}[m(\vect x)\cdot\indicator_{a\leq\vect x\cdot\vect v<b}]\right|\leq
\frac{260(2k_2)^{k_2+2} (d+1)^{k_2+2}}{\delta}
\left(
\frac{ \log N}{N }
\right)^{1/4}
\]

\noindent Taking the equation above,  together with Proposition \ref{prop: sandwiching wrt covering implies fooling}
and Proposition \ref{thm: sandwiching degree of purple} we conclude
that 
\begin{multline*}
\left|\pr_{\vect x\sim S}[\sign(\vect v\cdot\vect x)\neq\sign(\vect v'\cdot\vect x)]-\pr_{\vect x\sim\N(0, I_{d})}[\sign(\vect v\cdot\vect x)\neq\sign(\vect v'\cdot\vect x)]\right\rceil \leq O\left(\frac{\measuredangle(\vect v, \vect v')}{k_{2}^{1/4}}\right)+\\
+O\left(\frac{(2k_2)^{k_2+2} (d+1)^{k_2+2}}{\delta}
\left(
\frac{ \log N}{N }
\right)^{1/4}\right)
\end{multline*}
Substituting 
$k_{2}\leftarrow\frac{C^{0.1}}{\mu^{4}}$,  $N\geq\left(\frac{Cd}{\epsilon\delta}\right)^{C}$, 
taking $C$ to be a sufficiently large absolute constant and recalling
that $\pr_{\vect x\sim\N(0, I_{d})}[\sign(\vect v\cdot\vect x)\neq\sign(\vect v'\cdot\vect x)]$
equals to $\measuredangle(\vect v, \vect v')/\pi$ we conclude that 
\[
\pr_{\vect x\sim S}\left[\sign(\vect x\cdot\vect v)\neq\sign(\vect x\cdot\vect v')\right]=(1\pm\mu)\frac{\measuredangle(\vect v, \vect v')}{\pi}\pm O(\epsilon).
\]

\subsection{Bounding sandwiching degree of the disagreement region}\label{section:sandwiching-1}

\label{subsec: bounding sandwiching degree}

To prove Proposition \ref{thm: sandwiching degree of purple},  we
will need the following result by \cite{diakonikolas2010bounded},  \cite{gopalan2010fooling}.
\begin{fact}
\label{fact:DGJSV}For every positive integer $k$ and a real value
$t$,  the function $f(z)=\indicator_{z\leq t}$ has $(O(\frac{\log^{3}k}{\sqrt{k}}), O(2^{10k}))$-sandwiching
degree in $L_{1}$ norm of at most $k$ under $\mathcal{N}(0, 1)$.
\end{fact}

The following corollary slightly strengthens the fact above:
\begin{corollary}
\label{corr: slightly stengthened 1-d sandwiching} Let $t\in \R$. For every positive
integer $k\geq2$,  the function $f:\R\to \{0, 1\}$ with $f(z)=\indicator_{z\leq t}$
is $(O(\min(\frac{\log^{3}k}{\sqrt{k}}, \frac{1}{t^{2}})), 2^{10k})$-sandwiched
in $L_{1}$ norm under $\mathcal{N}(0, 1)$ between a pair of polynomials
$R_{\text{down}}^{t}$ and $R_{\text{up}}^{t}$ of degree
$k$.
\end{corollary}

\begin{proof}
Indeed,  if $\frac{\log^{3}k}{\sqrt{k}}\leq\frac{1}{t^{2}}$ then the
corollary follows from Fact \ref{fact:DGJSV}. So all we need to do
is to consider the other case. We see that either $t>1$ or $t<-1$
(since $k\geq2)$. if $t>1$ we take $p_{\text{down}}(\x)=0$ and $p_{\text{up}}(\x)=\left(\frac{x}{t}\right)^{2}$.
If $t<-1$,  we take take $p_{\text{up}}(\x)=1$ and $p_{\text{down}}(\x)=1-\left(\frac{x}{t}\right)^{2}$.
In either case,  we see that the polynomials $p_{\text{down}}$ and
$p_{\text{up}}$ form a pair of $\left(O\left(\min\left(\frac{\log^{3}k}{\sqrt{k}}, \frac{1}{t^{2}}\right)\right), 1\right)$-sandwiching
polynomials of degree $2$.
\end{proof}
Let $\vect v_{\perp}$ be the unit vector equal up to scaling to the
component of $\vect v'$ perpendicular to $\vect v$. Then,  we have

\begin{align}
\psi_{\text{down}}(\vect x) & \leq\indicator_{\vect v\cdot\vect x\geq0\land\vect v'\cdot\vect x<0} \leq\psi_{\text{up}}(\vect x)
\label{eq: psis one}\\
\psi_{\text{down}}(\vect x) & \leq\indicator_{\vect v\cdot\vect x>0\land\vect v'\cdot\vect x\leq0} \leq\psi_{\text{up}}(\vect x)
\label{eq: psis two}
\end{align}

where

\[
\psi_{\text{up}}(\vect x)=\begin{cases}
1 & \text{if }\text{\ensuremath{\vect v\cdot\vect x}\ensuremath{\ensuremath{\geq k_{1}\epsilon}}}\text{ or }\vect v\cdot\vect x\ensuremath{=0}\\
\indicator_{\vect v_{\perp}\cdot\vect x\tan\angleparam\leq-j\epsilon} & \text{\text{if }\text{\ensuremath{\vect v\cdot\vect x}\ensuremath{\ensuremath{\neq}0} and } }\vect v\cdot\vect x\in\left[j\epsilon, (j+1)\epsilon\right)\text{ for }0\leq j\leq k_{1}-1\\
0 & \text{if }\text{\ensuremath{\vect v\cdot\vect x}\ensuremath{<0}}
\end{cases}
\]

\[
\psi_{\text{down}}(\vect x)=\begin{cases}
0 & \text{if }\text{\ensuremath{\vect v\cdot\vect x}\ensuremath{\ensuremath{\geq k_{1}\epsilon}}}\text{ or }\vect v\cdot\vect x\ensuremath{=0}\\
\indicator_{\vect v_{\perp}\cdot\vect x\tan\angleparam<-(j+1)\epsilon} & \text{if }\text{\ensuremath{\vect v\cdot\vect x}\ensuremath{\ensuremath{\neq}0} and }\vect v\cdot\vect x\in\left[j\epsilon, (j+1)\epsilon\right)\text{ for }0\leq j\leq k_{1}-1\\
0 & \text{if }\text{\ensuremath{\vect v\cdot\vect x}\ensuremath{<0}}
\end{cases}
\]

Recall that for every $t\in\R$,  Corollary \ref{corr: slightly stengthened 1-d sandwiching}
gives us one-dimensional degree-$k_{2}$ sandwiching polynomials $R_{\text{down}}^{t}(z)$
and $R_{\text{up}}^{t}(z)$ for $\indicator_{z\leq t}$. Using this
notation,  we have for all $\vect x$ in $\R^{d}$ 
\begin{multline}
\overbrace{\sum_{j=0}^{k_{1}-1}\indicator_{\vect v\cdot\vect x\cdot\left[j\epsilon, (j+1)\epsilon\right)}R_{\text{down}}^{-(j+1)\epsilon/\tan\angleparam}(\vect v_{\perp}\cdot\vect x)}^{\text{Denote this \ensuremath{\phi_{\text{down}}(\vect x)}}}\leq\psi_{\text{down}}(\vect x)\leq\indicator_{\vect v\cdot\vect x\geq0\land\vect v'\cdot\vect x<0}\leq\\
\leq\psi_{\text{up}}(\vect x)\leq\underbrace{\indicator_{\vect v\cdot\vect x\ensuremath{\geq k_{1}\epsilon}}+\sum_{j=0}^{k_{1}-1}\indicator_{\vect v\cdot\vect x\cdot\left[j\epsilon, (j+1)\epsilon\right)}R_{\text{up}}^{-j\epsilon/\tan\angleparam}(\vect v_{\perp}\cdot\vect x)}_{\text{Denote this \ensuremath{\phi_{\text{up}}(\vect x)}}}\label{eq: phi are sandwichers}
\end{multline}

In order to conclude Proposition \ref{thm: sandwiching degree of purple}.
We show the following two claims:
\begin{claim}
\label{claim: sandwiching error small}We have 
\[
\E_{\vect x\sim\N(0, I_{d})}\left[\phi_{\text{up}}(\vect x)-\phi_{\text{down}}(\vect x)\right]\leq O\left(\frac{\log^{1.5}k_{2}}{k_{2}^{1/4}}\cdot\measuredangle(\vect v, \vect v')\right)+10\epsilon
\]
\end{claim}

\begin{claim}
\label{claim:coefficient small}For all integers $j$ in $\left[0, k_{1}-1]\right]$, 
every coefficient of $R_{\text{down}}^{-(j+1)\epsilon/\tan\angleparam}(\vect v_{\perp}\cdot\vect x)$
and $R_{\text{up}}^{-j\epsilon/\tan\angleparam}(\vect v_{\perp}\cdot\vect x)$
is at most $O\left(d^{10k_{2}}\right)$ in absolute value.
\end{claim}

Proposition \ref{thm: sandwiching degree of purple} follows from
the two claims above for as follows. We first observe that Equations
\ref{eq: psis two} and \ref{eq: phi are sandwichers} imply that 

\[
\phi_{\text{down}}(-\vect x)\leq\indicator_{\vect v\cdot\vect x<0\land\vect v'\cdot\vect x\geq0}\leq\phi_{\text{up}}(-\vect x).
\]
Recalling our convention that $\sign(0)=1$,  we see that
\[
\indicator_{\sign(\vect v\cdot\vect x)\neq\sign(\vect v'\cdot\vect x)}=\indicator_{\vect v\cdot\vect x\geq0\land\vect v'\cdot\vect x<0}+\indicator_{\vect v\cdot\vect x<0\land\vect v'\cdot\vect x\geq0}
\]
this,  together with \ref{eq: phi are sandwichers} allows us to bound

\begin{equation}
\phi_{\text{down}}(\vect x)+\phi_{\text{down}}(-\vect x)\leq\indicator_{\sign(\vect v\cdot\vect x)\neq\sign(\vect v'\cdot\vect x)}\leq\phi_{\text{up}}(\vect x)+\phi_{\text{up}}(-\vect x), \label{eq: sandwiching for the whole disagreement region}
\end{equation}
Claim \ref{eq: phi are sandwichers} allows us to conclude that 
\begin{equation}
\E_{\vect x\sim\N(0, I_{d})}\left[\phi_{\text{up}}(\vect x)+\phi_{\text{up}}(-\vect x)-\phi_{\text{down}}(\vect x)-\phi_{\text{down}}(-\vect x)\right]\leq O\left(\frac{\log^{1.5}k_{2}}{k_{2}^{1/4}}\cdot\measuredangle(\vect v, \vect v')\right)+20\epsilon.\label{eq: closensess for the whole disagreement region}
\end{equation}
Equations \ref{eq: sandwiching for the whole disagreement region}
and \ref{eq: closensess for the whole disagreement region},  together
with comparing the definition of $\phi_{\text{up}}$ and $\phi_{\text{down}}$
with Definition \ref{def: sandwiching wrt partition} and recalling
Claim \ref{claim:coefficient small},  allow us to conclude that there
exists a partition $\mathcal{C}$ of $\R^{d}$ consisting of sets
of the form $\left\{ \vect x\in\R^{d}:a\leq\vect v\cdot\vect x\leq b\right\} $
for a certain collection of pairs $a, b$ in

\noindent $\left\{ -\infty, -k_{1}\epsilon, -(k_{1}-1)\epsilon, \cdots, -\epsilon, 0, +\epsilon, \cdots, (k_{1}-1)\epsilon, k_{1}\epsilon, +\infty\right\} $, 
such that for every unit vector $\vect v'$,  the function $f(\vect x)=\indicator_{\sign(\vect v\cdot\vect x)\neq\sign(\vect v'\cdot\vect x)}$
has $\left(O\left(\frac{\measuredangle(\vect v, \vect v')}{k_{2}^{1/4}}\right)+10\epsilon, O\left(d^{10k_{2}}\right)\right)$-sandwiching
degree of at most $k_{2}$ in $L_{1}$ norm under $\N(0, I_{d})$ with
respect to the partition $\mathcal{C}$ of $\R^{d}$. This implies
Proposition \ref{thm: sandwiching degree of purple}. 

We now proceed to proving Claim \ref{claim: sandwiching error small}
\begin{proof}[Proof of Claim \ref{claim: sandwiching error small}]
We have the following.
\begin{multline}
\E_{\vect x\sim\N(0, I_{d})}\left[\psi_{\text{up}}(\vect x)-\psi_{\text{down}}(\vect x)\right]\leq\\
\pr_{\vect x\sim\N(0, I_{d})}\left[\vect v\cdot\vect x>k_{1}\epsilon\right]+\sum_{j=0}^{k_{1}-1}\pr_{\vect x\sim\N(0, I_{d})}\left[\left\{ \vect v\cdot\vect x\in\left[j\epsilon, (j+1)\epsilon\right)\right\} \land\left\{ \vect v_{\perp}\cdot\vect x\tan\angleparam\in\left[-(j+1)\epsilon, -j\epsilon\right)\right\} \right]\leq\\
e^{-\left(k_{1}\epsilon\right)^{2}}+\epsilon\underbrace{\sum_{j=0}^{\infty}\pr_{\vect x\sim\N(0, I_{d})}\left[\vect v_{\perp}\cdot\vect x\tan\angleparam\in\left[-(j+1)\epsilon, -j\epsilon\right)\right]}_{\leq1}\leq2\epsilon\label{eq: rounding to steps does not introduce much error}
\end{multline}

Let $\angleparam$ denote the angle $\measuredangle(\vect v, \vect v')$. Given
the inequality above,  in order to finish the proof of Claim \ref{claim: sandwiching error small}, 
it remains to upper-bound $\E_{\vect x\sim\N(0, I_{d})}\left[\phi_{\text{up}}(\vect x)-\psi_{\text{up}}(\vect x)\right]$
and $\E_{\vect x\sim\N(0, I_{d})}\left[\phi_{\text{up}}(\vect x)-\psi_{\text{up}}(\vect x)\right]$
by $O(\angleparam)+4\epsilon$. 

From Corollary \ref{corr: slightly stengthened 1-d sandwiching} we
know that for any $t$ we have:

\begin{equation}
\E_{z\sim\N(0, 1)}\left[R_{\text{up}}^{t}(z)-R_{\text{down}}^{t}(z)\right]\leq O\left(\min\left(\frac{\log^{3}k}{\sqrt{k}}, \frac{1}{t^{2}}\right)\right), \label{eq: 1-d step function indicator error}
\end{equation}
and for every $z$ in $\R$
\begin{equation}
R_{\text{down}}^{t}(z)\leq\indicator_{z\leq t}\leq R_{\text{up}}^{t}(z).\label{eq: indicator sandwiching}
\end{equation}

Since $\vect v$ and $\vect v_{\perp}$ are orthogonal,  the random
variables $\vect v_{\perp}\vect{\vect x}$ and $\vect v\cdot\vect x$
are independent standard Gaussians. Using this,  together with Equations
\ref{eq: 1-d step function indicator error} and \ref{eq: indicator sandwiching}
we obtain the following.
\begin{align*}
\E_{\vect x\sim\N(0, I_{d})}&\left[\phi_{\text{up}}(\vect x)-\psi_{\text{up}}(\vect x)\right]\\
=&\sum_{j=0}^{k_{1}-1}\pr_{\vect x\sim\N(0, I_{d})}\left[\vect v\cdot\vect x\in\left[j\epsilon, (j+1)\epsilon\right)\right]\E_{\vect x\sim\N(0, I_{d})}\left[R_{\text{up}}^{-j\epsilon/\tan\angleparam}(\vect v_{\perp}\cdot\vect x)-\indicator_{\vect v_{\perp}\cdot\vect x\leq-j\epsilon/\tan\angleparam}\right] \\
\leq&\sum_{j=0}^{k_{1}-1}\pr_{z_{1}\sim\N(0, 1)}\left[z_{1}\in\left[j\epsilon, (j+1)\epsilon\right)\right]\E_{z_{2}\sim\N(0, 1)}\left[R_{\text{up}}^{-j\epsilon/\tan\angleparam}(z_{2})-R_{\text{down}}^{-j\epsilon/\tan\angleparam}(z_{2})\right]\\
\leq&\sum_{j=0}^{k_{1}-1}\pr_{z_{1}\sim\N(0, 1)}\left[z_{1}\in\left[j\epsilon, (j+1)\epsilon\right)\right]O\Bigr(\min\Bigr(\frac{\log^{3}k_{2}}{\sqrt{k_{2}}}, \Bigr(\frac{\tan\angleparam}{j\epsilon}\Bigr)^{2}\Bigr)\Bigr)
\end{align*}

First,  consider the case $\angleparam\geq\pi/4$. The above inequality
implies 
\begin{align*}
\E_{\vect x\sim\N(0, I_{d})}\left[\phi_{\text{up}}(\vect x)-\psi_{\text{up}}(\vect x)\right] &\leq O\left(\frac{\log^{3}k}{\sqrt{k}}\right)\underbrace{\sum_{j=-k_{1}}^{k_{1}-1}\pr_{z_{1}\sim\N(0, 1)}\left[z_{1}\in\left[j\epsilon, (j+1)\epsilon\right)\right]}_{=\pr_{z_{1}\sim\N(0, 1)}\left[-k_{1}\epsilon\leq z_{1}<(k_{1}-1)\epsilon\right]} \\
&=O\left(\frac{\log^{3}k_{2}}{\sqrt{k_{2}}}\right)=O\left(\frac{\angleparam\log^{1.5}k_{2}}{k_{2}^{1/4}}\right)
\end{align*}
On the other hand,  if $\angleparam\leq\pi/4$ we have $\tan\angleparam\leq2\angleparam$
and therefore,  recalling that for any $j$ it is the case that $\pr_{z_{1}\sim\N(0, 1)}\left[z_{1}\in\left[j\epsilon, (j+1)\epsilon\right)\right]\leq\epsilon$, 
we have

\begin{align*}
\E_{\vect x\sim\N(0, I_{d})}\Bigr[\phi_{\text{up}}(\vect x)-\psi_{\text{up}}(\vect x)\Bigr] &\leq\sum_{j=0}^{k_{1}-1}O\Bigr(\min\Bigr(\frac{\log^{3}k_{2}}{\sqrt{k_{2}}}, \Bigr(\frac{\tan\angleparam}{j\epsilon}\Bigr)^{2}\Bigr)\epsilon\Bigr) \\
&=\int_{0}^{k_1\epsilon}O\Bigr(\min\Bigr(\frac{\log^{3}k_{2}}{\sqrt{k_{2}}}, \Bigr(\frac{\tan\angleparam}{\lfloor z/\epsilon\rfloor \epsilon}\Bigr)^{2}\Bigr)\Bigr) \ dz
\\
&\leq\int_{0}^{+\infty}O\Bigr(\min\Bigr(\frac{\log^{3}k_{2}}{\sqrt{k_{2}}}, \Bigr(\frac{\tan\angleparam}{z-\epsilon}\Bigr)^{2}\Bigr)\Bigr)\d z \\
&=O\Bigr(\frac{\log^{3}k_{2}}{\sqrt{k_{2}}}\epsilon\Bigr)+\int_{0}^{+\infty}O\Bigr(\min\Bigr(\frac{\log^{3}k_{2}}{\sqrt{k_{2}}}, \Bigr(\frac{\tan\angleparam}{z}\Bigr)^{2}\Bigr)\Bigr)\d z, 
\end{align*}
which together with a change of variables with a new variable $z'=z/\tan \angleparam$ allows us to proceed as follows:
\begin{multline}
\E_{\vect x\sim\N(0, I_{d})}\Bigr[\phi_{\text{up}}(\vect x)-\psi_{\text{up}}(\vect x)\Bigr]
=O\Bigr(\frac{\log^{3}k_{2}}{\sqrt{k_{2}}}\epsilon\Bigr)+\tan\angleparam\int_{0}^{+\infty}O\Bigr(\min\Bigr(\frac{\log^{3}k_{2}}{\sqrt{k_{2}}}, \Bigr(\frac{1}{z'}\Bigr)^{2}\Bigr)\Bigr)\d z' =\\=O\Bigr(\frac{\log^{3}k_{2}}{\sqrt{k_{2}}}\epsilon+\tan\angleparam\Bigr(\frac{\log^{3}k_{2}}{\sqrt{k_{2}}}\Bigr(\frac{\sqrt{k_{2}}}{\log^{3}k_{2}}\Bigr)^{0.5}+\Bigr(\frac{\log^{3}k_{2}}{\sqrt{k_{2}}}\Bigr)^{0.5}\Bigr)\Bigr)=O\Bigr(\frac{\angleparam\log^{1.5}k_{2}}{k_{2}^{1/4}}\Bigr)\label{eq: replacing the sum with an integral}
\end{multline}
Overall,  in either case we have $\E_{\vect x\sim\N(0, I_{d})}\left[\phi_{\text{up}}(\vect x)-\psi_{\text{up}}(\vect x)\right]=O\left(\frac{\angleparam\log^{1.5}k_{2}}{k_{2}^{1/4}}\right)$.
We now go through a fully analogous argument to show that also $\E_{\vect x\sim\N(0, I_{d})}\left[\psi_{\text{down}}(\vect x)-\phi_{\text{down}}(\vect x)\right]=O\left(\frac{\angleparam\log^{1.5}k_{2}}{k_{2}^{1/4}}\right)$.
Again,  from the independence of $\vect v_{\perp}\vect{\vect x}$ and
$\vect v\cdot\vect x$,  together with Equations \ref{eq: 1-d step function indicator error}
and \ref{eq: indicator sandwiching} we have:
\begin{align*}
\E_{\vect x\sim\N(0, I_{d})}&\left[\psi_{\text{down}}(\vect x)-\phi_{\text{down}}(\vect x)\right]\\
=&\sum_{j=0}^{k_{1}-1}\pr_{\vect x\sim\N(0, I_{d})}\left[\vect v\cdot\vect x\in\left[j\epsilon, (j+1)\epsilon\right)\right]\E_{\vect x\sim\N(0, I_{d})}\left[\indicator_{\vect v_{\perp}\cdot\vect x\leq-(j+1)\epsilon/\tan\angleparam}-R_{\text{down}}^{-(j+1)\epsilon/\tan\angleparam}(\vect v_{\perp}\cdot\vect x)\right]\\
\leq& \sum_{j=0}^{k_{1}-1}\pr_{z_{1}\sim\N(0, 1)}\left[z_{1}\in\left[j\epsilon, (j+1)\epsilon\right)\right]\E_{z_{2}\sim\N(0, 1)}\left[R_{\text{up}}^{-(j+1)\epsilon/\tan\angleparam}(z_{2})-R_{\text{down}}^{-(j+1)\epsilon/\tan\angleparam}(z_{2})\right]\\
\leq &\sum_{j=-k_{1}}^{k_{1}-1}\pr_{z_{1}\sim\N(0, 1)}\Bigr[z_{1}\in\Bigr[j\epsilon, (j+1)\epsilon\Bigr)\Bigr]O\Bigr(\min\Bigr(\frac{\log^{3}k_{2}}{\sqrt{k_{2}}}, \Bigr(\frac{\tan\angleparam}{(j+1)\epsilon}\Bigr)^{2}\Bigr)\Bigr)
\end{align*}

Again,  we first consider the case $\angleparam\geq\pi/4$. The above inequality
implies 
\begin{multline*}
\E_{\vect x\sim\N(0, I_{d})}\Bigr[\psi_{\text{down}}(\vect x)-\phi_{\text{down}}(\vect x)\Bigr]\leq\\
O\Bigr(\frac{\log^{3}k}{\sqrt{k}}\Bigr)\underbrace{\sum_{j=0}^{k_{1}-1}\pr_{z_{1}\sim\N(0, 1)}\Bigr[z_{1}\in\Bigr[j\epsilon, (j+1)\epsilon\Bigr)\Bigr]}_{=\pr_{z_{1}\sim\N(0, 1)}\Bigr[0\leq z_{1}<(k_{1}-1)\epsilon\Bigr]}=O\Bigr(\frac{\log^{3}k_{2}}{\sqrt{k_{2}}}\Bigr)=O\Bigr(\frac{\angleparam\log^{1.5}k_{2}}{k_{2}^{1/4}}\Bigr)
\end{multline*}
On the other hand,  if $\angleparam\leq\pi/4$ we have $\tan\angleparam\leq2\angleparam$
and therefore,  recalling that for any $j$ it is the case that $\pr_{z_{1}\sim\N(0, 1)}\left[z_{1}\in\left[j\epsilon, (j+1)\epsilon\right)\right]\leq\epsilon$, 
we have

\begin{align*}
\E_{\vect x\sim\N(0, I_{d})}\Bigr[\psi_{\text{down}}(\vect x)-\phi_{\text{down}}(\vect x)\Bigr]&\leq\sum_{j=0}^{k_{1}-1}O\Bigr(\min\Bigr(\frac{\log^{3}k_{2}}{\sqrt{k_{2}}}, \Bigr(\frac{\tan\angleparam}{(j+1)\epsilon}\Bigr)^{2}\Bigr)\epsilon\Bigr)\\
&\leq\int_{\infty}^{+\infty}O\Bigr(\min\Bigr(\frac{\log^{3}k_{2}}{\sqrt{k_{2}}}, \Bigr(\frac{\tan\angleparam}{z}\Bigr)^{2}\Bigr)\Bigr)\d z=O\Bigr(\frac{\angleparam\log^{1.5}k_{2}}{k_{2}^{1/4}}\Bigr), 
\end{align*}
where the last step follows via precisely the same chain of inequalities
as in Equation \ref{eq: replacing the sum with an integral}.

In total,  combining our bounds on $\E_{\vect x\sim\N(0, I_{d})}\left[\psi_{\text{down}}(\vect x)-\phi_{\text{down}}(\vect x)\right]$, 
$\E_{\vect x\sim\N(0, I_{d})}\left[\phi_{\text{up}}(\vect x)-\psi_{\text{up}}(\vect x)\right]$
and $\E_{\vect x\sim\N(0, I_{d})}\left[\psi_{\text{up}}(\vect x)-\psi_{\text{down}}(\vect x)\right]$
we conclude that the quantity $\E_{\vect x\sim\N(0, I_{d})}[\phi_{\text{up}}(\vect x)-\phi_{\text{down}}(\vect x)]$ is at most $ O(\frac{\log^{1.5}k_{2}}{k_{2}^{1/4}}\cdot\measuredangle(\vect v, \vect v'))+10\epsilon$,  as desired.
\end{proof}
It only remains to prove Claim \ref{claim:coefficient small} to conclude the proof of the completeness condition.
\begin{proof}[Proof of Claim \ref{claim:coefficient small}]
Corrollary \ref{prop: sandwiching wrt covering implies fooling}
says that for any value of $t$,  the degree-$k_{2}$ one-dimensional
polynomials $R_{\text{up}}^{t}(z)$ and $R_{\text{down}}^{t}(z)$
have all their coefficients bounded by $O\left(2^{10k_{2}}\right)$.
If one substitutes $\vect v\cdot\vect x$ in place of $z$ into either
of these polynomials and opens the parentheses,  the fact that $\vect v$
is a unit vector allows us to bound the size of the largest coefficients
of $R_{\text{down}}^{t}(\vect v_{\perp}\cdot\vect x)$ and $R_{\text{up}}^{t}(\vect v_{\perp}\cdot\vect x)$
by $O((d+1)^{k_{2}}(k_{2}+1)2^{10k_{2}})=O(d^{10k_{2}})$, 
proving the claim.
\end{proof}

\subsection{Miscellaneous Claims}
\begin{claim}
\label{claim: deterministic moments in a strip}There is a deterministic
algorithm that given a unit vector $\vect v$ in $\R^{d}$,  scalars
$a$ and $b$,  a monomial $m$ over $\R^{d}$ of degree at most $k_{2}$, 
an accuracy parameter $\beta\in(0, 1]$,  runs in time $\poly\left(\left(k_{2}d\right)^{k_{2}}/\beta\right)$
and computes an approximation of $\E_{\vect x\sim\N(0, I_{d})}\left[m(\vect x)\cdot\indicator_{a\leq\vect x\cdot\vect v<b}\right]$
up to an additive error $\beta$.
\end{claim}

\begin{proof}
Firstly,  we compute an orthonormal basis $\left\{ \vect w_{1}, \cdots, \vect w_{d-1}\right\} $
for the $(d-1)$-dimensional subspace of $\R^{d}$ that is orthogonal
to $\vect v$. We express $m(\vect x)=p(\vect{w_{1}}\cdot\vect x, \cdots, \vect{w_{d-1}}\cdot\vect x, \vect v\cdot\vect x$), 
and note that the polynomial $p$ has all its coefficients between
$0$ and $\left(d+1\right)^{k_{2}}$,  and $p$ is comprised of at
most $\left(d+1\right)^{k_{2}}$ monomials. Thus,  to have an additive
$\beta$-approximation for $\E_{\vect x\sim\N(0, I_{d})}\left[m(\vect x)\cdot\indicator_{a\leq\vect x\cdot\vect v<b}\right]$, 
it sufficies to compute for every monomial $m'$ of degree at most
$k_{2}$ an additive $\frac{\beta}{d^{2k_{2}}}$-approximation to the quantity
\[
\E_{\vect x\sim\N(0, I_{d})}\left[m'(\vect{w_{1}}\cdot\vect x, \cdots, \vect{w_{d-1}}\cdot\vect x, \vect v\cdot\vect x)\cdot\indicator_{a\leq\vect v\cdot\vect x<b}\right]\,  , 
\]
which via the spherical symmetry of $\N(0, I_{d})$ equals to $\E_{\vect x\sim\N(0, I_{d})}\left[m'(\vect x)\cdot\indicator_{a\leq x_{1}<b}\right]$.

Secondly,  for every monomial $m'$ of degree at most $k_{2}$,  we
compute an approximation of the quantity $\E_{\vect x\sim\N(0, I_{d})}\left[m'(\vect x)\cdot\indicator_{a\leq x_{1}<b}\right]$
up to an additive error of $\frac{\beta}{10d^{k_{2}}}$. To this end, 
we write $m'(\vect x)=\prod_{i}\left(x_{i}\right)^{\coefficientvector_{i}}$ where
$\sum_{i}\coefficientvector_{i}\leq k_{2}$ and see that 
\[
\E_{\vect x\sim\N(0, I_{d})}\left[m'(\vect x)\cdot\indicator_{a\leq x_{1}<b}\right]=\underbrace{\left(\prod_{i>1}(\coefficientvector_{i}-1)!!\indicator_{\coefficientvector_{i}\text{ is even}}\right)}_{\leq k_{2}^{10k_{2}}}\frac{1}{\sqrt{2\pi}}\int_{a}^{b}e^{-z^{2}/2}z^{\coefficientvector_{1}}\d z.
\]
Note that $\coefficientvector_{1}$ is an integer between $0$ and $k_{2}$. Since
we were seeking to compute a $\frac{\beta}{d^{2k_{2}}}$-approximation
to $\E_{\vect x\sim\N(0, I_{d})}\left[m'(\vect x)\cdot\indicator_{a\leq x_{1}<b}\right]$, 
we see that this approximation can be obtained from the equaiton above
together with an additive $\frac{\beta}{\left(d+1\right)^{2k_{2}}k^{10k_{2}}}$-approximation
to $\frac{1}{\sqrt{2\pi}}\int_{a}^{b}e^{-z^{2}/2}z^{\coefficientvector_{1}}\d z$.
We denote $\rho(z)=e^{-z^{2}/2}z^{\coefficientvector_{1}}$,  and let $\beta'=\frac{\beta}{\left(d+1\right)^{2k_{2}}k_{2}^{10k_{2}}}$
. We see that the function $\rho$ has the following key properties:
\begin{enumerate}
\item For all $z$ in $\R^{d}$,  the derivative $\rho'(z)=\coefficientvector_{1}e^{-z^{2}/2}z^{\coefficientvector_{1}-1}\indicator_{\coefficientvector_{1}\geq1}-e^{-z^{2}/2}z^{\coefficientvector_{1}+1}$
we have $\abs{\rho'(z)}\leq\left(k_{2}+1\right)^{k_{2}+1}$
\item For all $z_{0}$ in $\R^{d}$ satisfying $z_{0}>4k_{2}+2$ the value
$\int_{\abs z>z_{0}}\abs{e^{-z^{2}/2}z^{\coefficientvector_{1}}}\d z$ is at most
$\int_{\abs z>z_{0}}e^{-z^{2}/4}\d z$ which in turn is at most $e^{-z_{0}^{2}/4}$.
\end{enumerate}
The three properties above imply that one can approximate the value
of $\int_{a}^{b}\rho(z)\d z$ up to an additive error of $\beta'$
via discretization,  i.e., by splitting
the interval $[a, b]\cap[-\sqrt{2\ln(\beta')}, \sqrt{2\ln\left(\beta'\right)}]$
into intervals of size at most $\Delta$ and for each of these intervals
$[a'_{j}, b'_{j}]$ use the inequality 
\[
\int_{a_{j}'}^{b_{j}'}\rho(z)\d z=\rho(a_{j}')(a_{j}'-b_{j}')\pm\left(\sup_{z\in\R}\abs{\rho'(z)}\right)(a_{j}'-b_{j}')^{2}, 
\]
which implies that 
\begin{align*}
\int_{z\in[a, b]}\rho(z)\d z &=\overbrace{\int_{z\in[a, b]\cap\left[-\sqrt{2\ln\left(\beta'\right)}, \sqrt{2\ln\left(\beta'\right)}\right]}\rho(z)\d z\pm\frac{\beta'}{2}}^{\text{by property (2) of }\rho}\\
&=\sum_{j}\left(\rho(a_{j}')(a_{j}'-b_{j}')\pm\left(\sup_{z\in\R}\abs{\rho'(z)}\right)(a_{j}'-b_{j}')^{2}\right)\pm\frac{\beta'}{2}\\
&=\sum_{j}\rho(a_{j}')(a_{j}'-b_{j}')\pm\Biggr(\underbrace{\left(\sup_{z\in\R}\abs{\rho'(z)}\right)}_{\substack{\leq\left(k_{2}+1\right)^{k_{2}+1}\\
\text{by property (1) of }\rho
}
}\sqrt{8\ln\left(\beta'\right)}\Delta+\frac{\beta'}{2}\Biggr), 
\end{align*}
which implies that if we take $\Delta$ to be $\frac{\beta'}{\sqrt{8\ln\left(\beta'\right)}\left(k_{2}+1\right)^{k_{2}+1}}$, 
then 
\[
\sum_{j}\rho(a_{j}')(a_{j}'-b_{j}')=\int_{z\in[a, b]}\rho(z)\d z\pm\beta'.
\]
Overall,  evaluating the sum above requires one to compute $\rho(a_{j}')$
on $\poly((k_{2})^{k_{2}}/\beta')$ values of
$a_{j}'$. Therefore,  substituting $\beta'=\frac{\beta}{(d+1)^{2k_{2}}k_{2}^{10k_{2}}}$
so it can be computed in time $\poly((k_{2}d)^{k_{2}}/\beta)$.
\end{proof}

\section{Spectral Tester}\label{appendix:spectral-tester}

In this section we prove the following theorem.
\begin{theorem}
\label{thm: spectral tester} There exists some absolute constant $C$ and a deterministic
algorithm $\Tspectral$ that,  given
\begin{itemize}
\item A positive integer $U\geq\left(\frac{Cd}{\epsilon\delta}\right)^{C}.$
\item a dataset $S$ of points in $\R^{d}$ of size $M\leq U$. 
\item a unit vector $\vect v$ in $\R^{d}$,  
\item parameters $\epsilon, \delta$ and $\mu$ in $(0, 1)$.
\end{itemize}
For every positive absolute constant
$\mu$,  the algorithm $\Tspectral$
runs in time $poly\left(\frac{dU}{\epsilon\delta}\right)$ and outputs
$\accept$or output $\reject$. For all $\epsilon,  \delta$ and $U\geq\left(\frac{Cd}{\epsilon\delta}\right)^{C}$ the algorithm $\Tspectral$ satisfies the following:
\begin{itemize}
\item \textbf{Completeness:} 
 If $S$ consists of $M\leq U$ i.i.d. samples from the
standard Gaussian distribution,  then with probability at least $1-O\left(\delta\right)$ the set $S$ is such that for all unit vectors $\vect{v}$ the algorithm $\Tspectral$ accepts when given $(U,  S,  \vect v,  \epsilon,  \delta,  \mu)$ as the input.
\item \textbf{Monotonicity under Datapoint Removal: }If the algorithm $\Tspectral$
outputs $\accept$ for some specific input $(U,  S,  \vect v,  \epsilon,  \delta,  \mu)$,  then for all subsets $S'\subset S$
the tester $\Tspectral$ will also accept the input $(U,  S',  \vect v,  \epsilon,  \delta,  \mu)$.
\item \textbf{Soundness:} For any dataset $S$ and unit vector $\vect v$, 
if the tester $\mathcal{\Tspectral}$ accepts the input $(U,  S,  \vect v,  \epsilon,  \delta,  \mu)$ then for every unit vector $\vect v'$ in $\R^{d}$
we have
\[
\frac{1}{U}\sum_{\vect x\in S}\left[\indicator_{\sign(\vect x\cdot\vect v)\neq\sign(\vect x\cdot\vect v')}\right]\leq(1+\mu)\frac{\measuredangle(\vect v, \vect v')}{\pi}+O(\epsilon).
\]

\end{itemize}
\end{theorem}

We argue that the following algorithm (which is essentially a restatement of \Cref{algorithm:spectral-tester}) satisfies the specifications
above:
\begin{itemize}
\item \textbf{Given: }parameter $\epsilon, \delta,  \mu$ in $(0, 1)$,  dataset
$S$ of points in $\R^{d}$ of size $M\leq U$,  a unit vector $\vect v$
in $\R^{d}$,  
\end{itemize}
\begin{enumerate}
\item $k_{1}\leftarrow\frac{2\sqrt{\log2/\epsilon}}{\epsilon}$, $k_{2}\leftarrow\frac{C^{0.1}}{\mu^{5}}$
\item $\Delta \leftarrow \frac{60 (4k_2)^{2k_2+2} (d+1)^{6k_2+2} }{\delta}
(
\frac{ \log U}{U }
)^{1/4}$
\item For all $a$ and \textbf{$b$} in $\left\{ -\infty, -k_{1}\epsilon, -(k_{1}-1)\epsilon, \cdots, -\epsilon, 0, +\epsilon, \cdots, (k_{1}-1)\epsilon, k_{1}\epsilon, +\infty\right\} $
\begin{enumerate}
\item Compute $W^{a, b}$ such that 
\begin{multline}
W^{a, b}-
\Delta
I_{\binom{d+1}{k_{2}}\times\binom{d+1}{k_{2}}}
\preceq
\underset{\vect x\sim\N}{\E}\bigr[\bigr(\vect x^{\otimes k_{2}}\bigr)\bigr(\vect x^{\otimes k_{2}}\bigr)^{\top}\cdot\indicator_{a\leq\vect x\cdot\vect v<b}\bigr]\preceq
W^{a, b}+\Delta I_{\binom{d+1}{k_{2}}\times\binom{d+1}{k_{2}}}\label{eq: deterministic spectral approximation}
\end{multline}
(For how to compute this approximation,  see Claim \ref{claim: deterministic spectral approximation to moments in a strip}).
\item If the following does not hold: 
\begin{equation}
\frac{1}{U}\sum_{\vect x\in S}\left(\vect x^{\otimes k_{2}}\right)\left(\vect x^{\otimes k_{2}}\right)^{\top}\indicator_{a\leq\vect x\cdot\vect v<b}\preceq W^{a, b}+3\Delta I_{\binom{d+1}{k_{2}}\times\binom{d+1}{k_{2}}}, \label{eq: spectral condition}
\end{equation}
then output $\reject$.
\end{enumerate}
\item If did not reject in any previous step,  output $\accept$.
\end{enumerate}
It is immediate that the algorithm indeed runs in time $poly\left(\frac{dU}{\epsilon\delta}\right)$, 
because step (2b) can be performed by computing the largest eigenvalue
of a $\binom{d+1}{k_{2}}\times\binom{d+1}{k_{2}}$-sized matrix.
Monotonicity over datapoint removal also follows immediately since
if $S'\subset S$ then 
\[
\frac{1}{U}\sum_{\vect x\in S'}\left(\vect x^{\otimes k_{2}}\right)\left(\vect x^{\otimes k_{2}}\right)^{\top}\preceq\frac{1}{U}\sum_{\vect x\in S}\left(\vect x^{\otimes k_{2}}\right)\left(\vect x^{\otimes k_{2}}\right)^{\top}, 
\]
and therefore if the condition in step (4) holds for $S$ then it
will also hold for $S'$.

\subsection{Completeness}

Since we have already proven the property of monotonicity under datapoint
removal,  we can assume without loss of generality that $M=U$. If
not,  the set $S$ can be obtained by first taking $U$ samples from
$\N(0, I_{d})$ and then removing the last $U-M$ of them. If the $\Tspectral$
accepted the dataset before removing these points,  then it will also
accept it after these datapoints are removed.

Suppose the set dataset $S$ consists of $U$ i.i.d. samples from
$\N(0, I_{d})$. 
Similar to Section \ref{subsec: disagreement tester completeness},  we again note that the collection $\mathcal{H}$ of sets of the form $\indicator_{a\leq \vect v \cdot \vect x < b}$ has VC dimension at most $(d+1)^2$.
Lemma \ref{lem: concentration for moments} then
implies that with probability at least $1-\delta$ for all  pairs
of $a$ and $b$,  for every unit vector $\vect v$ and for every polynomial $p$ of degree at most
$k_{2}$,  if $B_{p}$ denotes the largest coefficient of $p$ (in absolute
value) then we have
\[
\left|\frac{1}{U}\sum_{\vect x\in S}\left[\left(p(\vect x)\right)^{2}\indicator_{\vect x\in A}\right]-\E_{\vect x\sim\N(0, I_{d})}\left[\left(p(\vect x)\right)^{2}\indicator_{\vect x\in A}\right]\right|\leq2B_{p}^{2}\left(d+1\right)^{5k_{2}}\sqrt{\frac{(4k_{2})^{2k_{2}+2}}{\delta U}}, 
\]
\[
\left|\frac{1}{U}\sum_{\vect x\in S}\left[\left(p(\vect x)\right)^{2}\indicator_{\vect x\in A}\right]-\E_{\vect x\sim\N(0, I_{d})}\left[\left(p(\vect x)\right)^{2}\indicator_{\vect x\in A}\right]\right|\leq
\frac{60 B_p^2 (4k_2)^{2k_2+2} (d+1)^{6k_2+2} }{\delta}
\left(
\frac{ \log U}{U }
\right)^{1/4}, 
\]

Combining this with Equation \ref{eq: coefficient L2 norm of polynomial vs largest coefficient}
we get
\[
\frac{1}{U}\sum_{\vect x\in S}\left[\left(p(\vect x)\right)^{2}\indicator_{\vect x\in A}\right]\leq\E_{\vect x\sim\N(0, I_{d})}\left[\left(p(\vect x)\right)^{2}\indicator_{\vect x\in A}\right]+\left(\norm p_{\mathrm{coeff}}\right)^{2}
\frac{60 (4k_2)^{2k_2+2} (d+1)^{6k_2+2} }{\delta}
\left(
\frac{ \log U}{U }
\right)^{1/4}
\]
and Claim \ref{claim: deterministic spectral approximation to moments in a strip} implies that Equation \ref{eq: spectral approximation to moments in a strip}
holds which implies that 
\[
p^{\top}W^{a, b}p\geq\E_{\vect x\sim\N(0, I_{d})}\left[\left(p(\vect x)\right)^{2}\indicator_{\vect x\in A}\right]+\left(\norm p_{\mathrm{coeff}}\right)^{2}\frac{60 B_p^2 (4k_2)^{2k_2+2} (d+1)^{6k_2+2} }{\delta}
\left(
\frac{ \log U}{U }
\right)^{1/4}.
\]
Combining the last two equations above,  we get:

\[
\frac{1}{U}\sum_{\vect x\in S}\left[\left(p(\vect x)\right)^{2}\indicator_{\vect x\in A}\right]\leq p^{\top}W^{a, b}p+\left(\norm p_{\mathrm{coeff}}\right)^{2}\frac{120 B_p^2 (4k_2)^{2k_2+2} (d+1)^{6k_2+2} }{\delta}
\left(
\frac{ \log U}{U }
\right)^{1/4}.
\]
Recalling the notation in \ref{subsec: Some standard notation.}, 
we see that the assertion that the inequality above holds for every
$p$,  is equivalent to the following matrix inequality
\[
\frac{1}{U}\sum_{\vect x\in S}\left[\vect x^{\otimes k_{2}}\left(\vect x^{\otimes k_{2}}\right)^{\top}\indicator_{\vect x\in A}\right]\preceq W^{a, b}+\frac{120 B_p^2 (4k_2)^{2k_2+2} (d+1)^{6k_2+2} }{\delta}
\left(
\frac{ \log U}{U }
\right)^{1/4}
I_{\binom{d+1}{k_{2}}\times{\binom{d+1}{k_2}}}.
\]
Finally,  substituting $k_{2}=\frac{C^{0.1}}{\mu^{5}}$ and $U\geq\left(\frac{Cd}{\epsilon\delta}\right)^{C}$, 
we see that for a sufficiently large absolute constant $C$,  the inequality
above implies Equation \ref{subsec: Some standard notation.},  and
thus $\Tspectral$ accepts.

\subsection{Soundness}
In order to deduce the soundness condition,  expand upon definitions
introduced in \ref{subsec: Soundness for disagreement tester}. We emphasize that unlike the $L_1$-sandwiching degree used to analyze the disagreement tester,  here we use the notion of $L_{2}$-sandwiching
polynomials.

\begin{definition}
We say that a function $f:\R^{d}\rightarrow\left\{ 0, 1\right\} $
is $\epsilon$-sandwiched in $L_{2}$ norm between a pair of functions
$f_{\text{up}}:\R^{d}\rightarrow\R$ and $f_{\text{down}}:\R^{d}\rightarrow\R$
under $\N(0, I_{d})$ if:
\begin{itemize}
\item For all $\vect x$ in $\R^{d}$ we have $f_{\text{down}}(\vect x)\leq f(\vect x)\leq f_{\text{up}}(\vect x)$
\item $\E_{\vect x\sim\N(0, I_{d})}\left[\left(f_{\text{up}}(\vect x)-f_{\text{down}}(\vect x)\right)^{2}\right]\leq\epsilon$.
\end{itemize}
\end{definition}

\begin{definition}
\label{def: sandwiching wrt partition-1} We say that a function $f:\R^{d}\rightarrow\left\{ 0, 1\right\} $
has $(\epsilon, B)$-sandwiching degree of at most $k$ in $L_{2}$
norm under $\N(0, I_{d})$ with respect to a partition $\mathcal{C}$
of $\R^{d}$ if the function $f$ is $\epsilon$-sandwiched in $L_{2}$
norm under $\N(0, I_{d})$ between $\sum_{A\in\mathcal{C}}\left(p_{\text{down}}^{A}\indicator_{A}\right)$
and $\sum_{A\in\mathcal{C}}\left(p_{\text{up}}^{A}\indicator_{A}\right)$, 
where $p_{\text{up}}^{A}$ and $p_{\text{down}}^{A}$ are degree$-k$
polynomials over $\R^{d}$whose coefficients are bounded by $B$ in
absolute value.
\end{definition}

Subsection \ref{subsec: bounding sandwiching degree-1} is dedicated
to proving the following bound on the $L_{_{2}}$-sandwiching degree
of a specific family of functions with respect to a specific partition
of $\R^{d}$.
\begin{proposition}
\label{thm: sandwiching degree of purple-1}For all $\epsilon$ and
$k_{2}$,  let $k_{1}=\frac{2\sqrt{\log2/\epsilon}}{\epsilon}$,  and
let $\vect v$ be a unit vector in $\R^{d}$. Then,  there exists a
partition $\mathcal{C}$ of $\R^{d}$ consisting of sets of the form
$\{ \vect x\in\R^{d}:a\leq\vect v\cdot\vect x\leq b\} $
for a certain collection of pairs $a, b$ in $\{ -\infty, -k_{1}\epsilon, -(k_{1}-1)\epsilon, \cdots, -\epsilon, 0, +\epsilon, \cdots, (k_{1}-1)\epsilon, k_{1}\epsilon, +\infty\} $.
Then,  for every unit vector $\vect v'$,  the function $f(\vect x)=\indicator_{\sign(\vect v\cdot\vect x)\neq\sign(\vect v'\cdot\vect x)}$
has $(O(\measuredangle(\vect v, \vect v')\frac{\log^{5}k_{2}}{k_{2}^{1/4}}\cdot)+10\epsilon, O(d^{10k_{2}}))$-sandwiching
degree of at most $k_{2}$ in $L_{2}$ norm under $\N(0, I_{d})$ with
respect to the partition $\mathcal{C}$ of $\R^{d}$.
\end{proposition}

A bound on the sandwiching degree of a class of functions leads to
a guarantee for the tester $\Tdisagree$:

\begin{proposition}
\label{prop: sandwiching wrt covering implies fooling-1} Let $\mathcal{C}$
be a partition of $\R^{d}$ and $f$ a $\left\{ 0, 1\right\} $-valued
function that has $(\nu, B)$-sandwiching degree of at most $k_{2}$
in $L_{2}$ norm under $\N(0, I_{d})$ with respect to the partition
$\mathcal{C}$. If a set $S$ of points in $\R^{d}$ satisfies the
following condition for all $A$ in $\mathcal{C}$ : 
\begin{equation}
\frac{1}{U}\sum_{\vect x\in S}(\vect x^{\otimes k_{2}})(\vect x^{\otimes k_{2}})^{\top}\indicator_{\vect x\in A}\preceq\E_{\vect x\sim\N(0, I_{d})}[(\vect x^{\otimes k_{2}})(\vect x^{\otimes k_{2}})^{\top}\indicator_{\vect x\in A}]+\frac{\epsilon^{2}}{|\mathcal{C}|B^{2}(d+1)^{k_{2}}}I_{\binom{d+1}{k_{2}}\times\binom{d+1}{k_{2}}}\label{eq:moment-matching premise-1}
\end{equation}
then we have
\[
\sqrt{\frac{1}{U}\sum_{\vect x\sim S}\left[\indicator_{f(\vect x)=1}\right]}\leq\sqrt{\pr_{\vect x\sim\N(0, I_{d})}[f(\vect x)=1]}+\sqrt{\nu}+\epsilon.
\]
\end{proposition}

\begin{proof}
Since $f$ has $(\nu, B)$-sandwiching degree of at most $k_{2}$ in
$L_{2}$ norm under $\N(0, I_{d})$ with respect to the partition $\mathcal{C}$, 
we have a collection of polynomials $\left\{ p_{\text{down}}^{A}, p_{\text{up}}^{A}\right\} $for
all $A$ in $\mathcal{C}$ that have coefficients bounded by $B$, 
satisfy for all $\vect x$ the condition 
\begin{equation}
f(\vect x)\in\Bigr[\sum_{A\in\mathcal{C}}\Bigr(p_{\text{down}}^{A}(\vect x)\indicator_{\vect x\in A}\Bigr), \sum_{A\in\mathcal{C}}\Bigr(p_{\text{up}}^{A}(\vect x)\indicator_{\vect x\in A}\Bigr)\Bigr], \label{eq: f is sandwiched-1}
\end{equation}
as well as
\begin{equation}
\E_{\vect x\sim\N(0, I_{d})}\Bigr[\Bigr(\sum_{A\in\mathcal{C}}\Bigr(p_{\text{up}}^{A}(\vect x)\indicator_{\vect x\in A}\Bigr)-\sum_{A\in\mathcal{C}}\Bigr(p_{\text{down}}^{A}(\vect x)\indicator_{\vect x\in A}\Bigr)\Bigr)^{2}\Bigr]\leq\nu.\label{eq: sandwiching-1}
\end{equation}
For all $\vect x$ in $\R^{d}$ we have $f(\vect x)\leq(f(\vect x))^{2}\leq(\sum_{A\in\mathcal{C}}(p_{\text{up}}^{A}(\vect x)\indicator_{\vect x\in A}))^{2}$.
Since all distinct pairs $A_{1}, A_{2}$ in $\mathcal{C}$ are disjoint, 
we have $(\sum_{A\in\mathcal{C}}(p_{\text{up}}^{A}(\vect x)\indicator_{\vect x\in A}))^{2}=\sum_{A\in\mathcal{C}}(p_{\text{up}}^{A}(\vect x)\indicator_{\vect x\in A})^{2}$.
Therefore,  we have
\begin{equation}
\frac{1}{U}\sum_{\vect x\sim S}\left[f(\vect x)\right]\leq\sum_{A\in\mathcal{C}}\left(\frac{1}{U}\sum_{\vect x\sim S}\left[\left(p_{\text{up}}^{A}(\vect x)\indicator_{\vect x\in A}\right)^{2}\right]\right)\label{eq: average of f vs average of pup}
\end{equation}
Referring to definitions in Subsection \ref{subsec: Some standard notation.}, 
we see that Equation \ref{eq:moment-matching premise-1} is equivalent
to the assertion that for every $A$ in $\mathcal{C}$ and every degree-$k_{2}$
polynomial $p$ we have 
\[
\frac{1}{U}\sum_{\vect x\in S}\left(p(\vect x)\right)^{2}\indicator_{\vect x\in A}\leq\E_{\vect x\sim\N(0, I_{d})}\left[\left(p(\vect x)\right)^{2}\indicator_{\vect x\in A}\right]+\frac{\epsilon^{2}}{|\mathcal{C}|B^{2}\left(d+1\right)^{k_{2}}}\left(\norm p_{\mathrm{coeff}}\right)^{2}.
\]
Choosing $p=p_{\text{up}}^{A}$ in the inequality above and combining
with Equation \ref{eq: average of f vs average of pup} we get:
\[
\frac{1}{U}\sum_{\vect x\sim S}\Bigr[f(\vect x)\Bigr]\leq\sum_{A\in\mathcal{C}}\Bigr(\E_{\vect x\sim\N(0, I_{d})}[(p_{\text{up}}^{A}(\vect x)\indicator_{\vect x\in A})^{2}]+\frac{\epsilon^{2}}{|\mathcal{C}|B^{2}(d+1)^{k_{2}}}(\norm{p_{\text{up}}^{A}}_{\mathrm{coeff}})^{2}\Bigr).
\]
By Equation \ref{eq: coefficient L2 norm of polynomial vs largest coefficient}, 
we have $\norm{p_{\text{up}}^{A}}_{\mathrm{coeff}}\leq B^{2}d^{k_{2}}$.
Substituting this and again recalling that all distinct pairs $A_{1}, A_{2}$
in $\mathcal{C}$ are disjoint,  we obtain

\[
\frac{1}{U}\sum_{\vect x\sim S}\left[f(\vect x)\right]\leq\E_{\vect x\sim\N(0, I_{d})}\Bigr[\Bigr(\sum_{A\in\mathcal{C}}p_{\text{up}}^{A}(\vect x)\indicator_{\vect x\in A}\Bigr)^{2}\Bigr]+\epsilon^{2}.
\]
Taking square roots of both sides gives us 
\begin{align}
\sqrt{\frac{1}{U}\sum_{\vect x\sim S}\left[f(\vect x)\right]} &\leq\sqrt{\E_{\vect x\sim\N_d}\Bigr[\Bigr(\sum_{A\in\mathcal{C}}p_{\text{up}}^{A}(\vect x)\indicator_{\vect x\in A}\Bigr)^{2}\Bigr]+\epsilon^{2}} \nonumber\\
&\leq\sqrt{\E_{\vect x\sim\N_d}\Bigr[\Bigr(\sum_{A\in\mathcal{C}}p_{\text{up}}^{A}(\vect x)\indicator_{\vect x\in A}\Bigr)^{2}\Bigr]}+\epsilon.\label{eq: L2 norm of f vs L2 norm of pup}
\end{align}

Equation \ref{eq: f is sandwiched-1},  together with the triangle
inequality and the fact that $(f(\vect x))^{2}=f(\vect x)$,  implies
that 
\begin{align*}
&\sqrt{\E_{\vect x\sim\N(0, I_{d})}\Bigr[\Bigr(\sum_{A\in\mathcal{C}} p_{\text{up}}^{A}(\vect x)\indicator_{\vect x\in A}\Bigr)^{2}\Bigr]}\leq \\
&\le \sqrt{\E_{\vect x\sim\N(0, I_{d})}\Bigr[f(\vect x)\Bigr]}+\sqrt{\E_{\vect x\sim\N(0, I_{d})}\Bigr[\Bigr(\sum_{A\in\mathcal{C}}p_{\text{up}}^{A}(\vect x)\indicator_{\vect x\in A}-f(\vect x)\Bigr)^{2}\Bigr]}\\
&\le \sqrt{\E_{\vect x\sim\N(0, I_{d})}\Bigr[f(\vect x)\Bigr]}+\sqrt{\E_{\vect x\sim\N(0, I_{d})}\Bigr[\Bigr(\sum_{A\in\mathcal{C}}p_{\text{up}}^{A}(\vect x)\indicator_{\vect x\in A}-\sum_{A\in\mathcal{C}}p_{\text{down}}^{A}(\vect x)\indicator_{\vect x\in A}\Bigr)^{2}\Bigr]}
\end{align*}
Substituting Equation \ref{eq: sandwiching-1},  we get 
\[
\sqrt{\E_{\vect x\sim\N(0, I_{d})}\Bigr[\Bigr(\sum_{A\in\mathcal{C}}p_{\text{up}}^{A}(\vect x)\indicator_{\vect x\in A}\Bigr)^{2}\Bigr]}\leq\sqrt{\E_{\vect x\sim\N(0, I_{d})}\left[f(\x)\right]}+\sqrt{\nu}, 
\]
which combined with Equation \ref{eq: L2 norm of f vs L2 norm of pup}
finishes the proof. 
\end{proof}
Claim \ref{claim: deterministic spectral approximation to moments in a strip} implies that matrices $W^{a, b}$ satisfy Equation \ref{eq: deterministic spectral approximation},  which implies that for all monomials $p$ of degree at most $k_{2}$
 we have 
\[
W^{a, b}\preceq\E_{\vect x\sim\N(0.I_{d})}\Bigr[\left(\vect x^{\otimes k_{2}}\right)\left(\vect x^{\otimes k_{2}}\right)^{\top}\indicator_{a\leq\vect x\cdot\vect v<b}\Bigr]+ \Delta I_{\binom{d+1}{k_{2}}\times\binom{d+1}{k_{2}}}
\]
If the above is the case,  and the algorithm accepts,  then we have
for all pairs of $a$ and $b$ in 

\noindent $\left\{ -\infty, -k_{1}\epsilon, -(k_{1}-1)\epsilon, \cdots, -\epsilon, 0, +\epsilon, \cdots, (k_{1}-1)\epsilon, k_{1}\epsilon, +\infty\right\} $
that
\begin{multline*}
\frac{1}{U}\sum_{\vect x\in S}\left(\vect x^{\otimes k_{2}}\right)\left(\vect x^{\otimes k_{2}}\right)^{\top}\indicator_{a\leq\vect x\cdot\vect v<b}\preceq\\
\E_{\vect x\sim\N(0.I_{d})}
\left[\left(\vect x^{\otimes k_{2}}\right)\left(\vect x^{\otimes k_{2}}\right)^{\top}\indicator_{a\leq\vect x\cdot\vect v<b}
\right]+
\frac{210 (4k_2)^{2k_2+2} (d+1)^{6k_2+2} }{\delta}
\left(
\frac{ \log U}{U }
\right)^{1/4}
I_{\binom{d+1}{k_{2}}\times\binom{d+1}{k_{2}}}.
\end{multline*}
Taking the equation above,  together with Proposition \ref{prop: sandwiching wrt covering implies fooling-1}
and Proposition \ref{thm: sandwiching degree of purple-1} we conclude
that for all unit vectors $\vect{v}'$:
\begin{multline*}
\sqrt{\frac{1}{U}\sum_{\vect x\sim S}\left[\indicator_{\sign(\vect x\cdot\vect v)\neq\sign(\vect x\cdot\vect v')}\right]}\leq\sqrt{\pr_{\vect x\sim\N(0, I_{d})}[\sign(\vect v\cdot\vect x)\neq\sign(\vect v'\cdot\vect x)]}+\sqrt{O\left(\frac{\log^{5}k_{2}}{k_{2}^{1/4}}\measuredangle(\vect v, \vect v')\right)}+\\
+O\left(
\sqrt{
(2k_{1}+2)
\frac{(4k_2)^{2k_2+2} (d+1)^{7k_2+2} }{\delta}
\left(
\frac{ \log U}{U }
\right)^{1/4}
d^{10k_{2}}
}
\right)
\end{multline*}
Substituting $k_{1}\leftarrow\frac{2\sqrt{\log2/\epsilon}}{\epsilon}$, 
$k_{2}\leftarrow\frac{C^{0.1}}{\mu^{5}}$,  $Y\leftarrow C\left(\frac{\left(kd\right)^{k_{2}}}{\delta}\right)^{C}$, 
taking $C$ to be a sufficiently large absolute constant and recalling
that $\pr_{\vect x\sim\N(0, I_{d})}[\sign(\vect v\cdot\vect x)\neq\sign(\vect v'\cdot\vect x)]$
equals to $\measuredangle(\vect v, \vect v')/\pi$ we conclude that 
\[
\frac{1}{U}\sum_{\vect x\sim S}\left[\indicator_{\sign(\vect x\cdot\vect v)\neq\sign(\vect x\cdot\vect v')}\right]\leq(1+\mu)\frac{\measuredangle(\vect v, \vect v')}{\pi}+O(\epsilon).
\]

\subsection{Bounding the $L_2$ sandwiching degree of the disagreement region}

\label{subsec: bounding sandwiching degree-1}

To prove Proposition \ref{thm: sandwiching degree of purple-1},  we follow an exactly analogous approach as the one for \Cref{thm: sandwiching degree of purple}. We will need the following result from \cite{klivans2023testable}:
\begin{fact}[\cite{klivans2023testable}]
\label{fact:l2 sandwiching for halfspaces}For every positive integer
$k$ and a real value $t$,  the function $f(z)=\indicator_{z\leq t}$
has $(O(\frac{\log^{10}k}{\sqrt{k}}), O(2^{10k}))$-sandwiching
degree in $L_{2}$ norm of at most $k$ under $\mathcal{N}(0, 1)$.
\end{fact}

The following corollary slightly strengthens the fact above: 
\begin{corollary}
\label{corr: slightly stengthened 1-d sandwiching-1}For every positive
integer $k\geq4$ and a real value $t$,  the function $f(z)=\indicator_{z\leq t}$
is 

\noindent $(O(\min(\frac{\log^{10}k}{\sqrt{k}}, \frac{1}{t^{2}})), 2^{10k})$-sandwiched
in $L_{2}$ norm under $\mathcal{N}(0, 1)$ between a pair of polynomials
$J_{\text{up}}^{t}$ and $J_{\text{down}}^{t}$ of degree of at most
$k$. 
\end{corollary}

\begin{proof}
Indeed,  if $\frac{\log^{10}k}{\sqrt{k}}\leq\frac{1024}{t^{2}}$ then
the corollary follows from Fact \ref{fact:l2 sandwiching for halfspaces}.
So all we need to do is to consider the other case. We see that either
$t>1$ or $t<-1$ (since $k\geq4)$. if $t>1$ we take $p_{\text{down}}(\x)=0$
and $p_{\text{up}}(\x)=\left(\frac{x}{t}\right)^{2}$. If $t<-1$, 
we take take $p_{\text{up}}(\x)=1$ and $p_{\text{down}}(\x)=1-\left(\frac{x}{t}\right)^{2}$.
In either case, the polynomials $p_{\text{down}}$,
$p_{\text{up}}$ form a pair of $(O(\min(\frac{\log^{10}k}{\sqrt{k}}, \frac{1}{t^{2}})), 1)$-sandwiching
polynomials of degree $2$.
\end{proof}
Let $\vect v_{\perp}$ be the unit vector equal up to scaling to the
component of $\vect v'$ perpendicular to $\vect v$. Then,  we can
write

\begin{align}
\psi_{\text{down}}(\vect x) & \leq\psi_{\text{up}}(\vect x)\leq\indicator_{\vect v\cdot\vect x\geq0\land\vect v'\cdot\vect x<0}\label{eq: psis one-1}\\
\psi_{\text{down}}(\vect x) & \leq\psi_{\text{up}}(\vect x)\leq\indicator_{\vect v\cdot\vect x>0\land\vect v'\cdot\vect x\leq0}\label{eq: psis two-1}
\end{align}

where

\[
\psi_{\text{up}}(\vect x)=\begin{cases}
1 & \text{if }\text{\ensuremath{\vect v\cdot\vect x}\ensuremath{\ensuremath{\geq k_{1}\epsilon}}}\text{ or }\vect v\cdot\vect x\ensuremath{=0}, \\
\indicator_{\vect v_{\perp}\cdot\vect x\tan\angleparam\leq-j\epsilon} & \text{\text{if }\text{\ensuremath{\vect v\cdot\vect x}\ensuremath{\ensuremath{\neq}0} and }}\vect v\cdot\vect x\in\left[j\epsilon, (j+1)\epsilon\right)\text{ for }0\leq j\leq k_{1}-1, \\
0 & \text{if }\text{\ensuremath{\vect v\cdot\vect x}\ensuremath{<0}}
\end{cases}
\]

\[
\psi_{\text{down}}(\vect x)=\begin{cases}
0 & \text{if }\text{\ensuremath{\vect v\cdot\vect x}\ensuremath{\ensuremath{\geq k_{1}\epsilon}}}\text{ or }\vect v\cdot\vect x\ensuremath{=0}\\
\indicator_{\vect v_{\perp}\cdot\vect x\tan\angleparam<-(j+1)\epsilon} & \text{if }\text{\ensuremath{\vect v\cdot\vect x}\ensuremath{\ensuremath{\neq}0} and }\vect v\cdot\vect x\in\left[j\epsilon, (j+1)\epsilon\right)\text{ for }0\leq j\leq k_{1}-1\\
0 & \text{if }\text{\ensuremath{\vect v\cdot\vect x}\ensuremath{<0}}
\end{cases}
\]

Recall that for every $t\in\R$,  Corollary \ref{corr: slightly stengthened 1-d sandwiching-1}
gives us one-dimensional degree-$k_{2}$ sandwiching polynomials $J_{\text{down}}^{t}(z)$
and $J_{\text{up}}^{t}(z)$ for $\indicator_{z\leq t}$ under $L_{2}$
norm. Using this notation,  we have for all $\vect x$ in $\R^{d}$
\begin{multline}
\overbrace{\sum_{j=0}^{k_{1}-1}\indicator_{\vect v\cdot\vect x\cdot\left[j\epsilon, (j+1)\epsilon\right)}J_{\text{down}}^{-(j+1)\epsilon/\tan\angleparam}(\vect v_{\perp}\cdot\vect x)}^{\text{Denote this \ensuremath{\phi_{\text{down}}^{L_{2}}(\vect x)}}}\leq\psi_{\text{down}}(\vect x)\leq\indicator_{\vect v\cdot\vect x\geq0\land\vect v'\cdot\vect x<0}\leq\\
\leq\psi_{\text{up}}(\vect x)\leq\underbrace{\indicator_{\vect v\cdot\vect x\ensuremath{\geq k_{1}\epsilon}}+\sum_{j=0}^{k_{1}-1}\indicator_{\vect v\cdot\vect x\cdot\left[j\epsilon, (j+1)\epsilon\right)}J_{\text{up}}^{-j\epsilon/\tan\angleparam}(\vect v_{\perp}\cdot\vect x)}_{\text{Denote this \ensuremath{\phi_{\text{up}}^{L_{2}}(\vect x)}}}\label{eq: phi are sandwichers-1}
\end{multline}

In order to conclude Proposition \ref{thm: sandwiching degree of purple-1}.
We show the following two claims:
\begin{claim}
\label{claim: sandwiching error small-1}We have 
\[
\E_{\vect x\sim\N(0, I_{d})}\left[\left(\phi_{\text{up}}^{L_{2}}(\vect x)-\phi_{\text{down}}^{L_{2}}(\vect x)\right)^{2}\right]\leq O\left(\frac{\log^{5}k_{2}}{k_{2}^{1/4}}\cdot\measuredangle(\vect v, \vect v')\right)+10\epsilon
\]
\end{claim}

\begin{claim}
\label{claim:coefficient small-1}For all integers $j$ in $\left[0, k_{1}-1\right]$, 
every coefficient of $J_{\text{down}}^{-(j+1)\epsilon/\tan\angleparam}(\vect v_{\perp}\cdot\vect x)$
and $J_{\text{up}}^{-j\epsilon/\tan\angleparam}(\vect v_{\perp}\cdot\vect x)$
is at most $O\left(d^{10k_{2}}\right)$ in absolute value.
\end{claim}

Proposition \ref{thm: sandwiching degree of purple-1} follows from
the two claims above for as follows. We first observe that Equations
\ref{eq: psis two-1} and \ref{eq: phi are sandwichers-1} imply that 

\[
\phi_{\text{down}}^{L_{2}}(-\vect x)\leq\indicator_{\vect v\cdot\vect x<0\land\vect v'\cdot\vect x\geq0}\leq\phi_{\text{up}}^{L_{2}}(-\vect x).
\]
Recalling our convention that $\sign(0)=1$,  we see that
\[
\indicator_{\sign(\vect v\cdot\vect x)\neq\sign(\vect v'\cdot\vect x)}=\indicator_{\vect v\cdot\vect x\geq0\land\vect v'\cdot\vect x<0}+\indicator_{\vect v\cdot\vect x<0\land\vect v'\cdot\vect x\geq0}
\]
this,  together with \ref{eq: phi are sandwichers-1} allows us to
bound

\begin{equation}
\phi_{\text{down}}^{L_{2}}(\vect x)+\phi_{\text{down}}^{L_{2}}(-\vect x)\leq\indicator_{\sign(\vect v\cdot\vect x)\neq\sign(\vect v'\cdot\vect x)}\leq\phi_{\text{up}}^{L_{2}}(\vect x)+\phi_{\text{up}}^{L_{2}}(-\vect x), \label{eq: sandwiching for the whole disagreement region-1}
\end{equation}
Claim \ref{eq: phi are sandwichers-1} allows us to conclude that
\begin{multline}
\E_{\vect x\sim\N(0, I_{d})}\left[\left(\phi_{\text{up}}^{L_{2}}(\vect x)+\phi_{\text{up}}^{L_{2}}(-\vect x)-\phi_{\text{down}}^{L_{2}}(\vect x)-\phi_{\text{down}}^{L_{2}}(-\vect x)\right)^{2}\right]\leq\\
2\left(\E_{\vect x\sim\N(0, I_{d})}\left[\left(\phi_{\text{up}}^{L_{2}}(\vect x)-\phi_{\text{down}}^{L_{2}}(\vect x)\right)^{2}\right]+\E_{\vect x\sim\N(0, I_{d})}\left[\left(\phi_{\text{up}}^{L_{2}}(-\vect x)-\phi_{\text{down}}^{L_{2}}(-\vect x)\right)^{2}\right]\right)\leq\\
O\left(\frac{\log^{1.5}k_{2}}{k_{2}^{1/4}}\cdot\measuredangle(\vect v, \vect v')\right)+20\epsilon.\label{eq: closensess for the whole disagreement region-1}
\end{multline}
Equations \ref{eq: sandwiching for the whole disagreement region-1}
and \ref{eq: closensess for the whole disagreement region-1},  together
with comparing the definition of $\phi_{\text{up}}$ and $\phi_{\text{down}}$
with Definition \ref{def: sandwiching wrt partition-1} and recalling
Claim \ref{claim:coefficient small-1},  allow us to conclude that
there exists a partition $\mathcal{C}$ of $\R^{d}$ consisting of
sets of the form $\left\{ \vect x\in\R^{d}:a\leq\vect v\cdot\vect x\leq b\right\} $
for a certain collection of pairs $a, b$ in

\noindent $\left\{ -\infty, -k_{1}\epsilon, -(k_{1}-1)\epsilon, \cdots, -\epsilon, 0, +\epsilon, \cdots, (k_{1}-1)\epsilon, k_{1}\epsilon, +\infty\right\} $, 
such that for every unit vector $\vect v'$,  the function $f(\vect x)=\indicator_{\sign(\vect v\cdot\vect x)\neq\sign(\vect v'\cdot\vect x)}$
has $\left(O\left(\frac{\measuredangle(\vect v, \vect v')}{k_{2}^{1/4}}\right)+10\epsilon, O\left(d^{10k_{2}}\right)\right)$-sandwiching
degree of at most $k_{2}$ in $L_{1}$ norm under $\N(0, I_{d})$ with
respect to the partition $\mathcal{C}$ of $\R^{d}$. This implies
Proposition \ref{thm: sandwiching degree of purple-1}. 

We now proceed to proving Claim \ref{claim: sandwiching error small-1}
\begin{proof}[Proof of Claim \ref{claim: sandwiching error small-1}]
We have:
\begin{multline}
\E_{\vect x\sim\N(0, I_{d})}\left[\left(\psi_{\text{up}}(\vect x)-\psi_{\text{down}}(\vect x)\right)^{2}\right]=\E_{\vect x\sim\N(0, I_{d})}\left[\psi_{\text{up}}(\vect x)-\psi_{\text{down}}(\vect x)\right]\leq\\
\pr_{\vect x\sim\N(0, I_{d})}\left[\vect v\cdot\vect x>k_{1}\epsilon\right]+\sum_{j=0}^{k_{1}-1}\pr_{\vect x\sim\N(0, I_{d})}\left[\left\{ \vect v\cdot\vect x\in\left[j\epsilon, (j+1)\epsilon\right)\right\} \land\left\{ \vect v_{\perp}\cdot\vect x\tan\angleparam\in\left[-(j+1)\epsilon, -j\epsilon\right)\right\} \right]\leq\\
e^{-\left(k_{1}\epsilon\right)^{2}}+\epsilon\underbrace{\sum_{j=0}^{\infty}\pr_{\vect x\sim\N(0, I_{d})}\left[\vect v_{\perp}\cdot\vect x\tan\angleparam\in\left[-(j+1)\epsilon, -j\epsilon\right)\right]}_{\leq1}\leq2\epsilon\label{eq: rounding to steps does not introduce much error-1}
\end{multline}

Let $\angleparam$ denote the angle $\measuredangle(\vect v, \vect v')$. Given
the inequality above,  in order to finish the proof of Claim \ref{claim: sandwiching error small-1}, 
it remains to upper-bound $\E_{\vect x\sim\N(0, I_{d})}\left[\phi_{\text{up}}(\vect x)-\psi_{\text{up}}(\vect x)\right]$
and $\E_{\vect x\sim\N(0, I_{d})}\left[\phi_{\text{up}}(\vect x)-\psi_{\text{up}}(\vect x)\right]$
by $O(\angleparam)+4\epsilon$. 

From Corollary \ref{corr: slightly stengthened 1-d sandwiching-1}
we know that for any $t$ we have:

\begin{equation}
\E_{z\sim\N(0, 1)}\left[\left(J_{\text{up}}^{t}(z)-J_{\text{down}}^{t}(z)\right)^{2}\right]\leq O\left(\min\left(\frac{\log^{10}k}{\sqrt{k}}, \frac{1}{t^{2}}\right)\right), \label{eq: 1-d step function indicator error-1}
\end{equation}
and for every $z$ in $\R$
\begin{equation}
J_{\text{down}}^{t}(z)\leq\indicator_{z\leq t}\leq J_{\text{up}}^{t}(z).\label{eq: indicator sandwiching-1}
\end{equation}

Since $\vect v$ and $\vect v_{\perp}$ are orthogonal,  the random
variables $\vect v_{\perp}\vect{\vect x}$ and $\vect v\cdot\vect x$
are independent standard Gaussians. Using this,  together with Equations
\ref{eq: 1-d step function indicator error-1} and \ref{eq: indicator sandwiching-1}
we get:
\begin{multline*}
\E_{\vect x\sim\N(0, I_{d})}\left[\left(\phi_{\text{up}}^{L_{2}}(\vect x)-\psi_{\text{up}}(\vect x)\right)^{2}\right]=\sum_{j=0}^{k_{1}-1}\E_{\vect x\sim\N(0, I_{d})}\left[\indicator_{\vect v\cdot\vect x\in\left[j\epsilon, (j+1)\epsilon\right)}\left(\phi_{\text{up}}^{L_{2}}(\vect x)-\psi_{\text{up}}^{L_{2}}(\vect x)\right)^{2}\right]\\
\leq\sum_{j=0}^{k_{1}-1}\pr_{\vect x\sim\N(0, I_{d})}\left[\vect v\cdot\vect x\in\left[j\epsilon, (j+1)\epsilon\right)\right]\E_{\vect x\sim\N(0, I_{d})}\left[\left(J_{\text{up}}^{-j\epsilon/\tan\angleparam}(\vect v_{\perp}\cdot\vect x)-\indicator_{\vect v_{\perp}\cdot\vect x\leq-j\epsilon/\tan\angleparam}\right)^{2}\right]=\\
\sum_{j=0}^{k_{1}-1}\pr_{z_{1}\sim\N(0, 1)}\left[z_{1}\in\left[j\epsilon, (j+1)\epsilon\right)\right]\E_{z_{2}\sim\N(0, 1)}\left[\left(J_{\text{up}}^{-j\epsilon/\tan\angleparam}(z_{2})-J_{\text{down}}^{-j\epsilon/\tan\angleparam}(z_{2})\right)^{2}\right]\leq\\
\sum_{j=0}^{k_{1}-1}\pr_{z_{1}\sim\N(0, 1)}\left[z_{1}\in\left[j\epsilon, (j+1)\epsilon\right)\right]O\left(\min\left(\frac{\log^{10}k_{2}}{\sqrt{k_{2}}}, \left(\frac{\tan\angleparam}{j\epsilon}\right)^{2}\right)\right)
\end{multline*}

First,  consider the case $\angleparam\geq\pi/4$. The above inequality
implies 
\begin{multline*}
\E_{\vect x\sim\N(0, I_{d})}\left[\left(\phi_{\text{up}}^{L_{2}}(\vect x)-\psi_{\text{up}}(\vect x)\right)^{2}\right]\leq O\left(\frac{\log^{10}k}{\sqrt{k}}\right)\underbrace{\sum_{j=0}^{k_{1}-1}\pr_{z_{1}\sim\N(0, 1)}\left[z_{1}\in\left[j\epsilon, (j+1)\epsilon\right)\right]}_{=\pr_{z_{1}\sim\N(0, 1)}\left[0\leq z_{1}<(k_{1})\epsilon\right]}=\\O\left(\frac{\log^{10}k_{2}}{\sqrt{k_{2}}}\right)=O\left(\frac{\angleparam\log^{5}k_{2}}{k_{2}^{1/4}}\right).
\end{multline*}
On the other hand,  if $\angleparam\leq\pi/4$ we have $\tan\angleparam\leq2\angleparam$
and therefore,  recalling that for any $j$ it is the case that $\pr_{z_{1}\sim\N(0, 1)}\left[z_{1}\in\left[j\epsilon, (j+1)\epsilon\right)\right]\leq\epsilon$, 
we have
\begin{align}
\E_{\vect x\sim\N(0, I_{d})}\Bigr[\Bigr(\phi_{\text{up}}^{L_{2}}(\vect x)-\psi_{\text{up}}(\vect x)\Bigr)^{2}\Bigr] &\leq\sum_{j=0}^{k_{1}-1}O\Bigr(\min\Bigr(\frac{\log^{10}k_{2}}{\sqrt{k_{2}}}, \Bigr(\frac{\tan\angleparam}{j\epsilon}\Bigr)^{2}\Bigr)\epsilon\Bigr) \nonumber\\
&\leq\int_{0}^{+\infty}O\Bigr(\min\Bigr(\frac{\log^{10}k_{2}}{\sqrt{k_{2}}}, \Bigr(\frac{\tan\angleparam}{z-\epsilon}\Bigr)^{2}\Bigr)\Bigr)\d z \nonumber\\
&\leq O\Bigr(\frac{\log^{10}k_{2}}{\sqrt{k_{2}}}\epsilon\Bigr)+\int_{0}^{+\infty}O\Bigr(\min\Bigr(\frac{\log^{10}k_{2}}{\sqrt{k_{2}}}, \Bigr(\frac{\tan\angleparam}{z}\Bigr)^{2}\Bigr)\Bigr)\d z \nonumber\\
&=O\Bigr(\frac{\log^{10}k_{2}}{\sqrt{k_{2}}}\epsilon\Bigr)+\tan\angleparam\int_{0}^{+\infty}O\Bigr(\min\Bigr(\frac{\log^{10}k_{2}}{\sqrt{k_{2}}}, \Bigr(\frac{1}{z}\Bigr)^{2}\Bigr)\Bigr)\d z \nonumber\\
&=O\Bigr(\frac{\log^{10}k_{2}}{\sqrt{k_{2}}}\epsilon+\tan\angleparam\Bigr(\frac{\log^{10}k_{2}}{\sqrt{k_{2}}}\Bigr(\frac{\sqrt{k_{2}}}{\log^{10}k_{2}}\Bigr)^{0.5}+\Bigr(\frac{\log^{10}k_{2}}{\sqrt{k_{2}}}\Bigr)^{0.5}\Bigr)\Bigr) \nonumber\\
&=O\Bigr(\frac{\angleparam\log^{5}k_{2}}{k_{2}^{1/4}}\Bigr)\label{eq: replacing the sum with an integral-1}
\end{align}
Overall,  in either case we have $\E_{\vect x\sim\N(0, I_{d})}\left[\left(\phi_{\text{up}}^{L_{2}}(\vect x)-\psi_{\text{up}}(\vect x)\right)^{2}\right]=O\left(\frac{\angleparam\log^{1.5}k_{2}}{k_{2}^{1/4}}\right)$.
We now go through a fully analogous argument to show that also $\E_{\vect x\sim\N(0, I_{d})}\left[\left(\psi_{\text{down}}(\vect x)-\phi_{\text{down}}^{L_{2}}(\vect x)\right)^{2}\right]=O\left(\frac{\angleparam\log^{1.5}k_{2}}{k_{2}^{1/4}}\right)$.
Again,  from the independence of $\vect v_{\perp}\vect{\vect x}$ and
$\vect v\cdot\vect x$,  together with Equations \ref{eq: 1-d step function indicator error-1}
and \ref{eq: indicator sandwiching-1} we have:
\begin{multline*}
\E_{\vect x\sim\N(0, I_{d})}\left[\left(\psi_{\text{down}}(\vect x)-\phi_{\text{down}}^{L_{2}}(\vect x)\right)^{2}\right]=\\
=\sum_{j=0}^{k_{1}-1}\pr_{\vect x\sim\N(0, I_{d})}\left[\vect v\cdot\vect x\in\left[j\epsilon, (j+1)\epsilon\right)\right]\E_{\vect x\sim\N(0, I_{d})}\left[\left(\indicator_{\vect v_{\perp}\cdot\vect x\leq-(j+1)\epsilon/\tan\angleparam}-J_{\text{down}}^{-(j+1)\epsilon/\tan\angleparam}(\vect v_{\perp}\cdot\vect x)\right)^{2}\right]\leq\\
\sum_{j=0}^{k_{1}-1}\pr_{z_{1}\sim\N(0, 1)}\left[z_{1}\in\left[j\epsilon, (j+1)\epsilon\right)\right]\E_{z_{2}\sim\N(0, 1)}\left[\left(R_{\text{up}}^{-(j+1)\epsilon/\tan\angleparam}(z_{2})-R_{\text{down}}^{-(j+1)\epsilon/\tan\angleparam}(z_{2})\right)^{2}\right]\leq\\
\sum_{j=0}^{k_{1}-1}\pr_{z_{1}\sim\N(0, 1)}\left[z_{1}\in\left[j\epsilon, (j+1)\epsilon\right)\right]O\left(\min\left(\frac{\log^{10}k_{2}}{\sqrt{k_{2}}}, \left(\frac{\tan\angleparam}{(j+1)\epsilon}\right)^{2}\right)\right)\leq\\
\sum_{j=0}^{k_{1}-1}\pr_{z_{1}\sim\N(0, 1)}\left[z_{1}\in\left[j\epsilon, (j+1)\epsilon\right)\right]O\left(\min\left(\frac{\log^{10}k_{2}}{\sqrt{k_{2}}}, \left(\frac{\tan\angleparam}{j\epsilon}\right)^{2}\right)\right)
\end{multline*}
As it was shown previously,  the expression above is at most $O\left(\frac{\angleparam\log^{5}k_{2}}{k_{2}^{1/4}}\right)$.

In total,  combining our bounds on $\E_{\vect x\sim\N(0, I_{d})}[(\psi_{\text{down}}(\vect x)-\phi_{\text{down}}^{L_{2}}(\vect x))^{2}]$, 
$\E_{\vect x\sim\N(0, I_{d})}[(\phi_{\text{up}}^{L_{2}}(\vect x)-\psi_{\text{up}}(\vect x))^{2}]$
and $\E_{\vect x\sim\N(0, I_{d})}[(\psi_{\text{up}}(\vect x)-\psi_{\text{down}}(\vect x))^{2}]$
we conclude that

\[
\E_{\vect x\sim\N(0, I_{d})}\left[\left(\phi_{\text{up}}^{L_{2}}(\vect x)-\phi_{\text{down}}^{L_{2}}(\vect x)\right)^{2}\right]\leq O\left(\frac{\log^{5}k_{2}}{k_{2}^{1/4}}\cdot\measuredangle(\vect v, \vect v')\right)+10\epsilon.
\]
\end{proof}
It only remains to prove Claim \ref{claim:coefficient small-1}.
\begin{proof}[Proof of Claim \ref{claim:coefficient small-1}]
Corrollary \ref{prop: sandwiching wrt covering implies fooling-1}
says that for any value of $t$,  the degree-$k_{2}$ one-dimensional
polynomials $J_{\text{up}}^{t}(z)$ and $J_{\text{down}}^{t}(z)$
have all their coefficients bounded by $O\left(2^{10k_{2}}\right)$.
If one substitutes $\vect v\cdot\vect x$ in place of $z$ into either
of these polynomials and opens the parentheses,  the fact that $\vect v$
is a unit vector allows us to bound the size of the larges coefficients
of $R_{\text{down}}^{t}(\vect v_{\perp}\cdot\vect x)$ and $R_{\text{up}}^{t}(\vect v_{\perp}\cdot\vect x)$
by $O\left(d^{k_{2}}(k_{2}+1)2^{10k_{2}}\right)=O(d^{10k_{2}})$, 
proving the claim.
\end{proof}

\subsection{Miscellaneous Claims}
\begin{claim}
\label{claim: deterministic spectral approximation to moments in a strip}There
is a deterministic algorithm that given a unit vector $\vect v$ in
$\R^{d}$,  scalars $a$ and $b$,  a monomial $m$ over $\R^{d}$ of
degree at most $k_{2}$,  an accuracy parameter $\beta\in(0, 1]$,  runs
in time $\poly\left(\left(k_{2}d\right)^{k_{2}}/\beta\right)$ and
computes a $\binom{d+1}{k_{2}}\times\binom{d+1}{k_{2}}$-matrix
$W^{a, b}$ such that

\begin{equation}
W^{a, b}-\beta I_{\binom{d+1}{k_{2}}\times\binom{d+1}{k_{2}}}\preceq\underset{\vect x\sim\N(0, I_{d})}{\E}\left[\left(\vect x^{\otimes k_{2}}\right)\left(\vect x^{\otimes k_{2}}\right)^{\top}\cdot\indicator_{a\leq\vect x\cdot\vect v<b}\right]\preceq W^{a, b}+\beta I_{\binom{d+1}{k_{2}}\times\binom{d+1}{k_{2}}}\label{eq: spectral approximation to moments in a strip}
\end{equation}
\end{claim}

\begin{proof}
From Section \ref{claim: deterministic moments in a strip},  we recall
that $\vect x^{\otimes k_{2}}$ is the vector one gets by evaluating
all multidimensional monomials of degree at most $k_{2}$ on input
$\vect x$,  and therefore the $\binom{d+1}{k_{2}}\times\binom{d+1}{k_{2}}$-matrix
$(\vect x^{\otimes k_{2}})(\vect x^{\otimes k_{2}})^{\top}$, 
viewed as a bilinear form,  for degree-$k$ polynomials $p_{1}$ and
$p_{2}$ we have $p_{1}^{\top}(\vect x^{\otimes k_{2}})(\vect x^{\otimes k_{2}})^{\top}p_{2}=p_{1}(\vect x)p_{2}(\vect x)$.
Thus,  the entries of $(\vect x^{\otimes k_{2}})(\vect x^{\otimes k_{2}})^{\top}$
are indexed by pairs of monomials $m_{1}$ and $m_{2}$ over $\R^{d}$
of degree at most $k$,  and we have $m_{1}^{\top}(\vect x^{\otimes k_{2}})(\vect x^{\otimes k_{2}})^{\top}m_{2}=m_{1}(\vect x)m_{2}(\vect x)$.
And the entries of $\E_{\x\sim \Gauss}[(\vect x^{\otimes k_{2}})(\vect x^{\otimes k_{2}})^{\top}\cdot\indicator_{a\leq\vect x\cdot\vect v<b}]$
are also indexed by pairs of monomials $m_{1}$ and $m_{2}$ over
$\R^{d}$ of degree at most $k$ and equal to\\ $\E_{\x\sim \Gauss}[m_{1}(\vect x)m_{2}(\vect x)\cdot\indicator_{a\leq\vect x\cdot\vect v<b}].$
Since the product $m_{1}m_{2}$ is a monomial of degree at most $2k_{2}$, 
Claim \ref{claim: deterministic moments in a strip} tells us that
this value can be approximated up to error $\beta'$ in time $\poly((k_{2}d)^{k_{2}}/\beta')$
(we will set the value of $\beta'$ later).

Thus,  we take the $\binom{d+1}{k_{2}}\times\binom{d+1}{k_{2}}$-matrix
$W^{a, b}$ to have entries pairs of monomials $m_{1}$ and $m_{2}$
over $\R^{d}$ of degree at most $k$ and equal to additive $\beta'$-approximations
to ${\E}_{\vect x\sim\N}[m_{1}(\vect x)m_{2}(\vect x)\cdot\indicator_{a\leq\vect x\cdot\vect v<b}].$
Thus,  the difference between $\E_{\x\sim \Gauss}[(\vect x^{\otimes k_{2}})(\vect x^{\otimes k_{2}})^{\top}\cdot\indicator_{a\leq\vect x\cdot\vect v<b}]$
and $W^{a, b}$ is a $\binom{d+1}{k_{2}}\times\binom{d+1}{k_{2}}$-matrix
whose entries are bounded by $\beta'$ in absolute value. Thus,  the
Frobenius norm of this difference matrix is at most $(d+1)^{k_{2}/2}\beta'$, 
and therefore all eigenvalues of the matrix $\E_{\x\sim \Gauss}[m_{1}(\vect x)m_{2}(\vect x)\cdot\indicator_{a\leq\vect x\cdot\vect v<b}]-W^{a, b}$
are in $[-(d+1)^{k_{2}/2}\beta', (d+1)^{k_{2}/2}\beta']$.
Taking $\beta'=\beta(d+1)^{-k_{2}/2}$,  we can conclude that Equation
\ref{eq: spectral approximation to moments in a strip} holds.

Overall,  the run-time is $\poly((k_{2}d)^{k_{2}}/\beta')$, 
which we see equals to $\poly((k_{2}d)^{k_{2}}/\beta)$
since we have $\beta'=\beta(d+1)^{-k_{2}/2}$. This completes the
proof.
\end{proof}

\section{Testing Massart Noise}\label{appendix:testing-massart}

We give here the full proof of our main theorem (\Cref{theorem:main-result}), which we restate for convenience.

\begin{theorem}
There exists a deterministic algorithm $\mathcal{\mathcal{A_{\text{Massart}}}}$
that runs in time $\poly\left(\frac{Nd}{\epsilon\delta}\right)$ and
for a sufficiently large absolute constant $C$ satisfies the following.
Given parameters $\epsilon, \delta$ in $(0, 1)$ and a dataset ${\bar S}$
of size $N\geq\left(\frac{Cd}{\epsilon\delta}\right)^{C}$ consisting
of elements in $\R^{d}\times\left\{ \pm1\right\} $,  the algorithm
$\mathcal{\mathcal{\mathcal{A_{\text{Massart}}}}}$ outputs either
$(\accept,\vect v)$ for some unit vector \textbf{$\vect v$
}in $\R^{d}$,  or outputs $\reject$ (in the former case we say
$\mathcal{A}$ accepts,  while in the latter case we say $\mathcal{A}$
rejects). The algorithm $\mathcal{A_{\text{Massart}}}$ satisfies
the following conditions:
\end{theorem}

\begin{enumerate}
\item \textbf{Completeness: }The algorithm $\mathcal{\mathcal{A_{\text{Massart}}}}$
accepts with probability at least $1-O(\delta)$ if ${\bar S}$
is generated by ${\Massart_{\Gauss,f,\eta_0}}$ where $f$ is
an origin-centered halfspace and $\eta_{0}\leq1/3$.
\item \textbf{Soundness:} For any dataset ${\bar S}$ of size $N\geq\left(\frac{Cd}{\epsilon\delta}\right)^{C}$, 
if $\mathcal{\mathcal{\mathcal{A_{\text{Massart}}}}}$ accepts then
the vector $\vect v$ given by $\mathcal{\mathcal{A_{\text{Massart}}}}$
satisfies
\begin{equation}
\pr_{(\vect x, y)\sim {\bar S}}\left[\sign(\vect x\cdot\vect v)\neq y\right]\leq\opt+O(\epsilon), \label{eq: optimality}
\end{equation}
where $\opt$ is defined to be $\min_{\vect v'\in\R^{d}}\left(\pr_{(\vect x, y)\sim {\bar S}}\left[\sign(\vect x\cdot\vect v')\neq y\right]\right)$.
\end{enumerate}
The rest of this section proves the above theorem.
The algorithm $\mathcal{\mathcal{A_{\text{Massart}}}}$ does the following
(where $C$ is a sufficiently large absolute constant):
\begin{itemize}
\item \textbf{Given: }parameters $\epsilon, \delta$ in $(0, 1)$ and a dataset
${\bar S}$ of size $N\geq\left(\frac{Cd}{\epsilon\delta}\right)^{C}$
consisting of elements in $\R^{d}\times\left\{ \pm1\right\} $.
\begin{enumerate}
\item Let $\vect v$ be the output of the algorithm of \Cref{fact:massart-parameter-recovery} run on the dataset $\bar S$ with accuracy parameter $\eps' = \frac{\eps^{3/2}}{100 \sqrt{d-1}}$ and failure probability $\delta$. Without loss of generality we can assume that the algorithm is deterministic, because we can use some of the points in $\bar S$ for random seeds.
\item $\Sunlabeled\leftarrow\left\{ \vect x:\ (\vect x, y)\in {\bar S}\right\} $
and $N\leftarrow\abs{{\bar S}}$. 
\item Run the tester $\mathcal{\Tdisagree}$ from Theorem \ref{thm: disagreement tester}, 
on input $(\Sunlabeled,  \vect v,  \epsilon,  \delta,  0.1)$.
\item If $\Tdisagree$ rejects in the previous step,  output output $\reject$.
\item If $\abs{S_{\mathsf{False}}}>\frac{2}{5}N$,  then output $\reject$.
\item $S_{\mathsf{False}}^{\mathsf{far}}\leftarrow S_{\mathsf{False}}\cap\left\{ \vect x\in\R^{d}:\ \abs{\measuredangle(\vect x, \vect v)-\frac{\pi}{2}}>\frac{\epsilon^{3/2}}{\sqrt{d-1}}\right\} $;
$S_{\mathsf{False}}^{\mathsf{near}}\leftarrow S_{\mathsf{False}}\setminus S_{\mathsf{False}}^{\mathsf{far}}$.
\item If $S_{\mathsf{False}}^{\mathsf{near}}>4\epsilon N$,  then
output $\reject$.
\item Take $U=\frac{2}{5}N$ and then run the spectral tester $\mathcal{\Tspectral}$ from Theorem \ref{thm: spectral tester} with the input parameters $(U,  S_{\mathsf{False}}^{\mathsf{far}},  \vect v,  \epsilon,  \delta,  0.1)$.
\item If $\mathcal{\Tspectral}$ rejects in the previous step,  output output $\reject$.
\item Otherwise,  output  $(\accept,\vect v)$. 
\end{enumerate}
\end{itemize}
From the run-time guarantees given in Theorem \ref{thm: disagreement tester}
and Theorem \ref{thm: spectral tester},  we see immideately that the
run-time of the algorithm $\mathcal{A_{\text{Massart}}}$ is $\poly\left(\frac{d}{\epsilon}\log\frac{1}{\delta}\right)$. 

\subsection{Soundness}

We first show the soundness condition. For any dataset ${\bar S}$
of size $N\geq\left(\frac{Cd}{\epsilon\delta}\right)^{C}$,  we need
to show that if $\mathcal{A_{\text{Massart}}}$ accepts then the vector
$\vect v$ given by $\mathcal{\mathcal{A_{\text{Massart}}}}$ satisfies
Equation \ref{eq: optimality}. Theorems \ref{thm: disagreement tester}
and \ref{thm: spectral tester} imply that if the algorithm $\mathcal{A_{\text{Massart}}}$accepts
then

\begin{equation}
\pr_{(\vect x, y)\sim {\bar S}}\left[\sign(\vect x\cdot\vect v)\neq\sign(\vect x\cdot\vect v')\right]=(1\pm0.1)\frac{\measuredangle(\vect v, \vect v')}{\pi}\pm O(\epsilon).\label{eq: the bound from disagreement tester}
\end{equation}
\begin{equation}
\frac{1}{U}\sum_{\vect x\in S_{\mathsf{False}}^{\mathsf{far}}}\left[\indicator_{\sign(\vect x\cdot\vect v)\neq\sign(\vect x\cdot\vect v')}\right]\leq1.1\frac{\measuredangle(\vect v, \vect v')}{\pi}+O(\epsilon).\label{eq: the bound from spectral tester}
\end{equation}
Rearranging,  we get 
\begin{multline*}
\pr_{(\vect x, y)\sim {\bar S}}\left[\sign(\vect x\cdot\vect v')\neq y\right]-\pr_{(\vect x, y)\sim {\bar S}}\left[y\neq\sign(\vect v\cdot\vect x)\right]=\\
\pr_{(\vect x, y)\sim {\bar S}}\left[\sign(\vect x\cdot\vect v)\neq\sign(\vect x\cdot\vect v')\land y=\sign(\vect v\cdot\vect x)\right]-\\
-\pr_{(\vect x, y)\sim {\bar S}}\left[\sign(\vect x\cdot\vect v)\neq\sign(\vect x\cdot\vect v')\land y\neq\sign(\vect v\cdot\vect x)\right]=\\
\pr_{(\vect x, y)\sim {\bar S}}\left[\sign(\vect x\cdot\vect v)\neq\sign(\vect x\cdot\vect v')\right]-2\pr_{(\vect x, y)\sim {\bar S}}\left[\sign(\vect x\cdot\vect v)\neq\sign(\vect x\cdot\vect v')\land y\neq\sign(\vect v\cdot\vect x)\right]=\\
\pr_{\vect x\sim \Sunlabeled}\left[\sign(\vect x\cdot\vect v)\neq\sign(\vect x\cdot\vect v')\right]-\frac{2}{N}\sum_{\vect x\in S_{\mathsf{False}}}\left[\indicator_{\sign(\vect x\cdot\vect v)\neq\sign(\vect x\cdot\vect v')}\right]
\end{multline*}
By Equation \ref{eq: the bound from disagreement tester},  the first
term above is lower-bounded by $0.9\frac{\measuredangle(\vect v, \vect v')}{\pi}-O(\epsilon)$.
The second term can broken into two components: $\frac{2}{N}\sum_{\vect x\in S_{\mathsf{False}}^{\mathsf{near}}}\left[\indicator_{\sign(\vect x\cdot\vect v)\neq\sign(\vect x\cdot\vect v')}\right]$
and $\frac{2}{N}\sum_{\vect x\in S_{\mathsf{False}}^{\mathsf{far}}}\left[\indicator_{\sign(\vect x\cdot\vect v)\neq\sign(\vect x\cdot\vect v')}\right]$.
If the algorithm does not reject in step (8),  the former term is upper-bounded
by $O(\epsilon)$,  while Equation \ref{eq: the bound from disagreement tester}
tells us that the latter term is upper-bounded by $\frac{2U}{N}\left(1.1\frac{\measuredangle(\vect v, \vect v')}{\pi}+O(\epsilon)\right)$.
Overall,  substituting these bounds and recalling that $U/N=2/5$, 
we get
\begin{multline*}
\pr_{(\vect x, y)\sim {\bar S}}\left[\sign(\vect x\cdot\vect v')\neq y\right]-\pr_{(\vect x, y)\sim {\bar S}}\left[y\neq\sign(\vect v\cdot\vect x)\right]\geq\\
0.9\frac{\measuredangle(\vect v, \vect v')}{\pi}-O(\epsilon)-\frac{2U}{N}\left(1.1\frac{\measuredangle(\vect v, \vect v')}{\pi}+O(\epsilon)\right)=\\
0.9\frac{\measuredangle(\vect v, \vect v')}{\pi}-\frac{4}{5}\left(1.1\frac{\measuredangle(\vect v, \vect v')}{\pi}\right)-O(\epsilon)=0.02\frac{\measuredangle(\vect v, \vect v')}{\pi}-O(\epsilon)\geq-O(\epsilon).
\end{multline*}
Thus,  choosing $\vect v'$ to be $\arg\min_{\vect u}\left(\pr_{(\vect x, y)\sim\dpairs}\left[\sign(\vect x\cdot\vect u)\neq y\right]\right)$, 
we get 
\[
\pr_{(\vect x, y)\sim {\bar S}}\left[\sign(\vect x\cdot\vect v)\neq y\right]\leq\opt+O(\epsilon), 
\]
finishing the proof of soundess.

\subsection{Completeness}

We now argue that for a sufficiently large absolute constant $C$, 
the algorithm $\mathcal{\mathcal{A_{\text{Massart}}}}$ satisfies
the completeness condition. In this subsection we assume that ${\bar S}$
is generated by ${\Massart_{\Gauss,f,\eta_0}}$ where $f(\vect x)=\sign({\vv^*}\cdot\vect x)$
is an origin-centered halfspace and $\eta_{0}\leq1/3$. We remind
the reader that,  for some function $\eta:\R^{d}\rightarrow[0, \eta_{0}]$, 
every time ${\Massart_{\Gauss,f,\eta_0}}$ is invoked it generates
an i.i.d. pair $(\vect x, y)\in\R^{d}\times\left\{ \pm1\right\} $
where $\vect x$ is drawn from $\N(0, I_{d})$ and $y=f(\vect x)$
with probability $\eta(\vect x)$ and $-f(\vect x)$ with probability
$1-\eta(\vect x)$. We would like to show that $\mathcal{A_{\text{Massart}}}$
accepts with probability at least $1-O(\delta)$.

For the purposes of completeness analysis,  we define the set $S_{\text{augmented}}$
to be a set of points in $\R^{d}$ generated through the following
random process:
\begin{itemize}
\item If a datapoint $\vect x$ in $\Sunlabeled$ has label $y=-f(\vect x)$, 
then $\vect x$ in included into $S_{\text{augmented}}$.
\item If a datapoint $\vect x$ in $\Sunlabeled$ has label $y=f(\vect x)$, 
then $\vect x$ in included into $S_{\text{augmented}}$ with probability
$\frac{\eta_{0}-\eta(\vect x)}{1-\eta(\vect x)}$ (and this choice
is made independently for different $\vect x$ in $\Sunlabeled$).
\end{itemize}
With the definition above in hand,  we claim the following:
\begin{claim}
\label{claim: about the distribution of S augmented}If the absolute
constant $C$ is large enough,  then with probability at least $1-\delta$
it is the case that $\abs{S_{\text{augmented}}}\leq\frac{2}{5}N$.
Furthermore,  conditioned on any particular value of the size $|S_{\text{augmented}}|$ of this set,  the individual elements of
$S_{\text{augmented}}$ are distributed i.i.d. from the standard Gaussian
distribution $\N(0, I_{d})$.
\end{claim}

\begin{proof}
Overall,  we know that $y=-f(\vect x)$ with probability $\eta(\vect x)$, 
so overall each $\vect x$ in $\Sunlabeled$ gets included
into $S_{\text{augmented}}$ independently with probability $\eta(\vect x)+(1-\eta(\vect x))\frac{\eta_{0}-\eta(\vect x)}{1-\eta(\vect x)}=\eta_{0}$.
Overall every element $\Sunlabeled$ is included into $S_{\text{augmented}}$with
independently probability $\eta_{0}$. Since $\eta_{0}$ is at most
$1/3$ and $N\geq\left(\frac{Cd}{\epsilon\delta}\right)^{C}$,  we
see that the standard Hoeffding bound tells us that for a sufficiently
large absolute constant $C$ with probability at least $1-\delta$
it is the case that $\abs{S_{\text{augmented}}}\leq\frac{2}{5}\abs{\Sunlabeled}=\frac{2}{5}N$.
This proves the first part of the claim.

Additionally,  recall that in this subsection we are assuming that
${\bar S}$ is generated by ${\Massart_{\Gauss,f,\eta_0}}$.
This implies that the elements of $\Sunlabeled$ are generated
i.i.d. from $\N(0, I_{d})$. Since the decision wheather each datapoint
$\vect x$ in $\Sunlabeled$is included into $S_{\text{augmented}}$
is made with probability $\eta_{0}$ independently from the actual
value of $\vect x$,  this implies the element of $S_{\text{augmented}}$
are distributed i.i.d. as $\N(0, I_{d})$ even conditioned on any specific
value of $|S_{\text{augmented}}|$. This finishes the proof of the
claim.
\end{proof}
The following claim lists a number of desirable events for algorithm
$A_{\text{Massart}}$ and shows that they are likely to hold. 
\begin{claim}
\label{claim: massart bad events unlikely}If $C$ is a sufficiently
large absolute constant,  the following events take place with probability
at least $1-O(\delta)$:
\begin{enumerate}
\item The set $\Sunlabeled$ is such that for all unit vectors $\vect{v}'$ the algorithm $\Tdisagree$ accepts when given the input $(\Sunlabeled,  \vect{v}',  \epsilon,  \delta,  0.1)$.
\item For all vectors $\vect u$ in $\R^{d}$,  we have 
\[
\abs{\pr_{(\vect x, y)\sim {\bar S}}\left[\sign(\vect x\cdot\vect u)\neq y\right]-\pr_{(\vect x, y)\sim{\Massart_{\Gauss,f,\eta_0}}}\left[\sign(\vect x\cdot\vect v)\neq y\right]}\leq2d\sqrt{\frac{\log N}{N}}\log\frac{1}{\delta}.
\]
\item For all vectors $\vect u$ in $\R^{d}$ and scalars $\angleparam$,  we
have 
\[
\abs{\pr_{\vect x\sim \Sunlabeled}\left[\measuredangle(\vect x, \vect u)\leq\angleparam\right]-\pr_{\vect x\sim\N(0, I_{d})}\left[\measuredangle(\vect x, \vect u)\leq\angleparam\right]}\leq2d\sqrt{\frac{\log N}{N}}\log\frac{1}{\delta}.
\]
\item It is the case that $\measuredangle(\vect v, {\vv^*})\leq\frac{\epsilon^{3/2}}{10\sqrt{d-1}}$.
\item It is the case that $\abs{S_{\mathsf{False}}}\leq\frac{2}{5}N$
and $\abs{S_{\text{augmented}}}\leq\frac{2}{5}N$.
\item $S_{\text{augmented}}$ is such that for all unit vectors $\vect{v}'$ the algorithm $\mathcal{\Tspectral}$
accepts when given as input 
on the input $\left(U,  S_{\text{augmented}},  \vect{v}',  \epsilon,  \delta,  0.1\right)$ (we remind the reader that $U=\frac{2}{5}N$).
\end{enumerate}
\end{claim}

\begin{proof}
Event 1 holds with probability at least $1-O(\delta)$ by Theorem
\ref{thm: disagreement tester}. The Event (2) holds with probability
at least $1-O(\delta)$ by the standard VC bound,  together with the
fact that the VC dimension of the class of halfspaces in $\R^{d}$
is at most $d+1$. Analogously,  Event (2) holds with probability at
least $1-O(\delta)$ by the standard VC bound,  together with the fact
that the VC dimension of the class of origin-centric cones in $\R^{d}$
is most $O(d)$. 

Recall that in step (1) of $A_{\text{Massart}}$ we used the algorithm
of \cite{diakonikolas2020learning} (see \Cref{fact:massart-parameter-recovery}) which
implies that with probability at least $1-\delta$ we have 
$\measuredangle(\vect v, {\vv^*})\leq\frac{\epsilon^{3/2}}{10\sqrt{d-1}}$.

If Event (2) holds,  we have 

\[
\pr_{(\vect x, y)\sim {\bar S}}\left[\sign(\vect x\cdot\vect v)\neq y\right]\leq\pr_{(\vect x, y)\sim{\Massart_{\Gauss,f,\eta_0}}}\left[\sign(\vect x\cdot\vect v)\neq y\right]+2d\sqrt{\frac{\log N}{N}}\log\frac{1}{\delta}, 
\]
and if Equation \ref{thm: disagreement tester} also holds we have
\begin{multline*}
\frac{\abs{S_{\mathsf{False}}}}{N}=\pr_{(\vect x, y)\sim {\bar S}}\left[\sign(\vect x\cdot\vect v)\neq y\right]\leq\\
\pr_{(\vect x, y)\sim{\Massart_{\Gauss,f,\eta_0}}}\left[\sign(\vect x\cdot{\vv^*})\neq y\right]+\frac{\epsilon^{3/2}}{100\sqrt{d-1}}+2d\sqrt{\frac{\log N}{N}}\log\frac{1}{\delta}\leq\\
\eta_{0}+\frac{1}{100}+2d\sqrt{\frac{\log N}{N}}\log\frac{1}{\delta}\leq\frac{1}{3}+\frac{1}{100}+2d\sqrt{\frac{\log N}{N}}\log\frac{1}{\delta}.
\end{multline*}
Substituting $N\geq\left(\frac{Cd}{\epsilon\delta}\right)^{C}$,  we
see that the above is at most $\frac{2}{5}$ if $C$ is a sufficiently
large absolute constant. Thus,  with probability at least $1-O(\delta)$
we have $\abs{S_{\mathsf{False}}}\leq\frac{2N}{5}$. At the same
time,  Claim \ref{thm: disagreement tester} tells us that with probability
at least $1-\delta$ we have $\abs{S_{\text{augmented}}}\leq U=\frac{2N}{5}$.
Overall,  we see that Event (5) holds with probability at least $1-O(\delta)$.

Finally,  Claim \ref{thm: spectral tester} tells us that with probability
at least $1-O(\delta)$ it is the case that $\abs{S_{\text{augmented}}}\leq\frac{2}{5}N$.
Furthermore,  Claim \ref{thm: spectral tester} also tells us that, 
even conditioned on this event,  the set $S_{\text{augmented}}$ consists
of i.i.d. samples from $\N(0, I_{d})$. Then,  the Completeness condition
in Theorem \ref{thm: spectral tester} tells us that with probability
at least $1-O(\delta)$ it is the case that $\mathcal{\Tspectral}$
accepts if and is given $\mu=0.1$,  $U=\frac{2}{5}N$ and input dataset
$S_{\text{augmented}}$. 
\end{proof}
Now,  we first note that if Event 1 takes place,  then $\mathcal{\Tdisagree}$ accepts in step (3) of the algorithm
$A_{\text{Massart}}$.

If Event 3 in Claim \ref{claim: massart bad events unlikely} takes
place,  then from the triangle inequality it follows that 
\begin{equation}
\abs{\pr_{\vect x\sim \Sunlabeled}\left[\abs{\measuredangle(\vect x, \vect v)-\frac{\pi}{2}}\leq\frac{\epsilon^{3/2}}{\sqrt{d-1}}\right]-\pr_{\vect x\sim\N(0, I_{d})}\left[\abs{\measuredangle(\vect x, \vect v)-\frac{\pi}{2}}\leq\frac{\epsilon^{3/2}}{\sqrt{d-1}}\right]}\leq4d\sqrt{\frac{\log N}{N}}\log\frac{1}{\delta}.\label{eq: number of points close to margin close to expectation}
\end{equation}
It is also the case that 
\begin{multline}
\pr_{\vect x\sim\N(0, I_{d})}\left[\abs{\measuredangle(\vect x, \vect v)-\frac{\pi}{2}}\leq\frac{\epsilon^{3/2}}{\sqrt{d-1}}\right]\\
\le \pr_{\vect x\sim\N(0, I_{d})}\left[\abs{\vect x\cdot\vect v}\leq\epsilon\right]+\pr_{\vect x\sim\N(0, I_{d})}\left[\norm{\vect x-\vect v\left(\vect x\cdot\vect v\right)}\tan\left(\frac{\epsilon^{3/2}}{\sqrt{d-1}}\right)\leq\epsilon\right]\\\
\le \pr_{\vect x\sim\N(0, I_{d})}\left[\abs{\vect x\cdot\vect v}\leq\epsilon\right]+\pr_{\vect x\sim\N(0, I_{d})}\left[\norm{\vect x-\vect v\left(\vect x\cdot\vect v\right)}\leq\sqrt{\frac{d-1}{\epsilon}}\right]\\
=\pr_{x\sim\N(0, 1)}\left[\abs x\leq\epsilon\right]+\pr_{\vect x\sim\N(0, I_{d-1})}\left[\norm{\vect x}\leq\sqrt{\frac{d-1}{\epsilon}}\right]\leq3\epsilon\label{eq: small margin region has small Gaussian mass}
\end{multline}
Combining Equations \ref{eq: number of points close to margin close to expectation}
and \ref{eq: small margin region has small Gaussian mass}
we get
\[
\pr_{\vect x\sim \Sunlabeled}\left[\abs{\measuredangle(\vect x, \vect v)-\frac{\pi}{2}}\leq\frac{\epsilon^{3/2}}{\sqrt{d-1}}\right]\leq3\epsilon+4d\sqrt{\frac{\log N}{N}}\log\frac{1}{\delta}\leq4\epsilon, 
\]
where the last inequality holds if $C$ is a sufficiently large absolute
constant. Since every element $\vect x$ in $S_{\mathsf{False}}^{\mathsf{near}}$
is in $\Sunlabeled$ and also satisfies $\abs{\measuredangle(\vect x, \vect v)-\frac{\pi}{2}}\leq\frac{\epsilon^{3/2}}{10\sqrt{d}}$, 
we see that $S_{\mathsf{False}}^{\mathsf{near}}$ has a size
of at most $4\epsilon N$ and therefore the algorithm $\mathcal{A_{\text{Massart}}}$
does not reject in step 8.

If Event 4 in Claim \ref{claim: massart bad events unlikely} takes
place,  then it is the case that $\measuredangle(\vect v, {\vv^*})\leq\frac{\epsilon^{3/2}}{2\sqrt{d-1}}$.
If this is the case,  the halfspaces $\sign(\vect v\cdot\vect x)$
and $\sign({\vv^*}\cdot\vect x)$ will agree for all
vectors $\vect x$ satisfying $\abs{\measuredangle(\vect x, \vect v)-\frac{\pi}{2}}>\frac{\epsilon^{3/2}}{\sqrt{d-1}}$, 
which holds for all $\vect x$ in $S_{\mathsf{False}}^{\mathsf{far}}$.
Overall,  for every $\vect x$ in $S_{\mathsf{False}}^{\mathsf{far}}$
the corresponding label $y$ satisfies $y\neq\sign(\vect v\cdot\vect x)=\sign({\vv^*}\cdot\vect x)$.
Recalling the definition of the set $S_{\text{augmented}}$,  we see
that $S_{\mathsf{False}}^{\mathsf{far}}\subseteq S_{\text{augmented}}$.
If Event 6 in Claim \ref{claim: massart bad events unlikely} holds
then the set $S_{\text{augmented}}$ is such that if the
algorithm $\mathcal{\Tspectral}$ accepts when given as input $\left(U,  S_{\text{augmented}},  \vect{v},  \epsilon,  \delta,  0.1\right)$. But Theorem \ref{thm: spectral tester} shows
that $\mathcal{\Tspectral}$ satisfies Monotonicity under Datapoint
Removal,  which together with the inclusion $S_{\mathsf{False}}^{\mathsf{far}}\subseteq S_{\text{augmented}}$ implies that $\mathcal{\Tspectral}$ accepts if it
is given 
$\left(U,  S_{\mathsf{False}}^{\mathsf{far}},  \vect{v},  \epsilon,  \delta,  0.1\right)$.
Thus,  the tester $\mathcal{\Tspectral}$ does not reject in step 10.

We conclude that with probability at least $1-O(\delta)$
the algorithm $A_{\text{Massart}}$ will not reject in any of the
four steps in which it could potentially reject. If this is the case, 
the algorithm $A_{\text{Massart}}$ will accept.
\end{document}